\documentclass{article}

\usepackage{microtype}
\usepackage{booktabs} 
\usepackage{wrapfig}
\usepackage{graphicx}
\usepackage{blindtext}
\usepackage{tabularx}
\usepackage{booktabs}
\usepackage{amsthm,amstext}
\usepackage[labelfont=bf,format=plain,justification=raggedright,singlelinecheck=false]{caption}
\usepackage{cancel}

\usepackage{multirow}

\usepackage{color}
\usepackage{xcolor}

\usepackage{hyperref}
\usepackage{nicefrac}

\usepackage[accepted]{icml2023}

\usepackage{amsmath}
\usepackage{amssymb}
\usepackage{mathtools}
\usepackage{amsthm}
\usepackage[nothm]{YK}
\usepackage{wrapfig}
\usepackage{subcaption}

\usepackage[capitalize,noabbrev]{cleveref}
\usepackage{makecell}
\theoremstyle{plain}
\newtheorem{theorem}{Theorem}[section]

\newtheorem{lemma}[theorem]{Lemma}
\newtheorem{corollary}[theorem]{Corollary}
\theoremstyle{definition}
\newtheorem{definition}[theorem]{Definition}
\newtheorem{assumption}[theorem]{Assumption}
\theoremstyle{remark}

\def\reals{\bR}

\def\reals{\bR}

\def\stylefactor{\bF_\text{S}}
\def\contentfactor{\bF_\text{C}}

\def\method{\textsc{PISCO}}

\usepackage[textsize=tiny]{todonotes}

\icmltitlerunning{Simple Disentanglement of Style and Content in Visual Representations}

\begin{document}


\twocolumn[
\icmltitle{Simple Disentanglement of Style and Content in Visual Representations}



\icmlsetsymbol{equal}{*}

\begin{icmlauthorlist}
\icmlauthor{Lilian Ngweta}{equal,yyy}
\icmlauthor{Subha Maity}{equal,xxx}
\icmlauthor{Alex Gittens}{yyy}
\icmlauthor{Yuekai Sun}{xxx}
\icmlauthor{Mikhail Yurochkin}{comp,lab}
\end{icmlauthorlist}

\icmlaffiliation{yyy}{Department of Computer Science, Rensselaer Polytechnic Institute, Troy, New York, United States}
\icmlaffiliation{xxx}{Department of Statistics, University of Michigan, Ann Arbor, Michigan, United States}
\icmlaffiliation{comp}{IBM Research, Cambridge, Massachusetts, United States}
\icmlaffiliation{lab}{MIT-IBM Watson AI Lab, Cambridge, Massachusetts, United States}

\icmlcorrespondingauthor{Lilian Ngweta}{ngwetl@rpi.edu}
\icmlcorrespondingauthor{Subha Maity}{smaity@umich.edu}
\icmlcorrespondingauthor{Mikhail Yurochkin}{mikhail.yurochkin@ibm.com}

\icmlkeywords{disentanglement, image recognition, distribution shifts}

\vskip 0.3in
]



\printAffiliationsAndNotice{\icmlEqualContribution} 

\begin{abstract}
Learning visual representations with interpretable features, i.e., disentangled representations, remains a challenging problem. Existing methods demonstrate some success but are hard to apply to large-scale vision datasets like ImageNet. In this work, we propose a simple post-processing framework to disentangle content and style in learned representations from pre-trained vision models. We model the pre-trained features probabilistically as linearly entangled combinations of the latent content and style factors and develop a simple disentanglement algorithm based on the probabilistic model. We show that the method provably disentangles content and style features and verify its efficacy empirically. Our post-processed features yield significant domain generalization performance improvements when the distribution shift occurs due to style changes or style-related spurious correlations. 
\end{abstract}

\section{Introduction}
\label{sec:intro}

\begin{figure}
    \centering
    \includegraphics[scale=0.18]{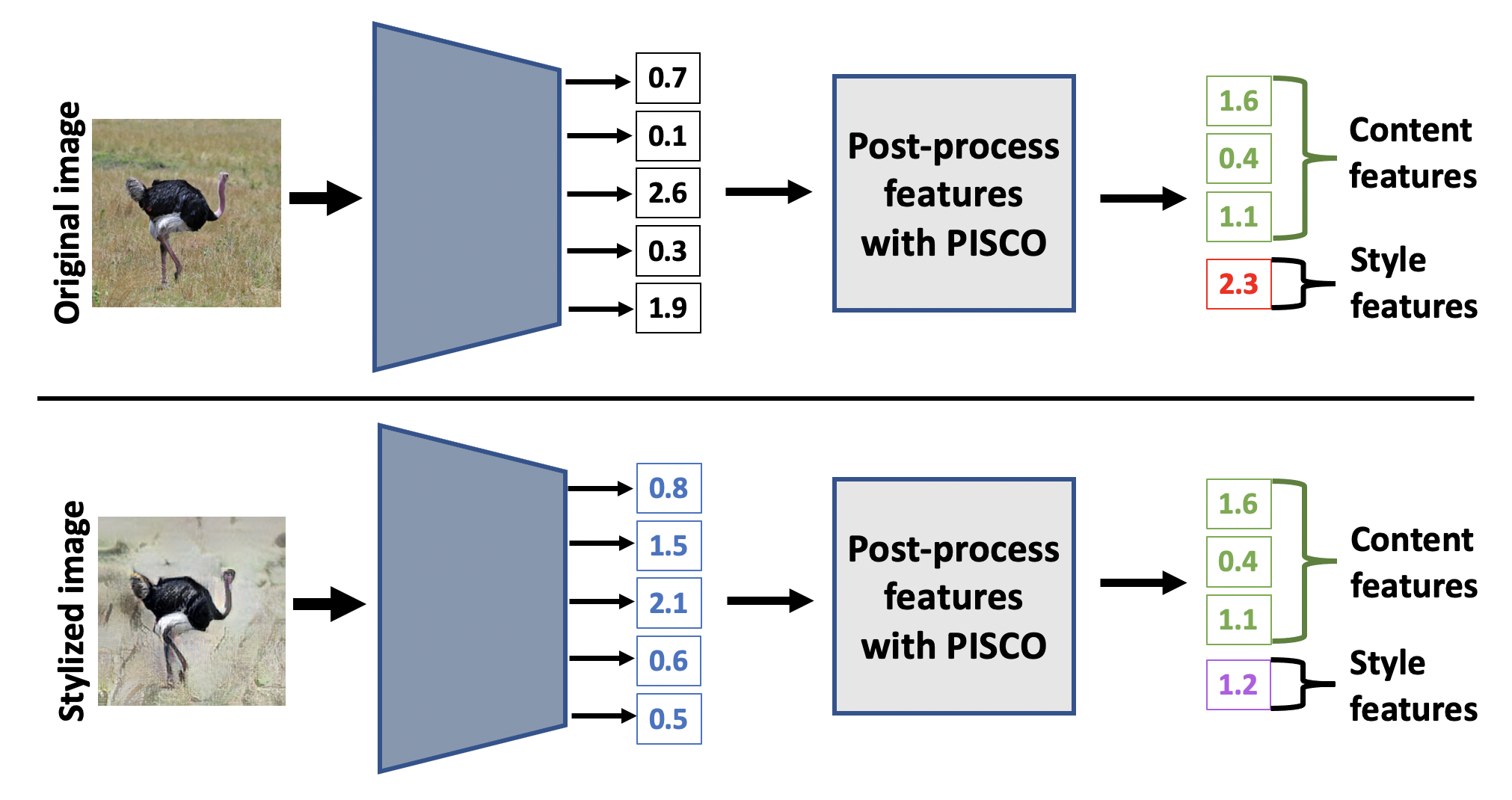}
    \caption{Illustration of the proposed method, PISCO. Features of an original image and features of a stylized image are different when extracted using a feature extractor such as ResNet-50 pre-trained in ImageNet. These features are entangled, thus changing the style affects all features. 
    When PISCO is used to disentangle these features, it isolates style features and content features, thus the content features of the two images are the same and only the style features are different. }
    \label{fig:pisco_diagram}
\end{figure}


Deep learning models produce data representations that are useful for many downstream tasks. Disentangled representations, i.e. representations where coordinates have meaningful interpretations, are harder to learn \citep{locatello2019challenging} but they come with many additional benefits, e.g., data-efficiency \citep{higgins2018towards} and use-cases in causality \citep{scholkopf2021toward}, fairness \citep{locatello2019fairness}, recommender systems \citep{ma2019learning}, and image \citep{lee2018diverse} and text \citep{john2018disentangled} processing.

In this paper, we consider the problem of isolating content from style in visual representations \citep{wu2019disentangling,nemeth2020adversarial,ren2021rethinking,kugelgen2021self}, a special case of learning disentangled representations. Here we use the term style to refer to features or factors that are not causally related to the outcome of interest. We also note prior works that, different from our work, study style in the context of image appearance \citep{garcia2018read, saleh2015large, ruta2022stylebabel, ruta2021aladin}. Our goal is to obtain representations where a pre-specified set of factors (coordinates) encodes image ``styles'' (e.g., rotation, color scheme, or style transfer \citep{huang2017arbitrary}), while the remaining factors encode content and are invariant to style changes (see Figure \ref{fig:pisco_diagram}).

An important application of such representations is out-of-distribution (OOD) generalization. Image recognition systems have been demonstrated to be susceptible to spurious correlations associated with style, e.g., due to background colors \citep{beery2018recognition,sagawa2019Distributionally}, and to various style-based distribution shifts, e.g., due to image corruptions \citep{hendrycks2018Benchmarking}, illumination, or camera angle differences \citep{koh2020WILDS}. Simply discarding style factors when training a prediction model on disentangled representations can aid OOD generalization. Disentangling content from style is also advantageous in many other applications, e.g., image retrieval \citep{wu2009Genomewide}, image-to-image translation \citep{ren2021rethinking}, and visually-aware recommender systems \citep{deldjoo2022leveraging}. 


While there is abundant literature on learning disentangled representations, most statistically principled methods fit sophisticated generative models \citep{bouchacourt2018multi,hosoya2018group,shu2019weakly,wu2019disentangling,locatello2020weakly}. These methods work well on synthetic and smaller datasets but are hard to train on larger datasets like ImageNet \citep{russakovsky2015imagenet}.
This is in stark contrast to representation learning practice; the most common representation learning methods only learn an encoder (\eg\ SimCLR \citep{chen2020simple}). That said, there are some recent works that consider how to learn disentangled encoders \citep{zimmermann2021contrastive,kugelgen2021self,wang2021self}. 

Specific to style and content, \citet{kugelgen2021self} show that contrastive learning, e.g., SimCLR \citep{chen2020simple}, \emph{theoretically} can isolate style and content, i.e. learn representations that are invariant to style. Contrastive learning methods are gaining popularity due to their ability to learn high-quality representations from large image datasets without labels via self-supervision \citep{doersch2015unsupervised,chen2020simple,chen2020weakly,grill2020bootstrap,chen2021exploring}. Unfortunately, style invariance of contrastive learning representations is rarely achieved in practice due to a variety of additional requirements that are hard to control for (see Section 5 and Appendix C.1 in \citet{kugelgen2021self}).


Most of the prior works are in-processing methods that train an end-to-end encoder from scratch. On the other hand, we focus on post-processing representations from a pre-trained deep model (which may not be disentangled) so that they become provably disentangled.
The post-processing setup is appealing as it allows re-using large pre-trained models, thus reducing the carbon footprint of training new large models \citep{strubell2019energy} and making deep learning more accessible to practitioners with limited computing budgets. Post-processing problem setups are prominent in the algorithmic fairness literature 
\citep{wei2019Optimized,petersen2021post}.

To develop our post-processing setup for learning disentangled representations, we assume that the pre-trained representations are simply an \emph{invertible} linear transformation of the style and content factors (\cf\ Assumption \ref{assmp:entangled-rep}). While the linear model assumption may appear too simple at a first glance, it is motivated by the result of \citet{zimmermann2021contrastive} showing that contrastive learning recovers true data-generating factors up to an \emph{orthogonal} transformation. The 
representation may not come from a contrastive learning model or assumptions of \citet{zimmermann2021contrastive} might be violated in practice, thus we consider a more general class of linear invertible transformations in our model which we justify theoretically and verify empirically. Our contributions are summarized below:

\begin{itemize}
\item We formulate a simple linear model of entanglement in pre-trained visual representations and a corresponding method for Post-processing to Isolate Style and COntent (\method).
\item We establish theoretical guarantees that \method\ learns disentangled style and content factors and recovers correlations among styles. Our theory is supported by a synthetic dataset study.
\item We verify the ability of \method\ to disentangle style and content on three image datasets of varying size and complexity via post-processing of various pre-trained deep visual feature extractors. In our experiments, discarding the learned style factors yields significant out-of-distribution performance improvements while preserving the in-distribution accuracy.
\end{itemize}


\section{Problem formulation}

In light of the recent success of contrastive learning techniques for obtaining self-supervised embeddings, \citet{zimmermann2021contrastive} have performed a theoretical investigation on the InfoNCE family \citep{gutmann2012noise,oord2018representation,chen2020simple} of contrastive losses. Under some distributional assumptions, their investigation reveals that  InfoNCE loss can invert the underlying generative model of the observed data. More specifically, given an observed data $\bx = g(\bz)$ where $\bz$ and $g$ are correspondingly the underlying latent factors and the generative model, \citet{zimmermann2021contrastive} showed that InfoNCE loss finds a representation model $f$ such that $f\circ g(\bz) = \bR\bz$ for some orthogonal matrix $\bR$. Though this is quite welcoming news in nonlinear independent component analysis (ICA) literature \citep{hyvarinen1999nonlinear,hyvarinen2016unsupervised,jutten2010nonlinear}, the representation model may not be good enough for learning a disentangled representation. In fact, \citet{zimmermann2021contrastive} show that only under a very specific generative modeling assumption a type of contrastive objective can achieve disentanglement, and that disentanglement is lost when the assumptions are violated. 

At a high level, disentanglement in representation learning means the style and content factors are not affected by each other. Looking back at the result \citep{zimmermann2021contrastive} that contrastive loss can recover the generative latent factors up to an unknown rotation, an implication is that disentanglement in the learned representation may not be achieved. However, all is not lost; we suggest a simple post-processing method for the learned representations and show that it achieves the desired disentanglement.

We now formally describe our post-processing setup. We denote the space of the latent factor as $\cZ\subset \reals^{d}$ and assume that the latent factor is being generated from a  probability distribution $\bbP_\bz$ on $\cZ$. Similar to \citet{zimmermann2021contrastive} we assume that there exists a one-to-one generative map $g$ such that the observed data is generated as $\cX \ni \bx = g(\bz), ~ \bz \sim \bbP_\bz$. The next assumption is crucial for our linear post-processing technique and is motivated by the finding in \citet{zimmermann2021contrastive}. 
\begin{assumption}
\label{assmp:entangled-rep}
There exists a representation map $f : \cX \to \cZ' \subset \reals^{d'}$ such that $f\circ g(\bz) = \bA\bz$ for some left invertible matrix $\bA \in \reals^{d'\times d}$. 
\end{assumption} 
An example of such $f$ could be the representation model learned from InfoNCE loss minimization, where \citet{zimmermann2021contrastive} showed that the assumption is true for $\cZ' = \cZ$ and $\bA$ is an orthogonal matrix. A consequence of left invertibility for $\bA$ is that $d' \ge d$, \ie, \ the dimension of learned representation could be potentially higher than that of the generating latent factors, which is often natural to assume in many applications.  

Throughout the paper, we denote $f\circ g(\bz)$ as $\bu$ and call it \emph{entangled representation}. With this setup, we're now ready to formally specify the disentanglement (also known as sparse recovery) in representation learning. 

\begin{definition}[Disentangled representation learning/sparse recovery]
\label{def:direp}
Let us denote $\stylefactor \subset [d] \triangleq \{1, 2, \dots, d\}$ as the set of \textbf{style factors} and it's cardinality as $m \triangleq |\stylefactor|$. We denote the remaining factors $\contentfactor \triangleq [d] - \stylefactor$ and call them \textbf{content factors}.   For a matrix $\bP \in \reals^{d \times d'}$ we say \textbf{the linear post-processing $\bu \mapsto \hat \bz \triangleq \bP \bu$ disentangles or sparsely recovers the style and content factors} if the following hold for the matrix $\bP \bA$:
\begin{enumerate}
    \item $[\bP \bA]_{\stylefactor, \stylefactor}$ is an $m \times m$ diagonal matrix. 
    \item $[\bP \bA]_{\contentfactor, \contentfactor}$ is a $(d - m) \times (d - m )$ invertible matrix. 
    \item $[\bP \bA]_{\stylefactor, \contentfactor}$ and $[\bP \bA]_{\contentfactor, \stylefactor}$ are $m \times (d - m)$ and $(d - m) \times m$ null matrices. 
\end{enumerate}
In other words, $\hat \bz = \bP_\cS \bu$ is disentangled or sparsely recovered in $\stylefactor$ if for any $j \in \stylefactor$ the coordinate $[\hat \bz]_{j}$ is a constant multiplication of $[\bz]_j$ and $[ \hat \bz]_{\contentfactor}$ is just a pre-multiplication of  $[\bz]_{\contentfactor}$ by an invertible matrix. 
\end{definition}

Without loss of generality we assume that $\stylefactor = [m]$. Next, we highlight a conclusion of the sparse recovery, which has a connection to independent component analysis (ICA). 

\begin{corollary}[Correlation recovery]
\label{cor:correlation-recovery}
One of the conclusions of sparse recovery is that the estimated style factors have the same correlation structure as the true style factors. Denoting $\operatorname{corr}(\bX)$ as the correlation matrix for a generic random vector $\bX$ the conclusion can be mathematically stated as
\begin{equation} \label{eq:corr-recovery}
    \operatorname{corr}([\hat \bz]_{\stylefactor}) = \operatorname{corr}([\bz]_{\stylefactor})\,.
\end{equation}
A proof of the statement is provided in \S\ref{sec:proof-sparse-recovery}.
In a special case connected to ICA, where  the true correlation distribution $\bbP_\bz$ has uncorrelated style factors, \ie, $\operatorname{corr}([\bz]_{\stylefactor}) = \bI_{m}$, then same is true for estimated style factors. 
\end{corollary}


The rest of the paper describes the estimation of $\bP_{\cS}$ and investigates its quality in achieving disentanglement.  




\section{\method}
\label{sec:method}

To achieve sparse recovery, we assume that we can manipulate the samples in some specific ways, which we describe below.


\begin{assumption}
\label{assmp:sample-manipulation}
 We assume the following:
\begin{enumerate}
    \item \textbf{Sample manipulations:} For each sample $\bx = g(\bz)$ and style factor $j \in \stylefactor$ we have access to the sample $ \bx^{(j)}  \triangleq g(\bz^{(j)}) $ that has been created by modifying the $j$-th style factor of $\bx$ while keeping content factors unchanged, \ie, 
    \begin{equation}
    \label{eq:unchanged-content}
         [\bz^{(j)}]_i =   \begin{cases}
      \neq [ \bz]_i, & i = j, \\
      [ \bz]_i, & i \in \contentfactor\,.
    \end{cases}
    \end{equation}
    \item \textbf{Sample annotations:} There exist two numbers $\alpha_j, \beta_j \in \reals, ~~ \beta_j \neq 0$ which are associated to each $j$-th style factor and independent of the latent factors $\bz$ such that for each sample $\bx = g(\bz)$ and it's modified version $\bx^{(j)} = g(\bz^{(j)})$ we observe the sample annotations $\by_j = \alpha_j + \beta_j[\bz]_j + \epsilon^{(j)}$ and $\tilde \by_j = \alpha_j + \beta_j[\bz^{(j)}]_j + \tilde\epsilon^{(j)}$, where $(\epsilon^{(j)}, \tilde \epsilon^{(j)})$ pair has zero mean and is uncorrelated with $(\bz, \bz^{(j)})$. 

   
\end{enumerate}
\end{assumption}

Sample annotations formalize the notion of concept from interpretable ML \citep{kim2018Interpretability} and generalize the usual disentangled representation setting in which the latent factors are the concepts. By taking $\alpha_j = 0$ and $\beta_j = 1$, we have $\by_j = [\bz]_j$ and $\tilde{\by}_j= [\bz^{(j)}]_j$, which equates the annotations and the latent factors. We provide an illustration for a single style factor, \ie\  $\stylefactor = \{1\}$, in Figure \ref{fig:pisco_diagram}. Here, $\bx$ is the original image and we annotate it as $\alpha_1 + \beta_1[\bz]_1 = +1$. We stylize the image to obtain $\bx^{(1)}$ and assume that style transformation does not change any content factors of the image. We annotate the transformed image as $\alpha_1 + \beta_1[\bz^{(1)}]_1 = -1$. Examples of such sample manipulations are easily available in vision problems, \eg, image corruptions \citep{hendrycks2018Benchmarking} and style transfer \citep{huang2017arbitrary}. Combining style transfer and prompt-based image generation systems like DALL$\cdot$E 2 further enables using natural language to describe desired sample manipulations (Figure \ref{fig:pisco_diagram} illustrates such image manipulation - see \cref{supp:exp-features} for prompt and other details and Figure \ref{fig:stylized_images} for more examples). In \S\ref{sec:experiments} we use these examples for our experiments.




We denote the entangled representations (obtained from Assumption \ref{assmp:entangled-rep}) corresponding to the images $\bx$ and $\bx^{(j)}$ as $\bu$ and $\bu^{(j)}$.
With access to such sample manipulations, one can recover the $j$-th latent factor from a simple minimum norm least square regression problem: 
\begin{equation}
\label{eq:jth-direction}
    \begin{aligned}
\relax    [\hat \bz]_j \triangleq \hat\bp_j^\top \bu, ~~ \text{where} \\
    \hat \bp_j \triangleq \lim_{\mu \to 0+}\underset{a \in \reals, \bp\in \reals^{d'}}{\argmin}  \frac{1}{2n} \sum_{i = 1}^n \Big[(\by_i^{(j)} - a - \bp^\top \bu_i)^2\\
  + (\tilde \by_i^{(j)} - a - \bp^\top \bu_i^{(j)})^2\Big] + \frac\mu 2 \|\bp\|_2^2
    \end{aligned}
\end{equation} We resort to the minimum norm least square regression instead of the simple 
least square regression because the variance of the predictor $\text{var}(\bu) = \text{var}(\bA z) = \bA \text{var}(\bz) \bA^\top$ has rank $d \le d'$, which leads to non-invertible covariance for the design matrix whenever $d < d'$. 

Intuitively, $\hat \bp_j^\top \bu$ is the one dimensional linear function of $\bu$ which is most aligned to the coordinate $[\bz]_j$. As we shall see later, under our setup $\hat \bp_j^\top \bu$ is just a scalar multiple of $[\bz]_j$, and hence we successfully recover the $j$-th style factor. 
We stack $\hat \bp_j$ into the $j$-th row of $\bP$, \ie\ $[\bP]_{j, \cdot} = \hat\bp_j^\top$. 

To extract the content factors we first recall that they should exhibit minimal change corresponding to  any changes in the style factors $[\bz]_j, ~ j \in \stylefactor$. We enforce this  by leveraging our ability to manipulate styles of samples, as described in Assumption \ref{assmp:sample-manipulation}. We describe our method below: 

\paragraph{Estimation of content factors:} We recall $m = |\stylefactor|$ and let $\bU = [\bu_1, \dots, \bu_n, \bu_1^{(j)}, \dots, \bu_n^{(j)}; j \in \stylefactor]^\top \in \reals^{(m+1)n \times d'}$ be the matrix of entangled representations, and for each $j \in \cS$ let $\Delta_j = [\bu_1 - \bu_1^{(j)}, \dots, \bu_n - \bu_n^{(j)}]\in \reals^{n \times d'}$ be the matrix of representation differences. We estimate the content factors from the following optimization: 
\begin{equation}
\label{eq:loss}
\begin{aligned}
\relax [\hat{\bz}]_{\contentfactor} = \hat \bQ(\lambda) \bu, ~~ \text{where}\\
\textstyle \hat \bQ(\lambda) \triangleq \underset{\substack{\bQ\in \reals^{(d - |\stylefactor|)\times d'}\\ \bQ \bQ^\top = \bI}}{\argmin} \operatorname{tr} \left[\Big(\bI_{d'} - \bQ^\top \bQ\Big) \Big(\frac{\bU^\top \bU}{(m+1)n} \Big)\right]\\
+ \textstyle\frac\lambda{m}\sum_{j \in \stylefactor} \operatorname{tr} \left[\bQ^\top \bQ \big(\nicefrac{\Delta_j^\top \Delta_j}{n} \big)\right]\,.
\end{aligned}
\end{equation} Our objective has two parts: the first part is easily recognized by noticing its similarity to a principle component analysis objective. To understand the second part, we fix a style factor $j \in \stylefactor$ and observe that,
\begin{equation}
\label{eq:second-part-of-loss}
    \begin{aligned}
&\frac{1}{n}\sum_{i = 1}^n \|\bQ(\bu_i - \bu_i^{(j)})\|_2^2\\
&= \frac{1}{n}\sum_{i = 1}^n \operatorname{tr}\left[\bQ^\top \bQ(\bu_i - \bu_i^{(j)})(\bu_i - \bu_i^{(j)})^\top\right]\\
& = \operatorname{tr} \left[\bQ^\top \bQ \big(\nicefrac{\Delta_j^\top \Delta_j}{n} \big)\right]\,.
\end{aligned}
\end{equation}

Following the above, one can easily realize that the second part of the objective enforces that the content factors $[\hat{\bz}]_{\contentfactor}$ exhibit minimal change for any changes in the style factors $[\bz]_j, ~ j \in \stylefactor$. In a special case $\lambda = +\infty$, the  $[\hat{\bz}]_{\stylefactor}$ will be invariant to any changes in the style factors. 


From \eqref{eq:jth-direction} and \eqref{eq:loss} we obtain the linear post-processing matrix (as defined in \ref{def:direp}) as 
\begin{equation}
\label{eq:post-processing-matrix}
   \bP \equiv  \bP(\lambda) \triangleq \begin{cases}
    [\bP]_{j, \cdot} = \hat \bp_j^\top, & j \in \stylefactor\\
    [\bP]_{\contentfactor, \cdot} = \hat \bQ(\lambda)\,.
    \end{cases}
\end{equation} We summarize our method in Algorithm \ref{alg:method} which is a combination of simple regressions (per style factors) and an eigen-decomposition. In the next section, we show that for large values of $\lambda$ the post-processing matrix  $\bP(\lambda)$ achieves sparse recovery with high probability.





 
 


\begin{algorithm}[tb]
   \caption{\method}
   \label{alg:method}
\begin{algorithmic}
   \STATE {\bfseries Input:} \textbf{Dataset and styles:} (1) entangled representations $\{\bu_i\}_{i = 1}^n\subset \reals^{d'}$ of $n$ images, and (2) $m$ styles. 
\textbf{Hyperparameters:}  (1) regularization strength for disentanglement between style and content factors $\lambda > 0$, and (2) number of content factors $k$.

   \COMMENT{representations of } 
   \FOR{$j=1$ {\bfseries to} $m$}
        \FOR{$i=1$ {\bfseries to} $n$}
            \STATE    $\bu^{(j)}_i\gets$  entangled feature  of $i$-th image after changing it's $j$-th style.
           \STATE  $\delta_{i}^{(j)} \gets \bu_i^{(j)} -\bu_i$.
        \ENDFOR
        \STATE $\hat\bp_j\gets$ coefficient from regression \eqref{eq:jth-direction} on $\{(\bu_i, -1)\}_{i = 1}^n \cup \{(\bu_i^{(j)}, +1)\}_{i = 1}^n$.
   \ENDFOR
   \STATE  $\bU \gets [\bu_1, \bu_1^{(1)}, \dots, \bu_1^{(m)}, \dots, \bu_n, \bu_n^{(1)}, \dots, \bu_n^{(m)}]^\top\in \reals^{n(m+1)\times d'}$ 
 \STATE $\Delta \gets [ \delta_1^{(1)}, \dots, \delta_1^{(m)}, \dots,  \delta_n^{(1)}, \dots, \delta_n^{(m)}]^\top\in \reals^{mn\times d'}$
 
 \STATE $\hat \bQ(\lambda)\gets$ top $k$ eigenvectors of  $\frac{\bU^\top\bU}{n(m+1)} - \lambda \frac{\Delta^\top \Delta}{mn}$
 
  \STATE\textbf{Return:} Post-processing  matrix $\bP(\lambda)$ as in \eqref{eq:post-processing-matrix}.
\end{algorithmic}
\end{algorithm}
\section{Theory}
\label{section:theory}

In this section, we theoretically establish that our post-processing approach \method\ guarantees sparse recovery. We divide the proof into two parts: the first part analyzes the asymptotic quality of the estimated style factors and the second part analyzes the quality of the estimated content factors. Our first result follows: 

\begin{theorem}
\label{th:style-factors}
Let $\Sigma_\bz\triangleq \operatorname{var}(\bz)$ be invertible. Then for any $j \in \stylefactor$ it holds: 
\begin{equation}
    \bA^\top \hat \bp_j \to \beta_j e_j
\end{equation} almost surely as $n \to \infty$, where $\{e_j\}_{j = 1}^{d}$ is the canonical basis vector for $\reals^d$. Subsequently, the following hold at almost sure limit: (1) $[\bP \bA]_{\stylefactor, \stylefactor}$ converges to a diagonal matrix, and (2) $[\bP \bA]_{\stylefactor, \contentfactor} \to \mathbf{0}$.
\end{theorem}

To establish theoretical guarantee for the content factors we require the following technical assumption.

\begin{assumption}[Linear independence in style factors]
\label{assmp:linear-independence}
Over the  distribution of latent factors, the style factors are not linearly dependent with each other. Mathematically speaking, the following event positive probability
\begin{equation}
\label{eq:invertibility-perturbation}
    \bE = \big\{ \bz: [\bz - \bz^{(1)}, \dots, \bz - \bz^{(m)}]_{\stylefactor, \cdot} ~~ \text{is invertible} \big\}\,.
\end{equation}  
\end{assumption}

The assumption is related to cases when the style factors are dependent with each other. One such example is the blurring and contrasting of images: we assume that one doesn't \emph{completely} determine the other.  To see how the assumption is violated under linear dependence of style factors let the first two of them completely determine one another, \ie, one is just a constant multiplication of the other. In that case, we point out that for any $\bz$ the first two rows of the matrix $[\bz - \bz^{(1)}, \dots, \bz - \bz^{(m)}]_{\stylefactor, \cdot}$ are just constant multiplication of one another and the matrix is singular with probability one.

With the setups provided by Assumptions \ref{assmp:entangled-rep}, \ref{assmp:sample-manipulation} and \ref{assmp:linear-independence} we're now ready to state our result about sparse recovery for our post-processing technique. 

\begin{theorem}[Sparse recovery for \method]
\label{th:sparse-recovery}
Let $n \ge d + 1$ and that the Assumptions \ref{assmp:entangled-rep}, \ref{assmp:sample-manipulation} and \ref{assmp:linear-independence} hold.  Define $\kappa \triangleq  1- \bbP_\bz(\bE) < 1$ where $\bbP_\bz$ is the distribution of latent factors and $\bE$ is defined in \eqref{eq:invertibility-perturbation}. With probability at least $1 - \kappa^n$ the post-processing matrix $\bP \equiv  \bP(+\infty) \triangleq\underset{\lambda\to \infty}{\lim}\bP(\lambda)$ satisfies the following: (1) $[\bP \bA]_{\stylefactor, \contentfactor} = \mathbf{0}$, and (2) $[\bP \bA]_{\contentfactor, \contentfactor}$ is invertible. 
\end{theorem}


Proofs of Theorems \ref{th:style-factors} and \ref{th:sparse-recovery} are provided in \S \ref{sec:proof-style-factor} and \S\ref{sec:proof-sparse-recovery}. We combine the conclusions of the two theorems in the following corollary. 
\begin{corollary}
\label{cor:sparse-recovery}
Let the Assumptions \ref{assmp:entangled-rep}, \ref{assmp:sample-manipulation} and \ref{assmp:linear-independence} hold. Then at the limit $n \to \infty$ the the post-processing matrix $\bP \equiv  \bP(+\infty) $ almost surely satisfies sparse recovery conditions in Definition \ref{def:direp}. 
\end{corollary}


\subsection{Synthetic data study}
\label{sec:sim}
We complement our theoretical study with an experiment in a synthetic setup. Below we describe the data generation, sample manipulations, and their annotations, and provide their detailed descriptions in \S\ref{sec:sim-supp}.

We generate \textbf{the latent variables} ($\bz$) from a 10-dimensional centered normal random variable, where the coordinates have unit variance, the first two coordinates are correlated with correlation coefficient $\rho$ and all the other cross-coordinate correlations are zero.  We consider the first five coordinates of $\bz$ as the style factors, \ie\ $\stylefactor = \{1, 2, \dots, 5\}$, and the rest of them as content factors. 

\textbf{The entangled representations}   are  $d' = 10$ dimensional vectors and which we obtain as $\bu = \bA \bz = \bL \bU \bz$, where $\bL$ is a $10 \times 10$ lower triangular matrix whose diagonal entries are one and off-diagonal entries are $0.9$ and $\bU$ is a randomly generated $d\times d$ orthogonal matrix. 

\textbf{Sample manipulations and annotations:} For $j$-th style coordinates we obtain two  {manipulated samples} per latent factor $\bz$, which (denoted as $\bz^{(j), +}$ and $\bz^{(j), -}$ ) set the $j$-th coordinate to it's positive (resp. negative) absolute value, \ie\ $[\bz^{(j), +}]_j = |[\bz]_j|$ (resp. $[\bz^{(j), -}]_j = -|[\bz]_j|$), and annotate it as $+1$ (resp. $-1$). 
Since the first two coordinates have correlation coefficient $\rho$, if either of them is changed by the value $\delta$ then the other one must be changed by $\rho \delta$. 
Note that one of $\bz^{(j), +}$ and $\bz^{(j), -}$ is exactly equal to $\bz$. The corresponding entangled representations to the manipulated latent factors are used for recovering the style factors (as in \eqref{eq:jth-direction} and \eqref{eq:post-processing-matrix}), and content factors (as in \eqref{eq:loss}). 

\begin{figure}
    \centering
    \includegraphics[scale = 0.5]{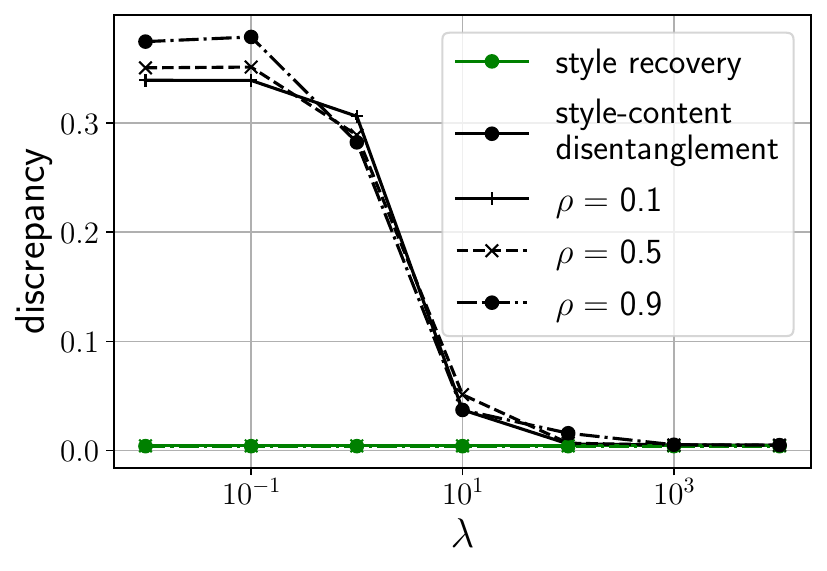}
    \caption{Plots (with error-bars over 50 repetitions) for discrepancies in style recovery ($\|\operatorname{corr}([\bz]_{\stylefactor}, [\hat\bz]_{\stylefactor}) - \operatorname{corr}([\bz]_{\stylefactor})\|_{\operatorname{F}}$) and style-content disentanglement ($\|\operatorname{corr}([\hat \bz]_{\contentfactor}, [\bz]_{\stylefactor})\|_{\operatorname{F}}$) for estimated factors, where $\|\cdot\|_{\operatorname{F}}$ is the normalized Frobenius norm (see \cref{normalized-frobenius}). Here, $\rho$ is the correlation between the first two coordinates in true factors. }
    \label{fig:sim-disentanglement}
\end{figure}

\textbf{Style recovery:} In our synthetic experiments we validate the quality of sparse recovery for estimated latent factors on two fronts: (1) recovery in the style factors, and (2) disentanglement between style and content factors. To verify recovery in style factors we recall  Theorem \ref{th:style-factors} that the estimated style factors ($[\hat \bz]_{\stylefactor}$) approximate the true style factors ($[ \bz]_{\stylefactor}$) up to  constant multiplications. This implies that the cross-correlation between estimated and true style factors ($\operatorname{corr}([\bz]_{\stylefactor}, [\hat\bz]_{\stylefactor})$) should be approximately identical to the correlation of the true style factors ($\operatorname{corr}([\bz]_{\stylefactor})$). In Figure \ref{fig:sim-disentanglement} we verify this by calculating $\|\operatorname{corr}([\bz]_{\stylefactor}, [\hat\bz]_{\stylefactor}) - \operatorname{corr}([\bz]_{\stylefactor})\|_{\operatorname{F}}$ where $\|\cdot\|_{\operatorname{F}}$ is the normalized Frobenius norm of a matrix.\footnote{\label{normalized-frobenius}The normalized Frobenius norm of a matrix $\bA\in \reals^{m\times n}$  is denoted as $\|\bA\|_{\operatorname{F}}$ and defined as $\|\bA\|_{\operatorname{F}} \triangleq \sqrt{\frac{\sum_{i,j}[\bA]_{i,j}^2}{mn}}$.} We refer to it as the \emph{discrepancy in style recovery} and observe that it is small and not affected by $\rho$. Even for $\rho$ as large as $0.9$ the recovery of style factors has small discrepancy, which matches with the conclusion of Theorem \ref{th:style-factors}.  Additionally, we observe that the discrepancies are the same for different values of $\lambda$ (that appears in \eqref{eq:loss}) since the estimation of the style factors doesn't involve $\lambda$. 

\textbf{Style and content disentanglement:} Note that the style and content factors are uncorrelated with each other, \ie\ $\operatorname{corr}([\bz]_{\stylefactor}, [\bz]_{\contentfactor}) = \mathbf{0}$. If the content factors and style factors are truly disentangled then the cross-correlation between estimated content factors $[\hat \bz]_{\contentfactor}$ and true style factors $[\bz]_{\stylefactor}$ should be approximately equal to zero. In Figure \ref{fig:sim-disentanglement} we verify this by plotting $\|\operatorname{corr}([\hat \bz]_{\contentfactor}, [\bz]_{\stylefactor})\|_{\operatorname{F}}$, which we refer to as the \emph{discrepancy in style-content disentanglement} (SCD) and notice that for large enough values of the parameter $\lambda$ (\ie\ $\lambda > 100$) the discrepancy is quite small. Though $\rho$ has a mild effect on disentanglement between style and content factors for smaller values of $\lambda$, the effect is indistinguishable for large $\lambda$ ($\lambda > 10^3$).

\section{Experiments}
\label{sec:experiments}

We verify the ability of \method\ (Algorithm \ref{alg:method}) to isolate content and style in pre-trained visual representations and the utility of the learned representations for OOD generalization when (i) train data is spuriously correlated with the style and the correlation is reversed in the test data; (ii) test data is modified with various image transformations, i.e., domain generalization with style-based distribution shifts. We consider nine transformations in our experiments: four types of image corruptions (rotation, contrast, blur, and saturation) on CIFAR-10 \citep{krizhevsky2009learning}, similar to ImageNet-C \citep{hendrycks2018Benchmarking}, four transformations based on style transfer \citep{huang2017arbitrary} on ImageNet \citep{russakovsky2015imagenet}, similar to Stylized ImageNet \citep{geirhos2018imagenet}, and a color transformation on MNIST, similar to Colored MNIST \citep{arjovsky2019Invariant} (see \S\ref{supp:mnist_results} for Colored MNIST experiment). The experiments code is available on GitHub.\footnote{Code: \href {https://github.com/lilianngweta/PISCO}{github.com/lilianngweta/PISCO}.}

\subsection{Transformed CIFAR}

In this set of experiments, our goal is to disentangle four styles ($m=4$) corresponding to image corruptions (rotation, contrast, blur, and saturation) from content.
For feature extraction we consider a ResNet-18 \citep{he2016Deep} pre-trained on ImageNet \citep{russakovsky2015imagenet} (\texttt{Supervised}) and a SimCLR \citep{chen2020simple} trained on CIFAR-10 via self-supervision with the same architecture (\texttt{SimCLR}). For each feature extractor, we learn a \emph{single} \method\ post-processing feature transformation matrix $\bP(\lambda)$ as in Algorithm \ref{alg:method} to jointly disentangle all considered styles from content. We report results for $\lambda \in \{1, 10, 50\}$.\footnote{In all experiments we set the number of content factors to $k=\eta d' - m$, where $d'$ is the representation dimension. We set $\eta = 0.95$ for all experiments in the main paper and report results for other values of $\eta$ in \S\ref{supp:results}. As long as $\eta$ is close to 1, baselines and PISCO in-distribution results are similar. For smaller values of $\eta$, PISCO in-distribution accuracy naturally deteriorates.}

\paragraph{Baselines} Our main baseline is the vanilla \texttt{SimCLR} representations due to \citet{kugelgen2021self} who argued that it is sufficient for style and content disentanglement under some assumptions. Thus we study whether we can further improve style-content disentanglement in \texttt{SimCLR} in a real data setting in addition to experiments with features obtained via supervised pretraining on ImageNet. We also compare PISCO's style-content disentanglement with IP-IRM \citep{wang2021self}, which is an \emph{in-processing}
method combining self-supervised learning and invariant risk minimization \citep{arjovsky2019Invariant} to learn disentangled representations. We use IP-IRM model trained on CIFAR-100 provided by the authors.

We note that there are many other methods for learning disentangled representations (\citet{wu2019disentangling,nemeth2020adversarial,ren2021rethinking,kugelgen2021self}, to name a few), however, they all require training an encoder-decoder model from scratch and can not take advantage of powerful feature extractors pre-trained on large datasets as in our setting. In comparison to these works, the simplicity and scalability of our method (as well as of using vanilla SimCLR features) come at a cost, i.e., we forego the ability to visualize disentanglement via controlled image generation due to the absence of a generator/decoder. Instead, we demonstrated disentanglement theoretically (\S\ref{section:theory}) and verify it empirically via correlation analysis of learned style and content factors, similar to prior works that studied disentanglement in settings without a generator/decoder \citep{zimmermann2021contrastive,kugelgen2021self}.

\paragraph{Disentanglement} In Table \ref{tb:correlations} we summarize the disentanglement metrics for the smallest considered $\lambda=1$. In the style correlation columns (Style Corr.), we report the correlation between the corresponding style value (encoded as $-1$ for the original images and $+1$ for the transformed ones) and the factor corresponding to style in the learned representations. None of the baselines explicitly identify style factors, thus we use the coordinate maximally correlated with the corresponding style as the style factor.

We notice that the blur style is the hardest to learn for both supervised and unsupervised representations. As we will see later, both representations are fairly invariant to this style. Comparing \method\ on \texttt{Supervised} and \texttt{SimCLR}, the style recovery is better on \texttt{Supervised} since \texttt{SimCLR} representations are more robust to style changes \citep{kugelgen2021self}. 

In the style-content disentanglement (SCD) columns, we report the disentanglement of style from content features as in the synthetic experiment in Figure \ref{fig:sim-disentanglement}. Here \texttt{SimCLR} representations appear slightly harder to disentangle using PISCO than \texttt{Supervised} representations. In the SCD of original representations for both \texttt{Supervised} and \texttt{SimCLR}, as expected, we observe that these representations are more entangled with the styles, especially the \texttt{Supervised} representations. Comparing the SCD for PISCO with that of IP-IRM, we see that PISCO can post-process popular pre-trained representations to achieve comparable or better disentanglement without re-training (i.e., in-processing). Overall we conclude that \method\ is successful in isolating style and content.



\begin{figure}
\captionsetup[subfigure]{justification=centering}
     \centering
     \begin{subfigure}[b]{0.22\textwidth}
         \centering
         \includegraphics[width=\textwidth]{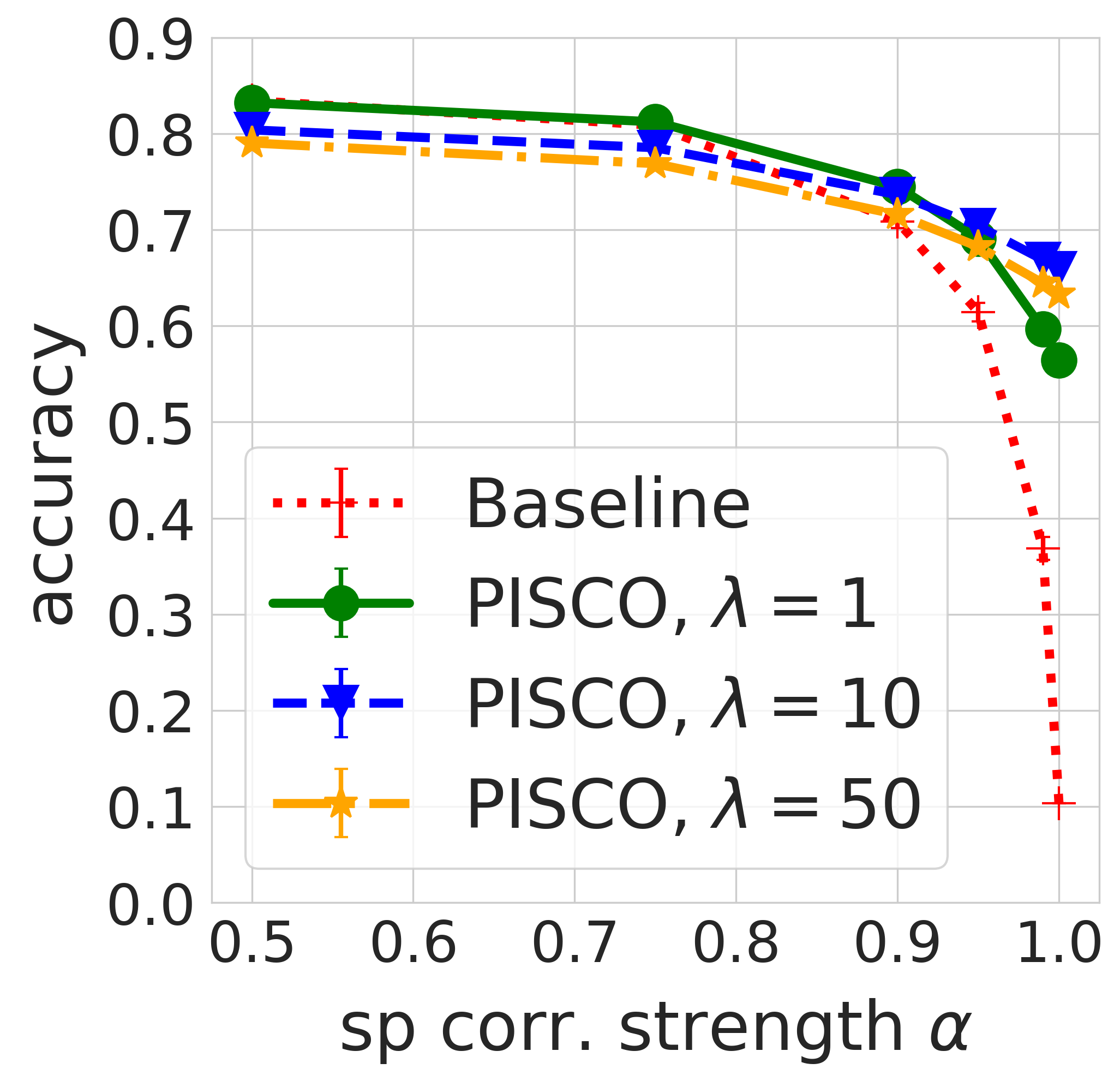}
         \caption{Rotation - \texttt{Supervised}}
         \label{fig:rotat-resnet}
     \end{subfigure}
     \hfill
     \begin{subfigure}[b]{0.22\textwidth}
         \centering
         \includegraphics[width=\textwidth]{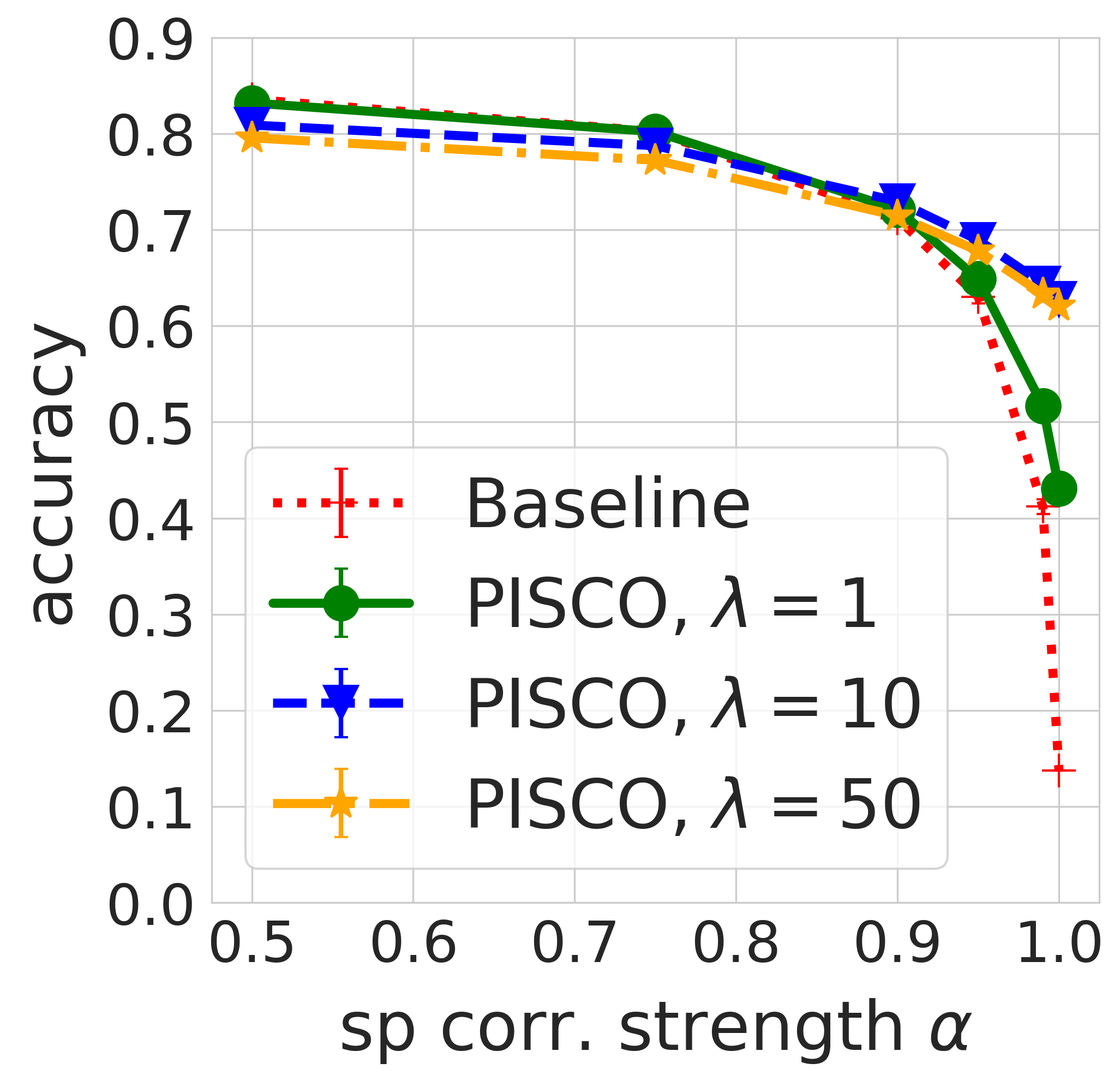}
         \caption{Contrast - \texttt{Supervised}}
         \label{fig:contr-resnet}
     \end{subfigure}
     \hfill
     \begin{subfigure}[b]{0.22\textwidth}
         \centering
         \includegraphics[width=\textwidth]{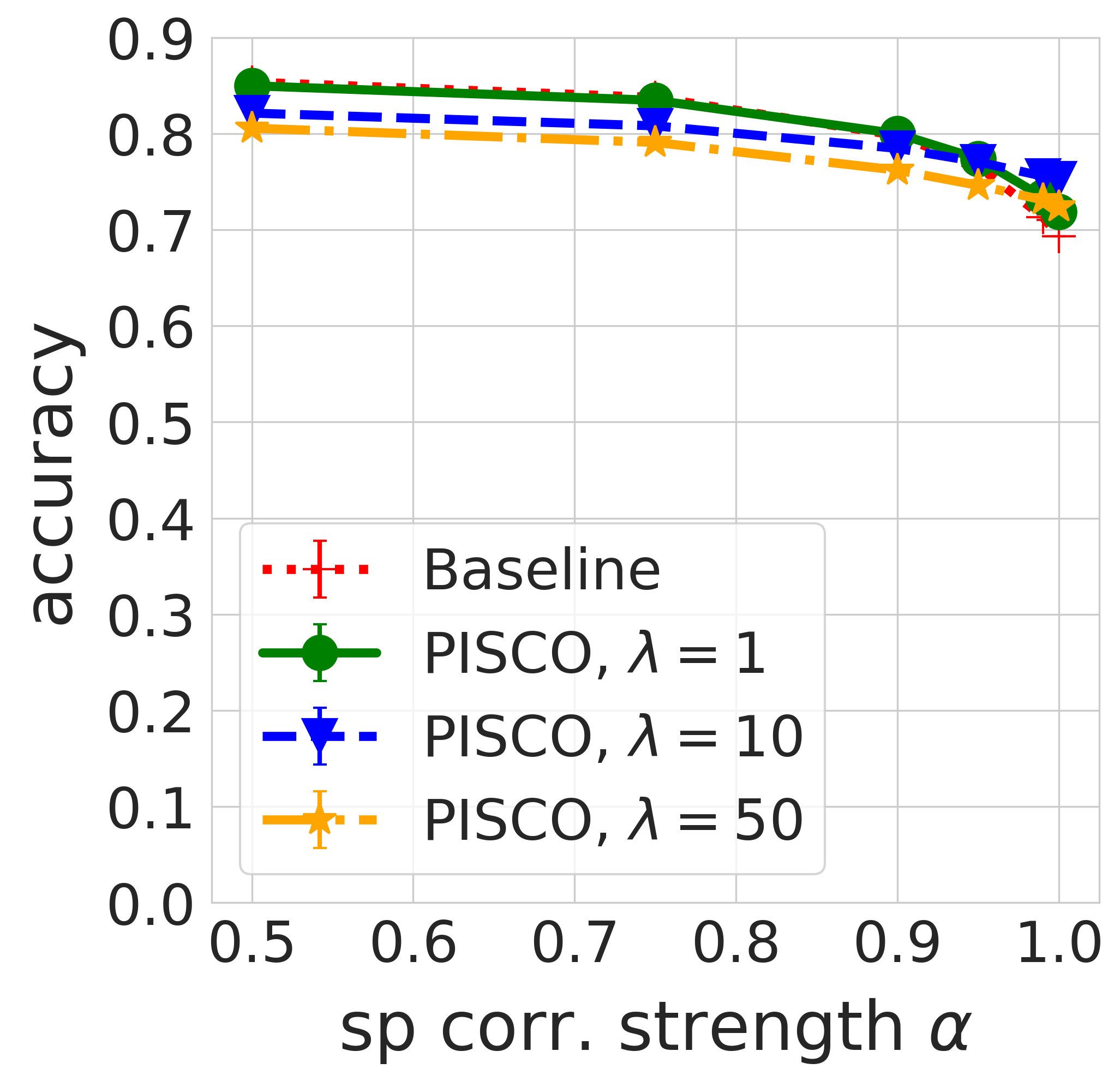}
         \caption{Blur - \texttt{Supervised}}
         \label{fig:blur-resnet}
     \end{subfigure}
     \hfill
     \begin{subfigure}[b]{0.22\textwidth}
         \centering
         \includegraphics[width=\textwidth]{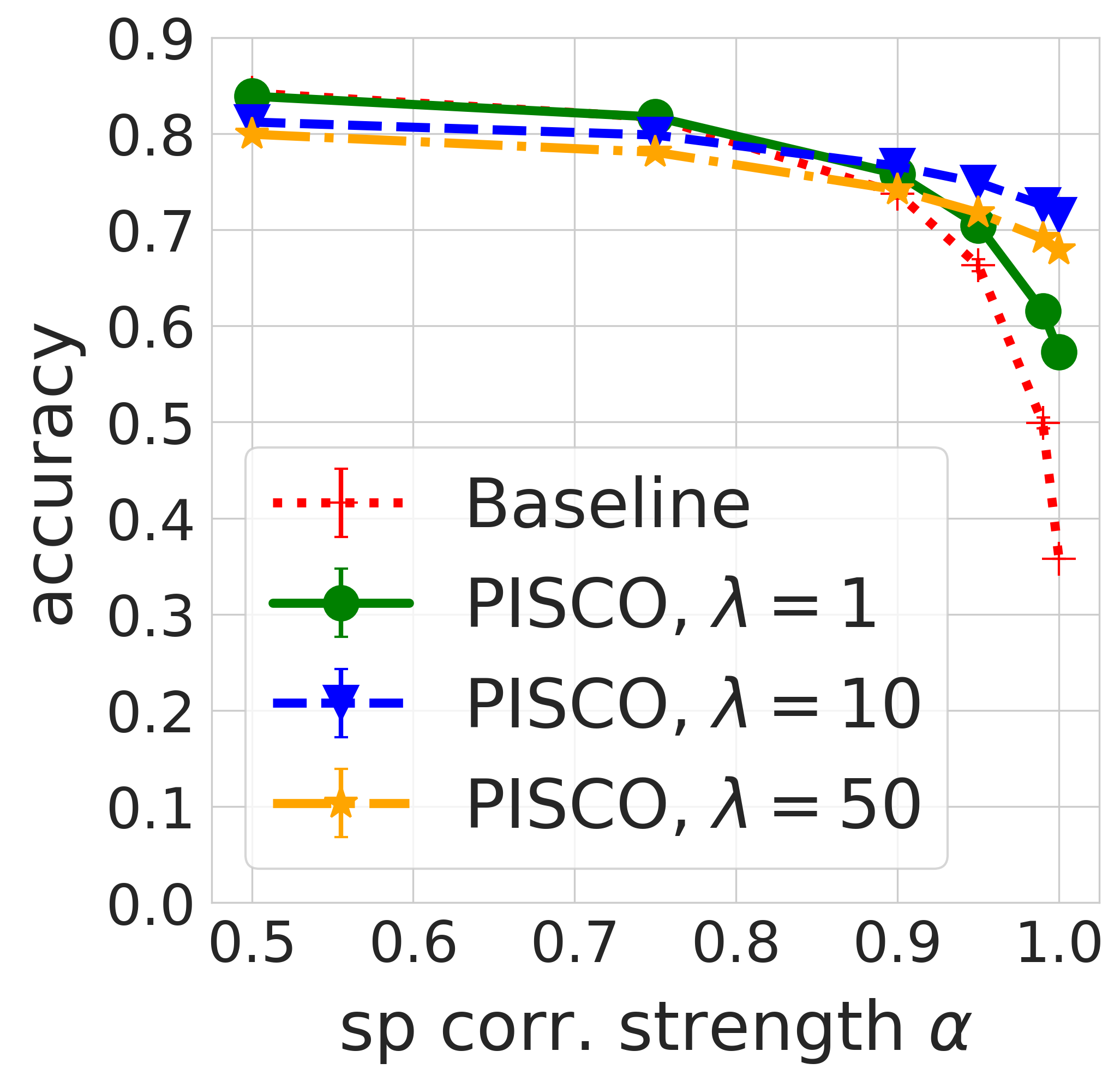}
         \caption{Satur. - \texttt{Supervised}}
         \label{fig:sat-resnet}
     \end{subfigure}
         \centering
         \caption{OOD accuracy of \texttt{Supervised} representations on CIFAR-10 where the label is spuriously correlated with the corresponding transformation. \method\ significantly improves OOD accuracy, especially in the case of rotation. Both $\lambda=1$ and $\lambda=10$ preserve in-distribution accuracy, while larger $\lambda=50$ may degrade it as per \eqref{eq:loss}.
         }
         \label{fig:resnet_results}

\end{figure}

\begin{figure}
\captionsetup[subfigure]{justification=centering}
     \centering
     \begin{subfigure}[b]{0.22\textwidth}
         \centering
         \includegraphics[width=\textwidth]{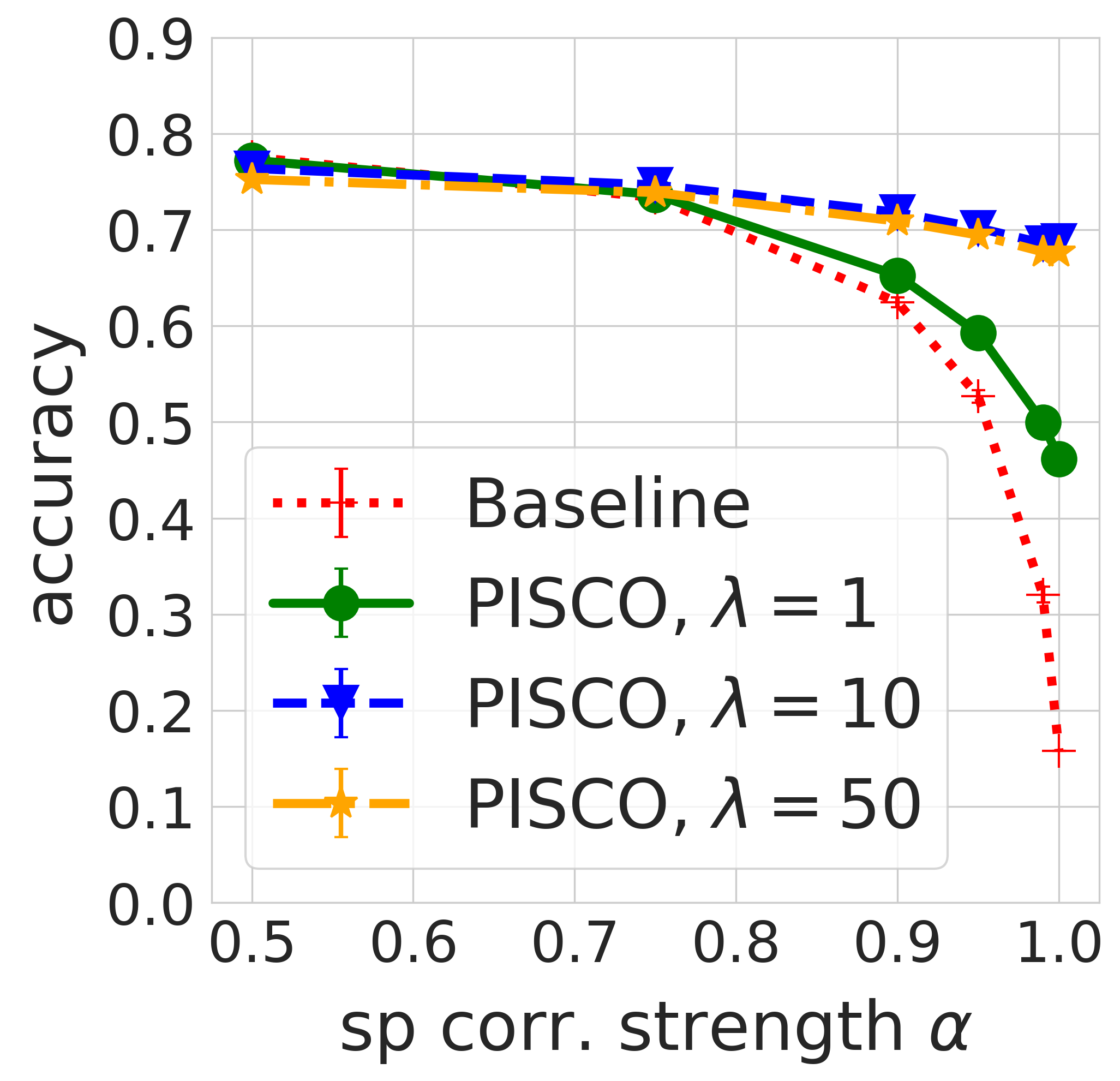}
         \caption{Rotation - \texttt{SimCLR}}
         \label{fig:rotat-simclr}
     \end{subfigure}
     \hfill
     \begin{subfigure}[b]{0.22\textwidth}
         \centering
         \includegraphics[width=\textwidth]{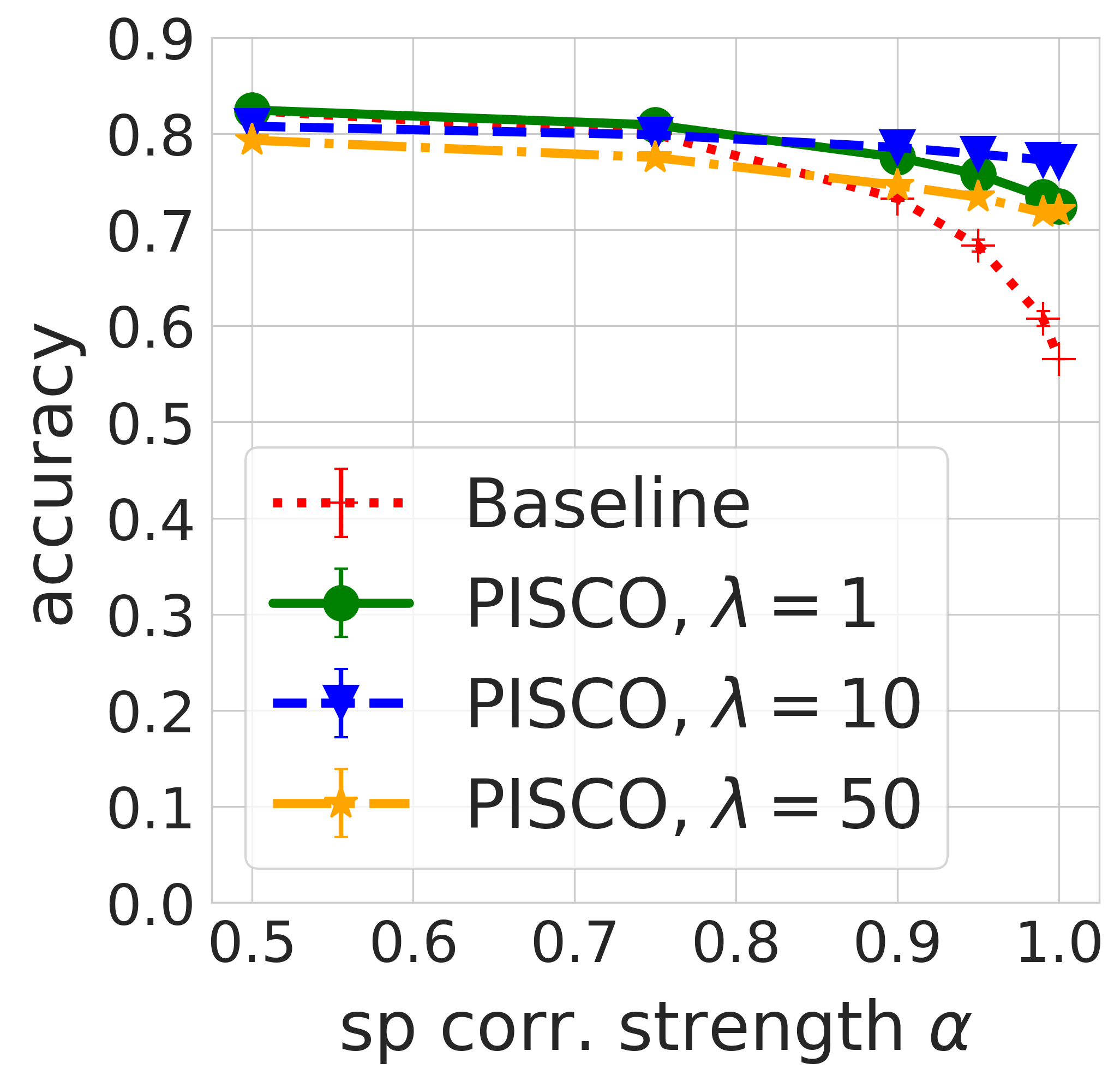}
         \caption{Contrast - \texttt{SimCLR}}
         \label{fig:contr-simclr}
     \end{subfigure}
     \hfill
     \begin{subfigure}[b]{0.22\textwidth}
         \centering
         \includegraphics[width=\textwidth]{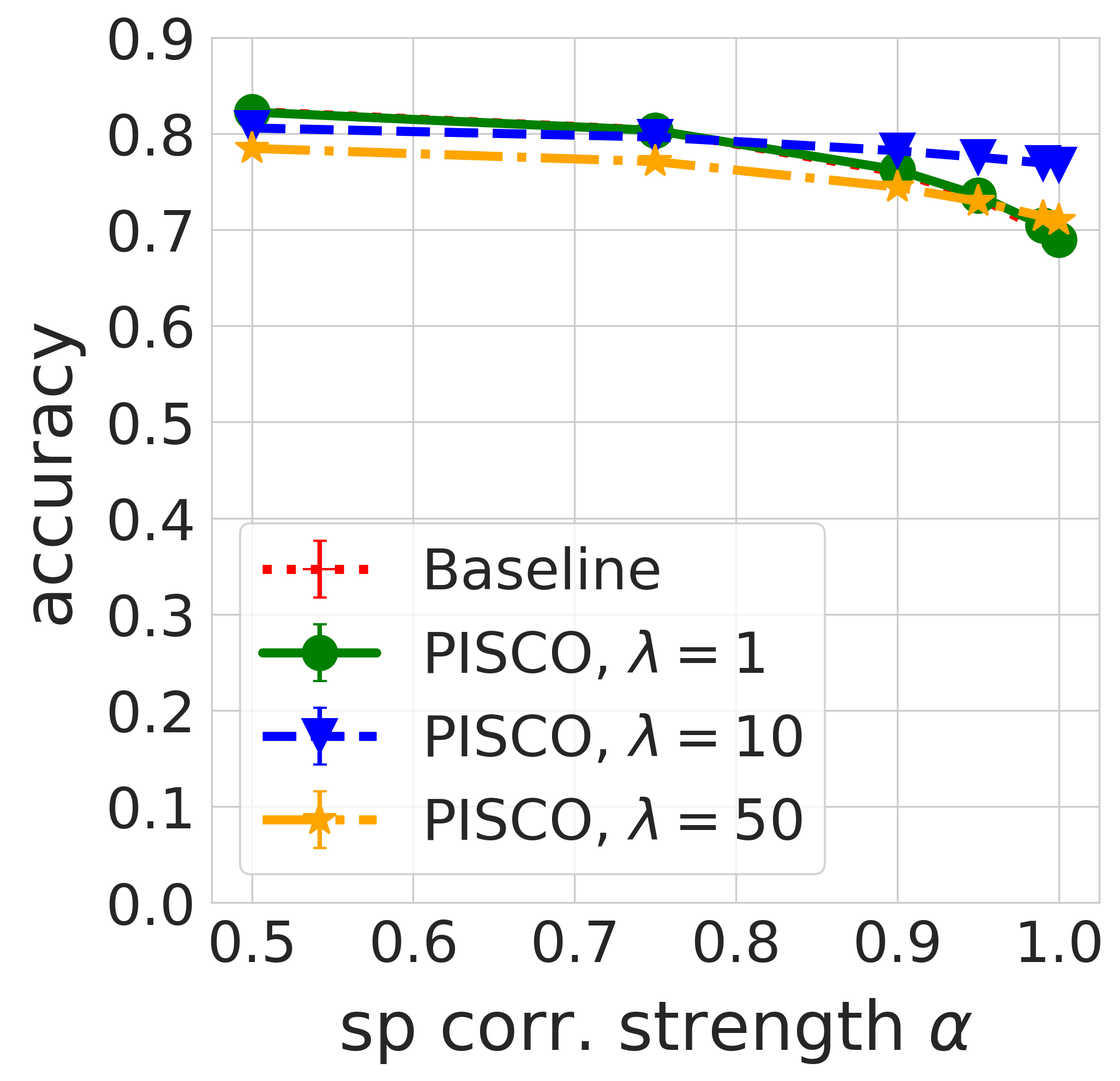}
         \caption{Blur - \texttt{SimCLR}}
         \label{fig:blur-simclr}
     \end{subfigure}
     \hfill
     \begin{subfigure}[b]{0.22\textwidth}
         \centering
         \includegraphics[width=\textwidth]{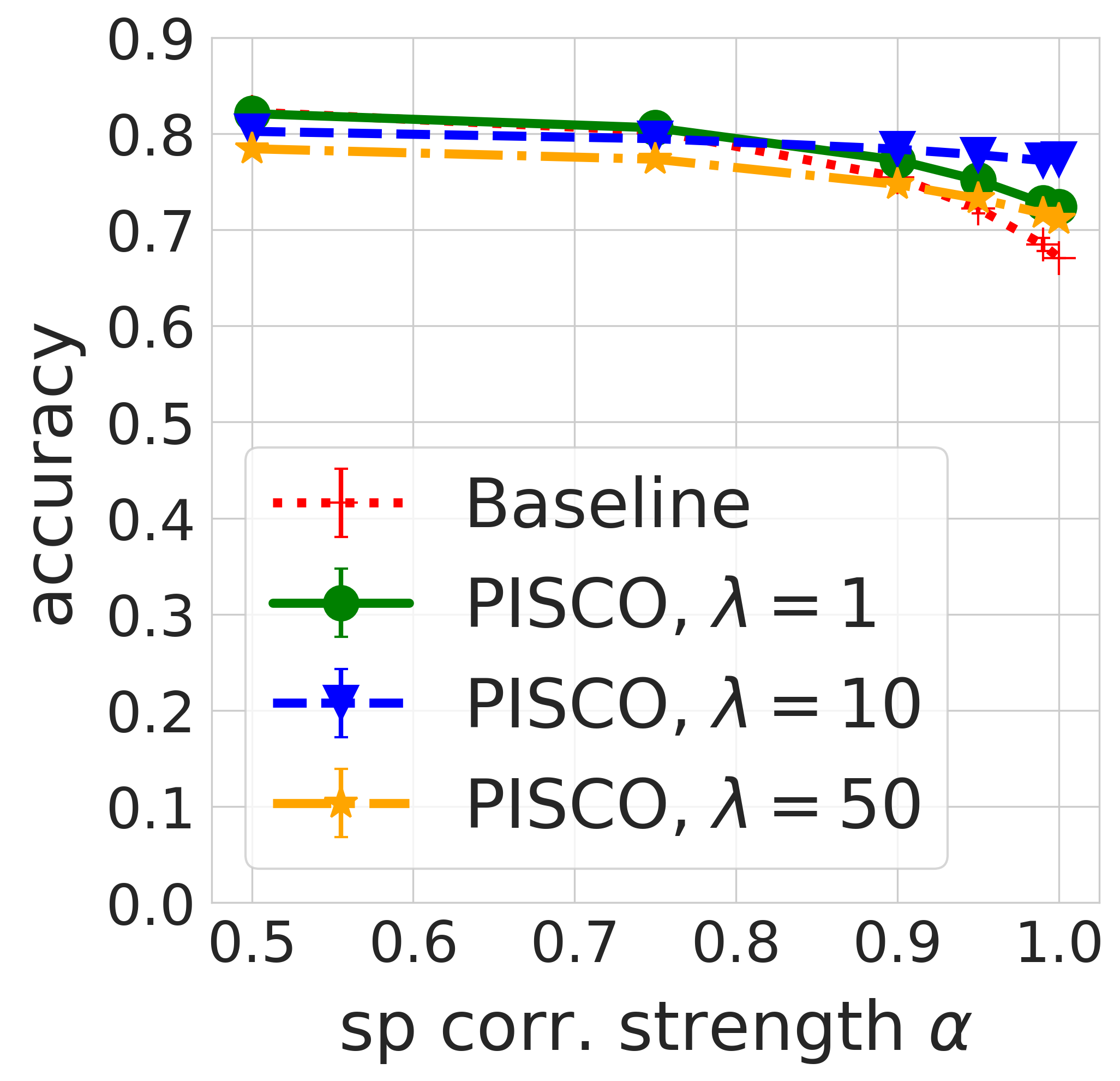}
         \caption{Saturation - \texttt{SimCLR}}
         \label{fig:sat-simclr}
     \end{subfigure}
         \centering
         \caption{OOD accuracy of \texttt{SimCLR} representations on CIFAR-10 where the label is spuriously correlated with the corresponding transformation. Results are analogous to Figure \ref{fig:resnet_results}. The \texttt{SimCLR} baseline representations are less sensitive to contrast and saturation but remain sensitive to rotation.
         }
         \label{fig:simclr_results}
        
\end{figure}

\begin{table*}
 \caption{Content and style disentanglement of \texttt{Supervised} and \texttt{SimCLR} representations on CIFAR-10 with \method. The style correlation columns (style Corr.) show correlations between styles in the data and representations corresponding to style. The isolation of style from content is measured with style-content disentanglement (SCD, see Figure \ref{fig:sim-disentanglement}). In the last column we compare to representations learned by IP-IRM.
 Bold denotes best results.
 }
 \label{tb:correlations}
 \centering
\begin{tabular}{l|cc|cc||cc|cc|cc} 
\noalign{\hrule height 1.5pt}

\multirow{ 3}{*}{\makecell{Style}} & \multicolumn{4}{c||}{\makecell{Supervised}} & \multicolumn{6}{l}{\makecell{Unsupervised}}   \\ \cline{2-5} \cline{6-11}\cline{2-5} \cline{6-11}
{} & \multicolumn{2}{c|}{\makecell{\texttt{Supervised}}}& \multicolumn{2}{c||}{\makecell{\method}} & \multicolumn{2}{c|}{\makecell{\texttt{SimCLR}}}& \multicolumn{2}{c|}{\makecell{\method}} & \multicolumn{2}{c}{\makecell{IP-IRM}} \\\cline{2-3} \cline{4-5} \cline{6-7} \cline{8-11}

{} & \makecell{Style\\Corr.} &  \makecell{SCD} & \makecell{Style\\Corr.} & \makecell{SCD} & \makecell{Style\\Corr.} &  \makecell{SCD} & \makecell{Style\\Corr.} &  \makecell{SCD} & \makecell{Style\\Corr.} &  \makecell{SCD}  \\
\noalign{\hrule height 0.5pt}
blur                &                0.319     &          0.090       &            \textbf{0.716}      &    \textbf{0.051}        &            0.304    &    0.096   &        \textbf{0.719}   &     0.060  &        0.032   &  \textbf{0.022}  \\
contrast           &                0.490    &          0.243       &            \textbf{0.927}       &     \textbf{0.055}       &            0.094     &    0.076   &        \textbf{0.897}    &     \textbf{0.049}  &        0.419   &     0.188 \\
rotation              &                0.746    &         0.212       &            \textbf{0.936}     &      \textbf{0.029}       &            0.368    &   0.182    &        \textbf{0.945}    &     \textbf{0.056}   &        0.617   &     0.114 \\
saturation              &                0.641   &          0.204      &             \textbf{0.882}      &      \textbf{0.048}       &            0.120  &   0.071    &         \textbf{0.738}    &     0.060   &        0.219   &     \textbf{0.044} \\

\noalign{\hrule height 0.5pt}
\end{tabular}
\end{table*}

\paragraph{Spurious correlations}
Next, we create four variations of CIFAR-10 where labels are spuriously correlated with one of the four styles (image corruptions). Specifically, in the training dataset, we corrupt images from the first half of the classes with probability $\alpha$ and from the second half of the classes with probability $1-\alpha$. In test data the correlation is reversed, i.e., images from the first half of the classes are corrupted with probability $1-\alpha$ and images from the second half with probability $\alpha$ (see \S\ref{supp:exp-details} for details). Thus, for $\alpha=0.5$ train and test data have the same distribution where each image is randomly transformed with the corresponding image corruption type, and $\alpha=1$ corresponds to the extreme spurious correlation setting.

For each $\alpha$ we train and test a linear model on the original representations and on \method\ representations (in this and subsequent experiments all learned style factors are discarded for downstream tasks; see \S\ref{supp:exp-details} for additional details) for varying $\lambda$. Recall that here we use the same \method\ transformation matrices learned previously without knowledge of the specific corruption type and spurious correlation value $\alpha$ of a given dataset.
We summarize results for \texttt{Supervised} features in Figure \ref{fig:resnet_results} and for \texttt{SimCLR} features in Figure \ref{fig:simclr_results}.
\method\ improves upon both original representations and across all transformations. For $\lambda=1$, \method\ always preserves the in-distribution accuracy, i.e. when $\alpha=0.5$, and improves upon the baselines in the presence of spurious correlations. Larger $\lambda=50$ can degrade in-distribution accuracy in some cases (recall that $\lambda$ controls the tradeoff between the reconstruction of the original features with the content factors and style-content disentanglement per \eqref{eq:loss}), while $\lambda=10$ provides a favorable tradeoff with a small reduction of in-distribution accuracy and large improvements when spurious correlations are present.

Comparing results across the representations, we notice that \texttt{SimCLR} features are less sensitive to image transformations as discussed previously. However, for both representations, spurious correlation with rotation causes a significant accuracy drop without \method\ post-processing.


\begin{table}
 \caption{Accuracy of \texttt{Supervised} representations on CIFAR-10 test set in-distribution, i.e., no transformation (referred to as ``none''; last row), and OOD when modified with the corresponding transformation. \method\ with $\lambda=1$ provides significant improvements for rotation, contrast, and saturation while preserving in-distribution accuracy.
 }
 \label{tb:ood_gen_resnet}
 \centering
 
\begin{tabular}{lccccc}
\toprule
Style &  \makecell{Baseline\\\scriptsize{(\texttt{Supervised})}} &  \makecell{\method\\(\small{$\lambda = 1$})} &  \makecell{\method\\(\small{$\lambda = 10$})} &  \makecell{\method\\(\small{$\lambda = 50$})} \\
\midrule
      rotation &                0.678 &                \textbf{0.737} &                 0.733 &                 0.710 \\
    contrast &                0.625 &                0.683 &                 \textbf{0.744} &                 0.726 \\
      saturation &                0.699 &                \textbf{0.758} &                 0.745 &                 0.721 \\
      blur &                \textbf{0.817} &                \textbf{0.817} &                 0.793 &                 0.775 \\
     
     none &  \textbf{0.873}      & 0.870  &  0.844   & 0.826      \\
\bottomrule
\end{tabular}
\vspace{-0.2cm}
\end{table}

\paragraph{Domain generalization} To evaluate the domain generalization performance, we train a logistic regression classifier on the corresponding representation of the clean CIFAR-10 dataset and compute accuracy on the test set with every image transformed with one of the four corruptions, as well as the original test set to verify the in-distribution accuracy. Results are presented in Table \ref{tb:ood_gen_resnet} for \texttt{Supervised} features and in Table \ref{tb:ood_gen_simclr} for \texttt{SimCLR} features. We observe significant OOD accuracy gains when applying \method\ post-processing on the \texttt{Supervised} features while preserving the in-distribution accuracy for $\lambda=1$. In this experiment, we see that \texttt{SimCLR} features are sufficiently robust and perform as well as \method\ post-processing with $\lambda=1$. Overall we have observed that applying our method with smaller $\lambda=1$ never hurts the performance, while it yields significant OOD accuracy gains in many settings.






\begin{table}
 \caption{
 Accuracy of \texttt{SimCLR} representations on CIFAR-10 test set in-distribution, i.e., no transformation (referred to as ``none''; last row), and OOD when modified with the corresponding transformation. \texttt{SimCLR} features are robust to the considered transformations and perform similarly to \method\ with $\lambda=1$.
 }
 \label{tb:ood_gen_simclr}
 \centering

\begin{tabular}{lccccc}
\toprule
 Style &  \makecell{Baseline\\\scriptsize{(SimCLR)}} &  \makecell{\method\\(\small{$\lambda = 1$})} &  \makecell{\method\\(\small{$\lambda = 10$})} &  \makecell{\method\\(\small{$\lambda = 50$})} \\
\midrule
       rotation &             0.620 &                0.625 &                \textbf{0.697} &                0.696 \\
    contrast &             \textbf{0.816} &                0.814 &                0.806 &                0.794 \\
      saturation &             \textbf{0.810} &                0.806 &                0.789 &                0.774 \\
      blur &             \textbf{0.808} &                0.801 &                0.793 &                0.780 \\
      
      none &     \textbf{0.828}    & 0.827	  & 0.808    & 0.792      \\
      
\bottomrule
\end{tabular}
\vspace{-0.2cm}
\end{table}

\subsection{Stylized ImageNet}
In this experiment, we evaluate the domain generalization of \method\ on more sophisticated styles obtained via style transfer \citep{huang2017arbitrary}, similar to the Stylized ImageNet \citep{geirhos2018imagenet} dataset. In addition, we evaluate the ability of \method\ to generalize to styles that are similar to but weren't used to fit \method. We used the ``dog sketch'' and ''Picasso dog'' styles to obtain \method\ transformation and evaluate on two additional similar but unseen styles, ``woman sketch'' and ``Picasso self-portrait''. See Figure \ref{fig:stylized_images} and \S\ref{supp:exp-details} for visualization and additional details.

As in the CIFAR-10 domain generalization experiment, the logistic regression classifier is trained on the original train images and tested on transformed test images. In Table \ref{tb:ood_gen_imagenet} we report results for ResNet-50 features pre-trained on ImageNet (Baseline) and for the same features transformed with \method\ with $\lambda=1$. \method\ improves OOD top-1 and top-5 accuracies across all four styles, including the unseen ones, while maintaining good in-distribution performance. We also report analogous results for another popular feature extractor, MAE-ViT-Base \citep{he2022masked}, in Table \ref{tb:ood_mae_vit_1}. We again observe that \method\ ($\lambda=1$) improves top-1 and top-5 OOD performances with no degradation of the in-distribution performance. We present results for other values of $\lambda$ in \S\ref{supp:results}.

We note that in this experiment the \emph{sample manipulations and annotations} required for our method (\S\ref{sec:method}) were simple to obtain. We generated the styles for fitting \method\ with basic text prompts using DALL$\cdot$E 2 and obtained pairs of original and transformed images using a style transfer method \cite{huang2017arbitrary}. Thus, this experiment demonstrates how \method\ can be applied to improve robustness to a variety of distribution shifts in vision tasks where we have some amount of prior knowledge needed to formulate a relevant prompt to obtain a style image. 

\begin{table}
 \caption{Top-1 and top-5 accuracies on 5 variations of the ImageNet test set for Baseline pre-trained ResNet-50 features and the corresponding post-processed \method\, ($\lambda=1$) features.
 }
 \label{tb:ood_gen_imagenet}
 \centering

\begin{tabular}{lccccccc}
\toprule
          Style & \multicolumn{2}{l}{\makecell{Baseline} } & \multicolumn{2}{l}{\makecell{\method}}  \\
{} &       Top-1  & Top-5  &       Top-1  & Top-5    \\
\midrule
dog sketch                    &                0.516 &  0.752 &                 \textbf{0.546} &   \textbf{0.777} \\
woman sketch               &                0.478 &  0.712 &                 \textbf{0.518} &   \textbf{0.752} \\
Picasso dog                 &                0.445 &  0.686 &                 \textbf{0.500}&   \textbf{0.738} \\
Picasso s.-p.                &                0.474 &  0.706 &                 \textbf{0.514} &   \textbf{0.747} \\
none                               &                 \textbf{0.757} &   \textbf{0.927} &                0.749 &  0.921 \\
\bottomrule
\end{tabular}
\vspace{-0.2cm}
\end{table}

\begin{table}
 \caption{Top-1 and top-5 accuracies on 5 variations of the ImageNet test set for Baseline pre-trained MAE-ViT-Base features and the corresponding post-processed \method\, ($\lambda=1$) features.
 }
 \label{tb:ood_mae_vit_1}
 \centering

\begin{tabular}{lccccccc} 
\toprule
              Style & \multicolumn{2}{l}{\makecell{Baseline}} & \multicolumn{2}{l}{\makecell{\method}}  \\
{} & Top-1 &  Top-5 & Top-1 & Top-5  \\
\midrule
dog sketch         &                   0.530    &  0.749   &                \textbf{0.575} &  \textbf{0.773}  \\
Picasso dog       &                   0.472     &  0.686  &                \textbf{0.519} &  \textbf{0.716} \\
Picasso s.-p.      &                   0.512     &  0.727  &                \textbf{0.558} &  \textbf{0.752}  \\
woman sketch    &                   0.504    &  0.719   &                \textbf{0.550 }&  \textbf{0.746}  \\
none                    &                   0.811      &  0.952  &                \textbf{0.818} &  \textbf{0.953}  \\

\bottomrule
\end{tabular}
\vspace{-0.2cm}
\end{table}

\section{Conclusion}
In this paper, we studied the problem of disentangling style and content of pre-trained visual representations. We presented \method, a simple post-processing algorithm with theoretical guarantees. In our experiments, we demonstrated that post-processing with \method\ can improve OOD performance of popular pre-trained deep models while preserving the in-distribution accuracy. Our method is computationally inexpensive and simple to implement.

In our experiments, we mainly were interested in discarding the style factors and keeping the style-invariant content factors for OOD generalization. However, we also demonstrated both theoretically and empirically that the learned style factors are representative of the presence or absence of the corresponding styles. Thus, the values of the style factors can be used to assist in outlier/OOD samples detection, or in some special cases of image retrieval, e.g., finding all images with a specific style.

One limitation of our method is the reliance on the availability of meaningful data transformations (or augmentations). While there are plenty of such transformations for images, they could be harder to identify for other data modalities. Natural language processing is one example where it is not as straightforward to define meaningful text augmentations. However, text data augmentations is also an active research area \citep{wei2019eda,bayer2021survey,shorten2021text} which could enable applications of \method\ to NLP.

Another interesting direction to explore is extending our model to various weak supervision settings \citep{bouchacourt2018multi,shu2019weakly,chen2020weakly}. In comparison to data augmentation functions, such forms of supervision are typically easier to obtain outside of the image domain. Thus, an extension of our model to weak supervision could enable disentanglement via post-processing for a broader class of data modalities.

\section*{Acknowledgements} This paper is based upon work supported by the National Science Foundation (NSF) under grants no.\ 2027737 and 2113373, and the Rensselaer-IBM AI Research Collaboration (\url{http://airc.rpi.edu}), part of the IBM AI Horizons Network (\url{http://ibm.biz/AIHorizons}).

\bibliography{YK,sm}

\begin{thebibliography}{51}
\providecommand{\natexlab}[1]{#1}
\providecommand{\url}[1]{\texttt{#1}}
\expandafter\ifx\csname urlstyle\endcsname\relax
  \providecommand{\doi}[1]{doi: #1}\else
  \providecommand{\doi}{doi: \begingroup \urlstyle{rm}\Url}\fi

\bibitem[Arjovsky et~al.(2019)Arjovsky, Bottou, Gulrajani, and
  {Lopez-Paz}]{arjovsky2019Invariant}
Arjovsky, M., Bottou, L., Gulrajani, I., and {Lopez-Paz}, D.
\newblock Invariant {{Risk Minimization}}.
\newblock \emph{arXiv:1907.02893 [cs, stat]}, September 2019.

\bibitem[Bayer et~al.(2021)Bayer, Kaufhold, and Reuter]{bayer2021survey}
Bayer, M., Kaufhold, M.-A., and Reuter, C.
\newblock A survey on data augmentation for text classification.
\newblock \emph{ACM Computing Surveys}, 2021.

\bibitem[Beery et~al.(2018)Beery, Van~Horn, and Perona]{beery2018recognition}
Beery, S., Van~Horn, G., and Perona, P.
\newblock Recognition in terra incognita.
\newblock In \emph{Proceedings of the European conference on computer vision
  (ECCV)}, pp.\  456--473, 2018.

\bibitem[Bouchacourt et~al.(2018)Bouchacourt, Tomioka, and
  Nowozin]{bouchacourt2018multi}
Bouchacourt, D., Tomioka, R., and Nowozin, S.
\newblock Multi-level variational autoencoder: Learning disentangled
  representations from grouped observations.
\newblock In \emph{Proceedings of the AAAI Conference on Artificial
  Intelligence}, volume 32(1), 2018.

\bibitem[Chen \& Batmanghelich(2020)Chen and Batmanghelich]{chen2020weakly}
Chen, J. and Batmanghelich, K.
\newblock Weakly supervised disentanglement by pairwise similarities.
\newblock In \emph{Proceedings of the AAAI Conference on Artificial
  Intelligence}, volume 34(04), pp.\  3495--3502, 2020.

\bibitem[Chen et~al.(2020)Chen, Wei, Huang, Ding, and Li]{chen2020simple}
Chen, M., Wei, Z., Huang, Z., Ding, B., and Li, Y.
\newblock Simple and {{Deep Graph Convolutional Networks}}.
\newblock \emph{arXiv:2007.02133 [cs, stat]}, July 2020.

\bibitem[Chen \& He(2021)Chen and He]{chen2021exploring}
Chen, X. and He, K.
\newblock Exploring simple siamese representation learning.
\newblock In \emph{Proceedings of the IEEE/CVF Conference on Computer Vision
  and Pattern Recognition}, pp.\  15750--15758, 2021.

\bibitem[Deldjoo et~al.(2022)Deldjoo, Di~Noia, Malitesta, and
  Merra]{deldjoo2022leveraging}
Deldjoo, Y., Di~Noia, T., Malitesta, D., and Merra, F.~A.
\newblock Leveraging content-style item representation for visual
  recommendation.
\newblock In \emph{European Conference on Information Retrieval}, pp.\  84--92.
  Springer, 2022.

\bibitem[Doersch et~al.(2015)Doersch, Gupta, and
  Efros]{doersch2015unsupervised}
Doersch, C., Gupta, A., and Efros, A.~A.
\newblock Unsupervised visual representation learning by context prediction.
\newblock In \emph{Proceedings of the IEEE international conference on computer
  vision}, pp.\  1422--1430, 2015.

\bibitem[Garcia \& Vogiatzis(2018)Garcia and Vogiatzis]{garcia2018read}
Garcia, N. and Vogiatzis, G.
\newblock How to read paintings: semantic art understanding with multi-modal
  retrieval.
\newblock In \emph{Proceedings of the European Conference on Computer Vision
  (ECCV) Workshops}, pp.\  0--0, 2018.

\bibitem[Geirhos et~al.(2018)Geirhos, Rubisch, Michaelis, Bethge, Wichmann, and
  Brendel]{geirhos2018imagenet}
Geirhos, R., Rubisch, P., Michaelis, C., Bethge, M., Wichmann, F.~A., and
  Brendel, W.
\newblock Imagenet-trained cnns are biased towards texture; increasing shape
  bias improves accuracy and robustness.
\newblock \emph{arXiv preprint arXiv:1811.12231}, 2018.

\bibitem[Grill et~al.(2020)Grill, Strub, Altch{\'e}, Tallec, Richemond,
  Buchatskaya, Doersch, Pires, Guo, Azar, Piot, Kavukcuoglu, Munos, and
  Valko]{grill2020bootstrap}
Grill, J.-B., Strub, F., Altch{\'e}, F., Tallec, C., Richemond, P.~H.,
  Buchatskaya, E., Doersch, C., Pires, B.~A., Guo, Z.~D., Azar, M.~G., Piot,
  B., Kavukcuoglu, K., Munos, R., and Valko, M.
\newblock Bootstrap your own latent: {{A}} new approach to self-supervised
  {{Learning}}.
\newblock \emph{arXiv:2006.07733 [cs, stat]}, September 2020.

\bibitem[Gutmann \& Hyv{\"a}rinen(2012)Gutmann and
  Hyv{\"a}rinen]{gutmann2012noise}
Gutmann, M.~U. and Hyv{\"a}rinen, A.
\newblock Noise-contrastive estimation of unnormalized statistical models, with
  applications to natural image statistics.
\newblock \emph{Journal of machine learning research}, 13\penalty0 (2), 2012.

\bibitem[He et~al.(2016)He, Zhang, Ren, and Sun]{he2016Deep}
He, K., Zhang, X., Ren, S., and Sun, J.
\newblock Deep {{Residual Learning}} for {{Image Recognition}}.
\newblock In \emph{2016 {{IEEE Conference}} on {{Computer Vision}} and
  {{Pattern Recognition}} ({{CVPR}})}, pp.\  770--778, {Las Vegas, NV, USA},
  June 2016. {IEEE}.
\newblock ISBN 978-1-4673-8851-1.
\newblock \doi{10.1109/CVPR.2016.90}.

\bibitem[He et~al.(2022)He, Chen, Xie, Li, Doll{\'a}r, and
  Girshick]{he2022masked}
He, K., Chen, X., Xie, S., Li, Y., Doll{\'a}r, P., and Girshick, R.
\newblock Masked autoencoders are scalable vision learners.
\newblock In \emph{Proceedings of the IEEE/CVF Conference on Computer Vision
  and Pattern Recognition}, pp.\  16000--16009, 2022.

\bibitem[Hendrycks \& Dietterich(2018)Hendrycks and
  Dietterich]{hendrycks2018Benchmarking}
Hendrycks, D. and Dietterich, T.
\newblock Benchmarking {{Neural Network Robustness}} to {{Common Corruptions}}
  and {{Perturbations}}.
\newblock In \emph{International {{Conference}} on {{Learning
  Representations}}}, September 2018.

\bibitem[Higgins et~al.(2018)Higgins, Amos, Pfau, Racaniere, Matthey, Rezende,
  and Lerchner]{higgins2018towards}
Higgins, I., Amos, D., Pfau, D., Racaniere, S., Matthey, L., Rezende, D., and
  Lerchner, A.
\newblock Towards a definition of disentangled representations.
\newblock \emph{arXiv preprint arXiv:1812.02230}, 2018.

\bibitem[Hosoya(2018)]{hosoya2018group}
Hosoya, H.
\newblock Group-based learning of disentangled representations with
  generalizability for novel contents.
\newblock \emph{arXiv preprint arXiv:1809.02383}, 2018.

\bibitem[Huang \& Belongie(2017)Huang and Belongie]{huang2017arbitrary}
Huang, X. and Belongie, S.
\newblock Arbitrary style transfer in real-time with adaptive instance
  normalization.
\newblock In \emph{Proceedings of the IEEE international conference on computer
  vision}, pp.\  1501--1510, 2017.

\bibitem[Hyv{\"a}rinen \& Morioka(2016)Hyv{\"a}rinen and
  Morioka]{hyvarinen2016unsupervised}
Hyv{\"a}rinen, A. and Morioka, H.
\newblock Unsupervised feature extraction by time-contrastive learning and
  nonlinear {{ICA}}.
\newblock In \emph{Proceedings of the 30th {{International Conference}} on
  {{Neural Information Processing Systems}}}, {{NIPS}}'16, pp.\  3772--3780,
  {Red Hook, NY, USA}, December 2016. {Curran Associates Inc.}
\newblock ISBN 978-1-5108-3881-9.

\bibitem[Hyv{\"a}rinen \& Pajunen(1999)Hyv{\"a}rinen and
  Pajunen]{hyvarinen1999nonlinear}
Hyv{\"a}rinen, A. and Pajunen, P.
\newblock Nonlinear independent component analysis: Existence and uniqueness
  results.
\newblock \emph{Neural networks}, 12\penalty0 (3):\penalty0 429--439, 1999.

\bibitem[John et~al.(2018)John, Mou, Bahuleyan, and
  Vechtomova]{john2018disentangled}
John, V., Mou, L., Bahuleyan, H., and Vechtomova, O.
\newblock Disentangled representation learning for non-parallel text style
  transfer.
\newblock \emph{arXiv preprint arXiv:1808.04339}, 2018.

\bibitem[Jutten et~al.(2010)Jutten, Babaie-Zadeh, and
  Karhunen]{jutten2010nonlinear}
Jutten, C., Babaie-Zadeh, M., and Karhunen, J.
\newblock Nonlinear mixtures.
\newblock In \emph{Handbook of Blind Source Separation}, pp.\  549--592.
  Elsevier, 2010.

\bibitem[Kim et~al.(2018)Kim, Wattenberg, Gilmer, Cai, Wexler, Viegas, and
  Sayres]{kim2018Interpretability}
Kim, B., Wattenberg, M., Gilmer, J., Cai, C., Wexler, J., Viegas, F., and
  Sayres, R.
\newblock Interpretability {{Beyond Feature Attribution}}: {{Quantitative
  Testing}} with {{Concept Activation Vectors}} ({{TCAV}}).
\newblock In \emph{International {{Conference}} on {{Machine Learning}}}, pp.\
  2668--2677, July 2018.

\bibitem[Koh et~al.(2020)Koh, Sagawa, Marklund, Xie, Zhang, Balsubramani, Hu,
  Yasunaga, Phillips, Beery, Leskovec, Kundaje, Pierson, Levine, Finn, and
  Liang]{koh2020WILDS}
Koh, P.~W., Sagawa, S., Marklund, H., Xie, S.~M., Zhang, M., Balsubramani, A.,
  Hu, W., Yasunaga, M., Phillips, R.~L., Beery, S., Leskovec, J., Kundaje, A.,
  Pierson, E., Levine, S., Finn, C., and Liang, P.
\newblock {{WILDS}}: {{A Benchmark}} of in-the-{{Wild Distribution Shifts}}.
\newblock \emph{arXiv:2012.07421 [cs]}, December 2020.

\bibitem[Krizhevsky et~al.(2009)Krizhevsky, Hinton,
  et~al.]{krizhevsky2009learning}
Krizhevsky, A., Hinton, G., et~al.
\newblock Learning multiple layers of features from tiny images.
\newblock 2009.

\bibitem[K{\"u}gelgen et~al.(2021)K{\"u}gelgen, Sharma, Gresele, Brendel,
  Sch{\"o}lkopf, Besserve, and Locatello]{kugelgen2021self}
K{\"u}gelgen, J., Sharma, Y., Gresele, L., Brendel, W., Sch{\"o}lkopf, B.,
  Besserve, M., and Locatello, F.
\newblock Self-supervised learning with data augmentations provably isolates
  content from style.
\newblock \emph{Advances in neural information processing systems},
  34:\penalty0 16451--16467, 2021.

\bibitem[Lee et~al.(2018)Lee, Tseng, Huang, Singh, and Yang]{lee2018diverse}
Lee, H.-Y., Tseng, H.-Y., Huang, J.-B., Singh, M., and Yang, M.-H.
\newblock Diverse image-to-image translation via disentangled representations.
\newblock In \emph{Proceedings of the European conference on computer vision
  (ECCV)}, pp.\  35--51, 2018.

\bibitem[Locatello et~al.(2019{\natexlab{a}})Locatello, Abbati, Rainforth,
  Bauer, Sch{\"o}lkopf, and Bachem]{locatello2019fairness}
Locatello, F., Abbati, G., Rainforth, T., Bauer, S., Sch{\"o}lkopf, B., and
  Bachem, O.
\newblock On the fairness of disentangled representations.
\newblock In \emph{Proceedings of the 33rd {{International Conference}} on
  {{Neural Information Processing Systems}}}, number 1309, pp.\  14611--14624.
  {Curran Associates Inc.}, {Red Hook, NY, USA}, December 2019{\natexlab{a}}.

\bibitem[Locatello et~al.(2019{\natexlab{b}})Locatello, Bauer, Lucic, Raetsch,
  Gelly, Sch{\"o}lkopf, and Bachem]{locatello2019challenging}
Locatello, F., Bauer, S., Lucic, M., Raetsch, G., Gelly, S., Sch{\"o}lkopf, B.,
  and Bachem, O.
\newblock Challenging {{Common Assumptions}} in the {{Unsupervised Learning}}
  of {{Disentangled Representations}}.
\newblock In \emph{Proceedings of the 36th {{International Conference}} on
  {{Machine Learning}}}, pp.\  4114--4124. {PMLR}, May 2019{\natexlab{b}}.

\bibitem[Locatello et~al.(2020)Locatello, Poole, R{\"a}tsch, Sch{\"o}lkopf,
  Bachem, and Tschannen]{locatello2020weakly}
Locatello, F., Poole, B., R{\"a}tsch, G., Sch{\"o}lkopf, B., Bachem, O., and
  Tschannen, M.
\newblock Weakly-supervised disentanglement without compromises.
\newblock In \emph{International Conference on Machine Learning}, pp.\
  6348--6359. PMLR, 2020.

\bibitem[Ma et~al.(2019)Ma, Zhou, Cui, Yang, and Zhu]{ma2019learning}
Ma, J., Zhou, C., Cui, P., Yang, H., and Zhu, W.
\newblock Learning disentangled representations for recommendation.
\newblock \emph{Advances in neural information processing systems}, 32, 2019.

\bibitem[Nemeth(2020)]{nemeth2020adversarial}
Nemeth, J.
\newblock Adversarial disentanglement with grouped observations.
\newblock In \emph{Proceedings of the AAAI Conference on Artificial
  Intelligence}, volume~34, pp.\  10243--10250, 2020.

\bibitem[Oord et~al.(2018)Oord, Li, and Vinyals]{oord2018representation}
Oord, A. v.~d., Li, Y., and Vinyals, O.
\newblock Representation learning with contrastive predictive coding.
\newblock \emph{arXiv preprint arXiv:1807.03748}, 2018.

\bibitem[Petersen et~al.(2021)Petersen, Mukherjee, Sun, and
  Yurochkin]{petersen2021post}
Petersen, F., Mukherjee, D., Sun, Y., and Yurochkin, M.
\newblock Post-processing for individual fairness.
\newblock \emph{Advances in Neural Information Processing Systems},
  34:\penalty0 25944--25955, 2021.

\bibitem[Ren et~al.(2021)Ren, Yang, Wang, and Zeng]{ren2021rethinking}
Ren, X., Yang, T., Wang, Y., and Zeng, W.
\newblock Rethinking content and style: exploring bias for unsupervised
  disentanglement.
\newblock In \emph{Proceedings of the IEEE/CVF International Conference on
  Computer Vision}, pp.\  1823--1832, 2021.

\bibitem[Russakovsky et~al.(2015)Russakovsky, Deng, Su, Krause, Satheesh, Ma,
  Huang, Karpathy, Khosla, Bernstein, et~al.]{russakovsky2015imagenet}
Russakovsky, O., Deng, J., Su, H., Krause, J., Satheesh, S., Ma, S., Huang, Z.,
  Karpathy, A., Khosla, A., Bernstein, M., et~al.
\newblock Imagenet large scale visual recognition challenge.
\newblock \emph{International journal of computer vision}, 115\penalty0
  (3):\penalty0 211--252, 2015.

\bibitem[Ruta et~al.(2021)Ruta, Motiian, Faieta, Lin, Jin, Filipkowski,
  Gilbert, and Collomosse]{ruta2021aladin}
Ruta, D., Motiian, S., Faieta, B., Lin, Z., Jin, H., Filipkowski, A., Gilbert,
  A., and Collomosse, J.
\newblock Aladin: all layer adaptive instance normalization for fine-grained
  style similarity.
\newblock In \emph{Proceedings of the IEEE/CVF International Conference on
  Computer Vision}, pp.\  11926--11935, 2021.

\bibitem[Ruta et~al.(2022)Ruta, Gilbert, Aggarwal, Marri, Kale, Briggs, Speed,
  Jin, Faieta, Filipkowski, et~al.]{ruta2022stylebabel}
Ruta, D., Gilbert, A., Aggarwal, P., Marri, N., Kale, A., Briggs, J., Speed,
  C., Jin, H., Faieta, B., Filipkowski, A., et~al.
\newblock Stylebabel: Artistic style tagging and captioning.
\newblock In \emph{Computer Vision--ECCV 2022: 17th European Conference, Tel
  Aviv, Israel, October 23--27, 2022, Proceedings, Part VIII}, pp.\  219--236.
  Springer, 2022.

\bibitem[Sagawa et~al.(2019)Sagawa, Koh, Hashimoto, and
  Liang]{sagawa2019Distributionally}
Sagawa, S., Koh, P.~W., Hashimoto, T.~B., and Liang, P.
\newblock Distributionally {{Robust Neural Networks}} for {{Group Shifts}}:
  {{On}} the {{Importance}} of {{Regularization}} for {{Worst-Case
  Generalization}}.
\newblock \emph{arXiv:1911.08731 [cs, stat]}, November 2019.

\bibitem[Saleh \& Elgammal(2015)Saleh and Elgammal]{saleh2015large}
Saleh, B. and Elgammal, A.
\newblock Large-scale classification of fine-art paintings: Learning the right
  metric on the right feature.
\newblock \emph{arXiv preprint arXiv:1505.00855}, 2015.

\bibitem[Sch{\"o}lkopf et~al.(2021)Sch{\"o}lkopf, Locatello, Bauer, Ke,
  Kalchbrenner, Goyal, and Bengio]{scholkopf2021toward}
Sch{\"o}lkopf, B., Locatello, F., Bauer, S., Ke, N.~R., Kalchbrenner, N.,
  Goyal, A., and Bengio, Y.
\newblock Toward causal representation learning.
\newblock \emph{Proceedings of the IEEE}, 109\penalty0 (5):\penalty0 612--634,
  2021.

\bibitem[Shorten et~al.(2021)Shorten, Khoshgoftaar, and Furht]{shorten2021text}
Shorten, C., Khoshgoftaar, T.~M., and Furht, B.
\newblock Text data augmentation for deep learning.
\newblock \emph{Journal of big Data}, 8\penalty0 (1):\penalty0 1--34, 2021.

\bibitem[Shu et~al.(2019)Shu, Chen, Kumar, Ermon, and Poole]{shu2019weakly}
Shu, R., Chen, Y., Kumar, A., Ermon, S., and Poole, B.
\newblock Weakly supervised disentanglement with guarantees.
\newblock \emph{arXiv preprint arXiv:1910.09772}, 2019.

\bibitem[Strubell et~al.(2019)Strubell, Ganesh, and
  McCallum]{strubell2019energy}
Strubell, E., Ganesh, A., and McCallum, A.
\newblock Energy and {{Policy Considerations}} for {{Deep Learning}} in
  {{NLP}}.
\newblock In \emph{Proceedings of the 57th {{Annual Meeting}} of the
  {{Association}} for {{Computational Linguistics}}}, pp.\  3645--3650,
  {Florence, Italy}, July 2019. {Association for Computational Linguistics}.
\newblock \doi{10.18653/v1/P19-1355}.

\bibitem[Wang et~al.(2021)Wang, Yue, Huang, Sun, and Zhang]{wang2021self}
Wang, T., Yue, Z., Huang, J., Sun, Q., and Zhang, H.
\newblock Self-supervised learning disentangled group representation as
  feature.
\newblock \emph{Advances in Neural Information Processing Systems},
  34:\penalty0 18225--18240, 2021.

\bibitem[Wei et~al.(2019)Wei, Ramamurthy, and Calmon]{wei2019Optimized}
Wei, D., Ramamurthy, K.~N., and Calmon, F. d.~P.
\newblock Optimized {{Score Transformation}} for {{Fair Classification}}.
\newblock \emph{arXiv:1906.00066 [cs, math, stat]}, December 2019.

\bibitem[Wei \& Zou(2019)Wei and Zou]{wei2019eda}
Wei, J. and Zou, K.
\newblock Eda: Easy data augmentation techniques for boosting performance on
  text classification tasks.
\newblock \emph{arXiv preprint arXiv:1901.11196}, 2019.

\bibitem[Wu et~al.(2009)Wu, Chen, Hastie, Sobel, and Lange]{wu2009Genomewide}
Wu, T.~T., Chen, Y.~F., Hastie, T., Sobel, E., and Lange, K.
\newblock Genome-wide association analysis by lasso penalized logistic
  regression.
\newblock \emph{Bioinformatics}, 25\penalty0 (6):\penalty0 714--721, March
  2009.
\newblock ISSN 1460-2059, 1367-4803.
\newblock \doi{10.1093/bioinformatics/btp041}.

\bibitem[Wu et~al.(2019)Wu, Cao, Li, Qian, and Loy]{wu2019disentangling}
Wu, W., Cao, K., Li, C., Qian, C., and Loy, C.~C.
\newblock Disentangling content and style via unsupervised geometry
  distillation.
\newblock \emph{arXiv preprint arXiv:1905.04538}, 2019.

\bibitem[Zimmermann et~al.(2021)Zimmermann, Sharma, Schneider, Bethge, and
  Brendel]{zimmermann2021contrastive}
Zimmermann, R.~S., Sharma, Y., Schneider, S., Bethge, M., and Brendel, W.
\newblock Contrastive learning inverts the data generating process.
\newblock In \emph{International Conference on Machine Learning}, pp.\
  12979--12990. PMLR, 2021.

\end{thebibliography}
\bibliographystyle{icml2023}

\appendix
\section{Supplementary proofs}

\subsection{Proof of Corollary \ref{cor:correlation-recovery}}
\label{proof-correlation-recovery}
\begin{proof}
We denote $\hat \bz_S \triangleq [\hat \bz]_{\stylefactor}$,  $\bz_S \triangleq [\bz]_{\stylefactor}$, $\Sigma \triangleq \operatorname{cov}(\bz)$  and  $\Delta$ as the diagonal matrix of $\Sigma$. Notice that 
\begin{equation}
    \operatorname{corr}(\bz) = \Delta^{-\nicefrac{1}{2}} \Sigma \Delta^{-\nicefrac{1}{2}}\,.
\end{equation} From Definition \ref{def:direp} $\hat \bz_S = [\bP \bA]_{\stylefactor, \stylefactor} \bz_S$ where $[\bP \bA]_{\stylefactor, \stylefactor}$ is a diagonal matrix. We denote $[\bP \bA]_{\stylefactor, \stylefactor}$ as $\bD$. Then the covariance matrix of $\hat\bz_S$ is \begin{equation}
    \begin{aligned}
     \operatorname{cov}(\hat \bz_S) &= \operatorname{cov}(\bD \bz_S)
      = \bD \Sigma \bD
    \end{aligned}
\end{equation} and its diagonal matrix is 
\begin{equation}
\begin{aligned}
  \operatorname{diag}(\bD \Sigma \bD) &= \bD\operatorname{diag}(\Sigma) \bD\\
  &= \bD\Delta \bD\\
  & = \Delta \bD^2\,,
\end{aligned}
\end{equation} where the last equality is obtained using the fact that the matrix multiplication of the diagonal matrices is commuting. Expressing $\operatorname{corr}(\hat \bz_S)$ in terms of $\operatorname{cov}(\hat \bz_S)$ and it's diagonal matrix we obtain
\begin{equation}
    \begin{aligned}
     &\operatorname{corr}(\hat \bz_S) \\
     & = \operatorname{diag}\big\{\operatorname{corr}(\hat \bz_S)\big\}^{-\nicefrac{1}{2}} \operatorname{corr}(\hat \bz_S) \operatorname{diag}\big\{\operatorname{corr}(\hat \bz_S)\big\}^{-\nicefrac{1}{2}}\\
     & = \{ \Delta \bD^2 \}^{-\nicefrac{1}{2}} \bD \Sigma \bD \{ \Delta \bD^2 \}^{-\nicefrac{1}{2}}\\
     & = \Delta^{-\nicefrac{1}{2}} \bD^{-1} \bD \Sigma \bD  \bD^{-1} \Delta^{-\nicefrac{1}{2}}\\
     & = \Delta^{-\nicefrac{1}{2}}\Sigma \Delta^{-\nicefrac{1}{2}} = \operatorname{corr}(\bz_S)
    \end{aligned}
\end{equation} and we obtain \eqref{eq:corr-recovery}.
\end{proof}

\subsection{Proof of Theorem \ref{th:style-factors}}
\label{sec:proof-style-factor}

\begin{proof}

The closed form of $\hat \bp_j$ in \eqref{eq:jth-direction} can be written as:
\begin{equation}
    \hat\bp_j = \hat \Sigma_{\bu}^{\dagger} \frac{1}{2n} \sum_{i = 1}^n \Big[(\bu_i - \bar \bu)(\by_i^{(j)} - \bar \by) +(\bu_i^{(j)} - \bar \bu)(\tilde \by_i^{(j)} - \bar \by) \Big ]
\end{equation} where $\bar \bu =\frac{1}{2n} \sum_{i = 1}^n \big[ \bu_i  + \bu_i^{(j)} \big] $, $\hat \Sigma_{\bu} = \frac{1}{2n} \sum_{i = 1}^n\big[ \bu_i \bu_i^\top  + \bu_i^{(j)} \{\bu_i^{(j)}\}^\top \big] - \bar \bu \bar \bu^\top$, $\hat \Sigma_{\bu}^{\dagger}$ is the Moore-Penrose inverse of $\hat \Sigma_{\bu}$, $\bar \by = \frac{1}{2n} \sum_{i = 1}^n \big[ \by_i^{(j)} + \tilde \by_i^{(j)}\big]$, and  $\bar \epsilon^{(j)} = \frac{1}{2n} \sum_{i = 1}^n \big[ \epsilon_i^{(j)} + \tilde \epsilon_i^{(j)}\big]$. 
Here, defining $\bar \bz =\frac{1}{2n} \sum_{i = 1}^n \big[ \bz_i  + \bz_i^{(j)} \big]$ and $\hat \Sigma_{\bz} = \frac{1}{2n} \sum_{i = 1}^n\big[ \bz_i \bz_i^\top  + \bz_i^{(j)} \{\bz_i^{(j)}\}^\top \big] - \bar \bz \bar \bz^\top$ we notice the following.
\begin{align}
    \by_i^{(j)} - \bar \by &= \beta_j \big([\bz_i]_j - [\bar \bz]_j\big) + (\epsilon_i^{(j)} - \bar \epsilon^{(j)})\\
     \tilde \by_i^{(j)} - \bar \by &= \beta_j \big([\bz_i^{(j)}]_j - [\bar \bz]_j\big) + (\tilde \epsilon_i^{(j)} - \bar \epsilon^{(j)})\\
     \hat \Sigma_{\bu} &= \bA \hat \Sigma_{\bz} \bA^\top \label{eq:sigma-u}
\end{align}
and hence
\begin{equation}
\label{eq:cov-uy}
    \begin{aligned}
     & \frac{1}{2n} \sum_{i = 1}^n \Big[(\bu_i - \bar \bu)(\by_i^{(j)} - \bar \by) +(\bu_i^{(j)} - \bar \bu)(\tilde \by_i^{(j)} - \bar \by) \Big ]  \\
     & = \frac{1}{2n} \sum_{i = 1}^n \Big[\bA (\bz_i - \bar \bz)\beta_j \big([\bz_i]_j - [\bar \bz]_j\big)\\
     & ~~~~~~~~~~~~~~~~+\bA(\bz_i^{(j)} - \bar \bz)\beta_j \big([\bz_i^{(j)}]_j - [\bar \bz]_j\big) \Big ]\\
      & ~~+ \frac{1}{2n} \sum_{i = 1}^n \Big[\bA (\bz_i - \bar \bz)(\epsilon_i^{(j)} - \bar \epsilon^{(j)})\\
     & ~~~~~~~~~~~~~~~~+\bA(\bz_i^{(j)} - \bar \bz)(\tilde \epsilon_i^{(j)} - \bar \epsilon^{(j)}) \Big ]\\
     & \triangleq \operatorname{cov}_1 + \operatorname{cov}_2
    \,,
    \end{aligned}
\end{equation} where 
\[
\begin{aligned}
 \operatorname{cov_1} & \triangleq \frac{1}{2n} \sum_{i = 1}^n \Big[\bA (\bz_i - \bar \bz)\beta_j \big([\bz_i]_j - [\bar \bz]_j\big)\\
     & ~~~~~~~~~~~~~~~~+\bA(\bz_i^{(j)} - \bar \bz)\beta_j \big([\bz_i^{(j)}]_j - [\bar \bz]_j\big) \Big ]\\
      & = \beta_j \bA \frac{1}{2n} \sum_{i = 1}^n \Big[ (\bz_i - \bar \bz) (\bz_i - \bar \bz)^\top e_j\\
     & ~~~~~~~~~~~~~~~~+(\bz_i^{(j)} - \bar \bz)(\bz_i^{(j)} - \bar \bz)^\top e_j \Big ]\\
     & = \beta_j \bA \hat \Sigma_{\bz} e_j\,.
\end{aligned}
\] and with $\hat \Sigma_{\bz, \epsilon} \triangleq \frac{1}{2n} \sum_{i = 1}^n \Big[ (\bz_i - \bar \bz)(\epsilon_i^{(j)} - \bar \epsilon^{(j)})
    +(\bz_i^{(j)} - \bar \bz)(\tilde \epsilon_i^{(j)} - \bar \epsilon^{(j)}) \Big ]$ we obtain 
\[
\begin{aligned}
 \operatorname{cov_2} & \triangleq \frac{1}{2n} \sum_{i = 1}^n \Big[\bA (\bz_i - \bar \bz)(\epsilon_i^{(j)} - \bar \epsilon^{(j)})\\
     & ~~~~~~~~~~~~~~~~+\bA(\bz_i^{(j)} - \bar \bz)(\tilde \epsilon_i^{(j)} - \bar \epsilon^{(j)}) \Big ]\\
     & = \bA \hat \Sigma_{\bz, \epsilon}
\end{aligned}
\]

Using the identities \eqref{eq:sigma-u} and \eqref{eq:cov-uy} we rewrite $\hat \bp_j$ as 
\begin{equation}
\label{eq:eq-pj}
\begin{aligned}
 \hat \bp_j &= (\bA \hat \Sigma_{\bz} \bA^\top)^\dagger \{\beta_j \bA \hat \Sigma_{\bz} e_j + \bA \hat \Sigma_{\bz, \epsilon}\}\\
 & = (\bA \hat \Sigma_{\bz} \bA^\top)^\dagger  \beta_j \bA \hat \Sigma_{\bz} e_j + (\bA \hat \Sigma_{\bz} \bA^\top)^\dagger \bA \hat \Sigma_{\bz, \epsilon}\\
 & \triangleq \hat \bp_j^{(1)} + \hat \bp_j^{(2)}\,,
\end{aligned}
\end{equation} where $\hat \bp_j^{(1)} \triangleq (\bA \hat \Sigma_{\bz} \bA^\top)^\dagger  \beta_j \bA \hat \Sigma_{\bz} e_j$ and $\hat \bp_j^{(2)} \triangleq (\bA \hat \Sigma_{\bz} \bA^\top)^\dagger \bA \hat \Sigma_{\bz, \epsilon}$. 

Since $n \ge d + 1$ we notice that the covariance matrix $\hat \Sigma_{\bz}$ is invertible. This fact combined with left invertibility of $\bA$ implies that the matrix  $ \bA \hat \Sigma_{\bz}^{\nicefrac{1}{2}}$ is also left invertible. We recall the property of Moore-Penrose inverse that for any left invertible matrix $\bG$ it holds: 
\[
(\bG\bG^\top )^\dagger \bG = \bG (\bG^\top \bG)^{-1}\,.
\] Letting $\bG = \bA \hat \Sigma_{\bz}^{\nicefrac{1}{2}}$ and using the property in \eqref{eq:eq-pj} we obtain 
\begin{equation}
   \begin{aligned}
    \hat \bp_j^{(1)} &= \beta _j (\bA \hat \Sigma_{\bz} \bA^\top)^\dagger  \bA \hat \Sigma_{\bz} e_j\\
    & = \beta _j (\bG\bG^\top)^\dagger  \bG \hat \Sigma_{\bz}^{\nicefrac{1}{2}} e_j\\
    & = \beta _j \bG (\bG^\top \bG)^{-1} \hat \Sigma_{\bz}^{\nicefrac{1}{2}} e_j\\
    & = \beta _j \bA \Sigma_{\bz}^{\nicefrac{1}{2}} \left(\Sigma_{\bz}^{\nicefrac{1}{2}}\bA^\top \bA \Sigma_{\bz}^{\nicefrac{1}{2}}\right )^{-1} \hat \Sigma_{\bz}^{\nicefrac{1}{2}} e_j\\
    & = \beta _j \bA \Sigma_{\bz}^{\nicefrac{1}{2}} \Sigma_{\bz}^{-\nicefrac{1}{2}} (\bA^\top \bA  )^{-1} \Sigma_{\bz}^{-\nicefrac{1}{2}} \hat \Sigma_{\bz}^{\nicefrac{1}{2}} e_j\\
    & = \beta _j \bA  (\bA^\top \bA  )^{-1}  e_j\,.
   \end{aligned}
\end{equation} Hence, we notice that 
\[
\begin{aligned}
 \bA^\top \hat \bp_j^{(1)} &= \bA^\top \hat\beta _j \bA  (\bA^\top \bA  )^{-1}  e_j\\
 & = \beta_j e_j\,.
\end{aligned}
\] 
Repeating same calculation as above we obtain 
\[
\begin{aligned}
 \bA^\top \hat \bp_j^{(2)} & =  \bA^\top (\bA \hat \Sigma_{\bz} \bA^\top)^\dagger  \bA \hat \Sigma_{\bz, \epsilon}\\
 &= \hat \Sigma_\bz^{-1} \hat \Sigma_{\bz, \epsilon}\,.
\end{aligned}
\]

From Assumption \ref{assmp:sample-manipulation} we recall that $(\epsilon^{(j)}, \tilde \epsilon^{(j)})$ and  $(\bz, \bz^{(j)})$ are uncorrelated and hence 
\[
\hat \Sigma_{\bz, \epsilon} \stackrel{\text{a.s.}}{\longrightarrow} \mathbf{0}\,.
\]  Since $\Sigma_\bz$ is invertible we obtain that
\[
\bA^\top \hat \bp_j^{(2)} \stackrel{\text{a.s.}}{\longrightarrow} \mathbf{0}\,,
\] and 
\[
\bA^\top \hat \bp_j \stackrel{\text{a.s.}}{\longrightarrow} \beta_j e_j
\] almost surely. Noticing that $\hat \bp_j$ is the $j$-th row of $\bP$ we conclude that  at almost sure limit it holds: (1) $[\bP \bA]_{\stylefactor, \stylefactor}$ converges to a diagonal matrix, and (2) $[\bP \bA]_{\stylefactor, \contentfactor} \stackrel{\text{a.s.}}{\longrightarrow} \mathbf{0}$.

\end{proof}

\subsection{Proof of Theorem \ref{th:sparse-recovery}}
\label{sec:proof-sparse-recovery}
We divide the proof in two steps which are stated as lemmas.

\begin{lemma}
\label{lemma:sparse-recovery-part2}
With probability at least $1 - \kappa^n$ ($\kappa$ is defined in Theorem \ref{th:sparse-recovery})
the following holds:    \[ \hat\bQ(+\infty) [\bA]_{\cdot, \stylefactor} = \mathbf{0}\,. \]
\end{lemma}

\begin{proof}[Proof of Lemma \ref{lemma:sparse-recovery-part2}]
At $\lambda \to \infty$ it necessarily holds: 
\[\frac1{m}\sum_{j \in \cS} \operatorname{tr} \left[\hat\bQ(+\infty)^\top \hat\bQ(+\infty) \Big(\frac{\Delta_j^\top \Delta_j}{n} \Big)\right] = 0\] which equivalently means for every $j \in \cS$: 
\begin{equation}
\label{eq:sec-part-zero}
    \operatorname{tr} \left[\hat\bQ(+\infty)^\top \hat\bQ(+\infty) \Big(\frac{\Delta_j^\top \Delta_j}{n} \Big)\right] = 0\,.
\end{equation} Combining \eqref{eq:second-part-of-loss} and the above we obtain
\[
\begin{aligned}
 &\frac{1}{n}\sum_{i = 1}^n \|\hat\bQ(+\infty)(\bu_i - \bu_i^{(j)})\|_2^2 \\
 &= \operatorname{tr} \left[\hat\bQ(+\infty)^\top \hat\bQ(+\infty) \Big(\frac{\Delta_j^\top \Delta_j}{n} \Big)\right] = 0
\end{aligned}
\] which implies that for each $i \in [n]$ and $j \in \stylefactor$
\begin{equation}
    \begin{aligned}
     \mathbf{0} & =  \hat\bQ(+\infty)(\bu_i - \bu_i^{(j)})\\
     & = \hat\bQ(+\infty)\bA (\bz_i - \bz_i^{(j)})\,,
    \end{aligned}
\end{equation} where the second equality follows from Assumption \ref{assmp:entangled-rep}. 
Since the latent factors $\{\bz_i\}_{i = 1}^n$ were drawn independently from the distribution $\bbP_\bz$, we conclude that one of the samples is in the event $\bE$ with probability at least $1 - \kappa^n$. Denote the sample as $\bz_0$. Then defining $ \bZ_0 = [\bz_0 - \bz_0^{(1)}, \dots, \bz_0 - \bz_0^{(m)}]$ we notice the following: (1) from \eqref{eq:unchanged-content} in Assumption \ref{assmp:sample-manipulation} it follows $[\bZ_0]_{\contentfactor, \cdot } = \textbf{0}$, and (2)  from the Assumption \ref{assmp:linear-independence} we see that the matrix $[\bZ_0]_{\stylefactor, \cdot}$ is invertible and hence we obtain 
\[
\begin{aligned}
 \textbf{0} & = \hat\bQ(+\infty)\bA [\bz_0 - \bz_0^{(1)}, \dots, \bz_0 - \bz_0^{(m)}]\\
 & =\hat\bQ(+\infty)
 \begin{bmatrix}
 [\bA]_{\cdot, \stylefactor} & [\bA]_{\cdot, \contentfactor}
 \end{bmatrix}. \begin{bmatrix}
 [\bZ_0]_{\stylefactor, \cdot}\\ [\bZ_0]_{\contentfactor, \cdot}
 \end{bmatrix}\\
 & = \hat\bQ(+\infty)
 \begin{bmatrix}
 [\bA]_{\cdot, \stylefactor} & [\bA]_{\cdot, \contentfactor}
 \end{bmatrix}. \begin{bmatrix}
 [\bZ_0]_{\stylefactor, \cdot}\\ \textbf{0}
 \end{bmatrix}\\
 & = \hat\bQ(+\infty) [\bA]_{\cdot, \stylefactor} [\bZ_0]_{\stylefactor, \cdot}\,,
\end{aligned}
\] where using invertibility of $[\bZ_0]_{\stylefactor, \cdot}$ we conclude \[
\hat\bQ(+\infty) [\bA]_{\cdot, \stylefactor} = \textbf{0}
\] and the lemma. 

\end{proof}

\begin{lemma}
Let $\bH^\perp\in \reals^{d'\times d'}$ be the orthogonal projector onto $\operatorname{span}\{[\bA]_{\cdot, j}: j \in \stylefactor\}^\perp$. Then the matrix $\bH^\perp\frac{\bU^\top \bU}{(m+1)n}\bH^\perp$ has exactly $(d-m)$ many positive eigen-values and $\bQ(+\infty)$ is the collection of the eigen-vectors corresponding to them. Furthermore, $[\bP (+\infty) \bA]_{\contentfactor, \contentfactor}$ is invertible.
\end{lemma}

\begin{proof}
We start by noticing that $\frac{\bU^\top \bU}{(m+1)n} = \bA \Sigma_{\bz}' \bA^\top$ where the matrix
\[
\Sigma_{\bz}' = \frac1{(m+1)n} \sum_{i = 1}^n \left[\bz_i\bz_i^\top + \sum_{j \in \cS} \bz_i^{(j)}\{\bz_i^{(j)}\}^\top  \right]
\] is invertible since $n \ge d+1$. Without loss of generality we assume that $\stylefactor = [m]$. Denoting $\cC_\bA \triangleq  \operatorname{col-space}(\bA)$ we notice that for any $\bx \in \cC_{\bA}^\perp$
\begin{equation}
\begin{aligned}
 & \bH^\perp\frac{\bU^\top \bU}{(m+1)n}\bH^\perp \bx \\
 & = \bH^\perp\bA \Sigma_{\bz}' \bA^\top \bx , ~~ \text{since} ~~  \bx \in \cC_{\bA}^\perp \subset \operatorname{span}\{[\bA]_{\cdot, \stylefactor}\}^\perp  \\
 & =  \mathbf{0}, ~~ \text{since} ~~ \bx \in \cC_{\bA}^\perp \,.
\end{aligned}
\end{equation} and for any $\bx \in \operatorname{span}\{[\bA]_{\cdot, j}: j \in \stylefactor\}$ it holds \begin{equation}
    \bH^\perp\frac{\bU^\top \bU}{(m+1)n}\bH^\perp \bx = \mathbf{0}\,.
\end{equation} Since $\operatorname{dim} \{ \cC_{\bA}^\perp \} = d' - d$ and $\operatorname{dim}\{[\bA]_{\cdot, j}: j \in \stylefactor\} = m$, counting the degrees of freedom we obtain that $\bH^\perp\frac{\bU^\top \bU}{(m+1)n}\bH^\perp$ doesn't have rank more than $d' - (d' - d) - m = d - m$. Expressing $\bA $ as  \begin{equation} \label{eq:qr}
    \bA = \tilde \bA \bU,
\end{equation} where $\tilde \bA\in \reals^{d' \times d}$ is an orthogonal matrix and $\bU\in \reals^{d \times d}$ is an upper triangular such that $\operatorname{span}\{[\bA]_{\cdot, j}: j \in \stylefactor\} = \operatorname{span}\{[\tilde \bA]_{\cdot, j}: j \in \stylefactor\}$. Such a decomposition can easily be obtained from Gram-Schmidt orthogonalization of the columns of $\bA$. Since $[\tilde\bA]_{\cdot, j} \in \operatorname{span}\{[\bA]_{\cdot, j}: j \in \stylefactor\}$ we notice that 
\begin{equation}
    \bH^\perp \tilde \bA = \begin{bmatrix}
    \mathbf{0}_{d' \times m} & [\tilde \bA]_{\cdot, \contentfactor}
    \end{bmatrix}
\end{equation} and defining $\bM \triangleq \bU \Sigma_{\bz}' \bU^\top $ which is an invertible matrix we notice that 
\begin{equation}
  \bH^\perp\frac{\bU^\top \bU}{(m+1)n}\bH^\perp =   [\tilde \bA]_{\cdot, \contentfactor} [\bM]_{\contentfactor, \contentfactor} [\tilde \bA]_{\cdot, \contentfactor}^\top \,.
\end{equation} and hence it has rank $|\contentfactor| = d-m$. This concludes a part of the lemma. 

Here, $[\bM]_{\contentfactor, \contentfactor}$ is a partition matrix of the non-negative definite matrix $\bM \triangleq \bU \Sigma_{\bz}' \bU^\top$. We consider it's spectral decomposition \begin{equation}
    [\bM]_{\contentfactor, \contentfactor} = \bW \bD \bW^\top
\end{equation} where $\bW \in \reals^{(d-m)\times (d-m)}$ is an orthogonal matrix and $\bD \in \reals^{(d-m) \times (d-m)}$ is diagonal with positive diagonal entries (since $[\bM]_{\contentfactor, \contentfactor}$ is full rank). This follows, 
\begin{equation}
    \bH^\perp \tilde \bA = [\tilde \bA]_{\cdot, \contentfactor} \bW \bD \bW^\top [\tilde \bA]_{\cdot, \contentfactor}^\top
\end{equation} where $[\tilde \bA]_{\cdot, \contentfactor} \bW$ is again a $\reals^{d'\times (d-m)}$ orthogonal matrix whose columns are the only eigen-vectors of $\bH^\perp\frac{\bU^\top \bU}{(m+1)n}\bH^\perp$ with positive eigen-values. Hence, 
\begin{equation} \label{eq:spectral}
    \bQ(+\infty) = [\tilde \bA]_{\cdot, \contentfactor} \bW\,.
\end{equation} Since $[\bP (+\infty)]_{\contentfactor, \cdot} = \bQ(+\infty)^\top $ we obtain
\[
\begin{aligned}
\relax  [\bP (+\infty)\bA]_{\contentfactor, \cdot} & = \bQ(+\infty)^\top \bA
\end{aligned}
\] where using \eqref{eq:qr} and \eqref{eq:spectral} we obtain
\[
\begin{aligned}
 &[\bP (+\infty)\bA]_{\contentfactor, \cdot}\\
 & = \bW^\top \big\{ [\tilde \bA]_{\cdot, \contentfactor}\big\}^\top \tilde \bA \bU\\
 &= \bW^\top \begin{bmatrix}
 \mathbf{0}_{(d-m)\times d} & \bI_{(d-m) \times (d-m)}
 \end{bmatrix}  \bU\\
 & = \bW^\top \bU_{\contentfactor, \cdot}\,.
\end{aligned}
\] Finally we obtain 
\[
[\bP (+\infty)\bA]_{\contentfactor, \contentfactor} = \bW^\top \bU_{\contentfactor, \contentfactor}\,.
\] where, following that $\bW$ is an orthogonal matrix and $\bU$ an invertible upper-triangular matrix, both $\bW$ and $\bU_{\contentfactor, \contentfactor}$ are invertible. This implies $[\bP (+\infty)\bA]_{\contentfactor, \contentfactor}$ is invertible, and we conclude the lemma. 
\end{proof}

\subsection{Proof of Corollary \ref{cor:sparse-recovery}}
\begin{proof}
Note that the convergences in Theorem \ref{th:style-factors} are almost sure convergences.
For each $j\in \stylefactor$ we define $\bB_j$ as the probability one event on which  $\bA^\top \hat \bp_j \to \beta_j e_j$. We further define $\bB \triangleq \cap_{j \in \stylefactor} \bB_j$ which is again a probability one event (an intersection of finitely many probability one events), and on the event the convergences hold simultaneously over $j \in \stylefactor$. 

Drawing our attention to the conclusions in Theorem \ref{th:sparse-recovery}, we define $\bC_n$ as the event that 
\[
\begin{aligned}
 \bC_n \triangleq \{\text{The conclusions in Theorem \ref{th:sparse-recovery} hold}\\
~~ \text{with sample size} ~~ n\}
\end{aligned}
\] for each $n \ge d+1 $ and notice that $\bbP_\bz(\bC_n) \ge 1- \kappa^n$. This implies
\[
\sum_{n \ge d+1} \bbP_\bz(\bC_n^c) \le \sum_{n \ge d+1}  \kappa^n <\infty\,.
\] We define $\bC = \{\bC^c_n ~~\text{holds infinitely often}\}$ and use the first Borel-Cantelli lemma to conclude that
\[
\bbP_\bz(\bC^c) = 0 ~~ \text{or,} ~~ \bbP_\bz(\bC) = 1\,.
\]
Note that the event $\bC$ is the same as the event that \{$\bC_n$ holds all but finitely often\}, or that \{The conclusions in Theorem \ref{th:sparse-recovery} holds at $n \to \infty$ \}, which are probability one events.  Thus it follows that $\bB \cap \bC$, an event on which the conclusions in both the theorems \ref{th:style-factors} and \ref{th:sparse-recovery} simultaneously hold (for $n \to \infty$), is a probability one event. Hence, we conclude that with probability one  the following hold at the limit $n\to \infty$: (1) $[\bP \bA]_{\stylefactor, \stylefactor}$ is a diagonal matrix,  (2) $[\bP \bA]_{\stylefactor, \contentfactor} = \mathbf{0}$, (3) $[\bP \bA]_{\stylefactor, \contentfactor} = \mathbf{0}$, and (4) $[\bP \bA]_{\contentfactor, \contentfactor}$ is invertible. These are the exact conditions in Definition \ref{def:direp} that are required for sparse recovery. Hence, the corollary follows. 
\end{proof}

\section{Details for synthetic data study in  \S\ref{sec:sim}}
\label{sec:sim-supp}

\subsection{The latent factors}

The latent factors are generated as 
\begin{equation}
   \reals^{10} \ni \bz \sim \bN(\mathbf{0}, \Sigma),
\end{equation} where $\Sigma$ is a $10 \times 10$ covariance matrix whose entries are described below. For $i, j \in \{1, \dots, 10\}$ 
\begin{equation}
    [\Sigma]_{i, j} = \begin{cases}
    1 & i = j ,\\
    \rho \in [0, 1) & (i, j) = (1, 2), ~~\text{or,} ~~ (i, j) = (2, 1), \\
    0 & \text{otherwise}\,.
    \end{cases}
\end{equation} We fix the first five coordinates as style factors, \ie\ $\stylefactor = \{1, \dots , 5\}$ and the rest of them as content factors, \ie\ $\contentfactor = \{6, \dots, 10\}$. Note that style and content factors are independent, \ie, \[
[\bz]_{\stylefactor} \perp [\bz]_{\contentfactor}\,.
\]

We draw $\{\bz_i\}_{i = 1}^{n}\stackrel{\iid}{\sim} \bN (\mathbf{0}, \Sigma)$.  

\subsection{The entangled representations}

We fix $d' = 10$ and obtain entangled representations as 
\begin{equation}
    \bu = \bA \bz = \bL \bU \bz\in  \reals^{10}\,,
\end{equation} where $\bA = \bL\bU$, $\bU$ is a randomly generated $10 \times 10 $ orthogonal matrix and $\bL$ is a $10 \times 10 $ lower triangular matrix described below. 
\begin{equation}
    [\bL]_{i, j} = \begin{cases}
    1 & i = j ,\\
   0.9  & i > j, \\
    0 & i < j\,.
    \end{cases}
\end{equation} We use the same orthogonal matrix throughout our experiment. Note that both $\bL$ and $\bU$ are invertible and hence $\bA = \bL \bU$ is also invertible. 

\subsection{Sample manipulations and annotations}

For each of the latent factors $\bz$ and $j$-th style coordinates we obtain two manipulated latent factors which we denote as $\bz^{(j), +}$ and $\bz^{(j), -}$ and their description follow. $\bz^{(j), +}$ (resp. $\bz^{(j), -}$) sets the $j$-th coordinate to its positive (resp. negative) absolute value, \ie\ 
\[
[\bz^{(j), +}]_j = |[\bz]_j| ~~ (\text{resp.} ~~ [\bz^{(j), -}]_j = -|[\bz]_j|)\,,
\]
 and annotate it as $+1$ (resp. $-1$). 
Since the first two coordinates are correlated with correlation coefficient $\rho$, if either of them changes by the value $\delta$ then the other one changes by $\rho \delta$. We provide a concrete example of change in the second coordinate for the change in the first coordinate, but a similar change happens vice-versa. Since 
\[
[\bz^{(1), +}]_1 - [\bz]_1 = |[\bz]_1| - [\bz]_1\,,
\] it must hold 
\[
[\bz^{(1), +}]_2 - [\bz]_2 = \rho\big(|[\bz]_1| - [\bz]_1\big)\,.
\] Note that one of $\bz^{(j), +}$ and $\bz^{(j), -}$ is exactly same as $\bz$. We obtain the entangled representations as $\bu^{(j), +} = \bA\bz^{(j), +}$ and $ \bu^{(j), -} = \bA\bz^{(j), -}$.

\subsection{Style factor estimations} Note that $\bA$ is invertible and hence the covariance matrix of $\bu = \bA \bz$ is also invertible. In this case, the minimum norm least square problem in \eqref{eq:jth-direction} is the simple least square problem. For $j\in \stylefactor$ we describe the estimation of $j$-th style factor below. 
\begin{equation}
    \begin{aligned}
\relax    [\hat \bz]_j \triangleq \hat\bp_j^\top \bu, ~~ \text{where} \\
    \hat \bp_j \triangleq \underset{a \in \reals, \bp\in \reals^{d'}}{\argmin}  \frac{1}{2n} \sum_{i = 1}^n \Big[\big(+1 - a - \bp^\top \bu_i^{(j), +}\big)^2\\
    + \big(-1 - a - \bp^\top \bu_i^{(j), -}\big)^2\Big] 
    \end{aligned}
\end{equation}

\section{Experimental details}
\label{supp:exp-details}

\subsection{Feature extractors and image style generation (image transformations)}
\label{supp:exp-features}
\paragraph{MNIST data} For the colored MNIST experiment, we train a multilayer perceptron (MLP) feature extractor, a 3-layer neural network with ReLU activation function and a hidden layer of size 50. The dataset for training the feature extractor is obtained by randomly coloring some original MNIST images green and some of them red. We then train the feature extractor by making it predict both the color of the image and the digit label. During training, we use a batch size of 256 and a learning rate of 0.001.

After training the feature extractor, we use it to extract features from MNIST images that we use in the experiments. For the experiments, we use original MNIST images and MNIST images colored green. 

\paragraph{CIFAR-10 data} For experiments on CIFAR-10, we use two different feature extractors, \texttt{Supervised} and \texttt{SimCLR}. For \texttt{Supervised}, we use a \texttt{Supervised} model that was pre-trained on ImageNet \citep{russakovsky2015imagenet} from Pytorch's Torchvision package \footnote{https://pytorch.org/vision/stable/index.html}. For \texttt{SimCLR}, we first train a \texttt{SimCLR} \footnote{https://github.com/spijkervet/SimCLR} model on the original CIFAR-10 dataset before using it to extract features.

We transform the CIFAR-10 dataset four different ways to generate new sets of data that we use in our various experiment settings. The first set of data is generated by rotating the original CIFAR-10 data at angle 15 degrees, the second set is generated by applying contrast to the original CIFAR-10 data using a contrast factor of 0.3, the third set is generated by blurring the original CIFAR-10 data using a sigma value of 0.3, and the fourth set is generated by making the original CIFAR-10 images saturated using a saturation factor of 5. We selected transformation parameters that transformed the original data without changing it into something completely different and unrecognizable. For the experiments, we extract features from the original CIFAR-10, rotated CIFAR-10, contrasted CIFAR-10, blurred CIFAR-10, and saturated CIFAR-10 data. We then use the extracted features to perform experiments as described in \S\ref{sec:experiments} and \S\ref{sec:method}.

\paragraph{ImageNet data} For experiments on ImageNet, we use a pre-trained ResNet-50 model to extract features from the ImageNet \citep{russakovsky2015imagenet} dataset. The ImageNet data that we use contains 1,281,167 images for training, 50,000 images for validation, and 1000 classes.

The ImageNet dataset was used to demonstrate PISCO's ability to scale and generalize under distribution shifts. For generalization, we use the code published by \citet{geirhos2018imagenet} to generate four stylized ImageNet datasets with covariate distribution shifts by applying four styles on original ImageNet images. The styles applied are ``dog sketch'', ``woman sketch'', ``Picasso self-portrait'', and ``Picasso dog'' (see Figure \ref{fig:stylized_images}). Two of the styles, ``dog sketch'' and ``Picasso dog'', were generated by DALL$\cdot$E 2. For the ``Picasso dog'' style, the prompt used to generate it from DALL$\cdot$E 2 was, ``portrait of a dog in Picasso’s 1907 self-portrait style''. For the ``dog sketch'' style, the prompt used to generate it from DALL$\cdot$E 2 was, ``artistic hand drawn sketch of a dog face''. The other two styles, ``woman sketch'' and ``Picasso self-portrait'', were downloaded from a GitHub repository\footnote{\url{https://github.com/xunhuang1995/AdaIN-style/tree/master/input/style}} of the style transfer project by \cite{huang2017arbitrary}. The generated stylized ImageNet data is then used to test PISCO's out-of-distribution (OOD) generalization capability. Figure \ref{fig:stylized_images} shows example images from the ImageNet dataset and the four styles that we use to generate the stylized ImageNet sets.

\begin{figure*}
    \centering
    \includegraphics[scale=0.35]{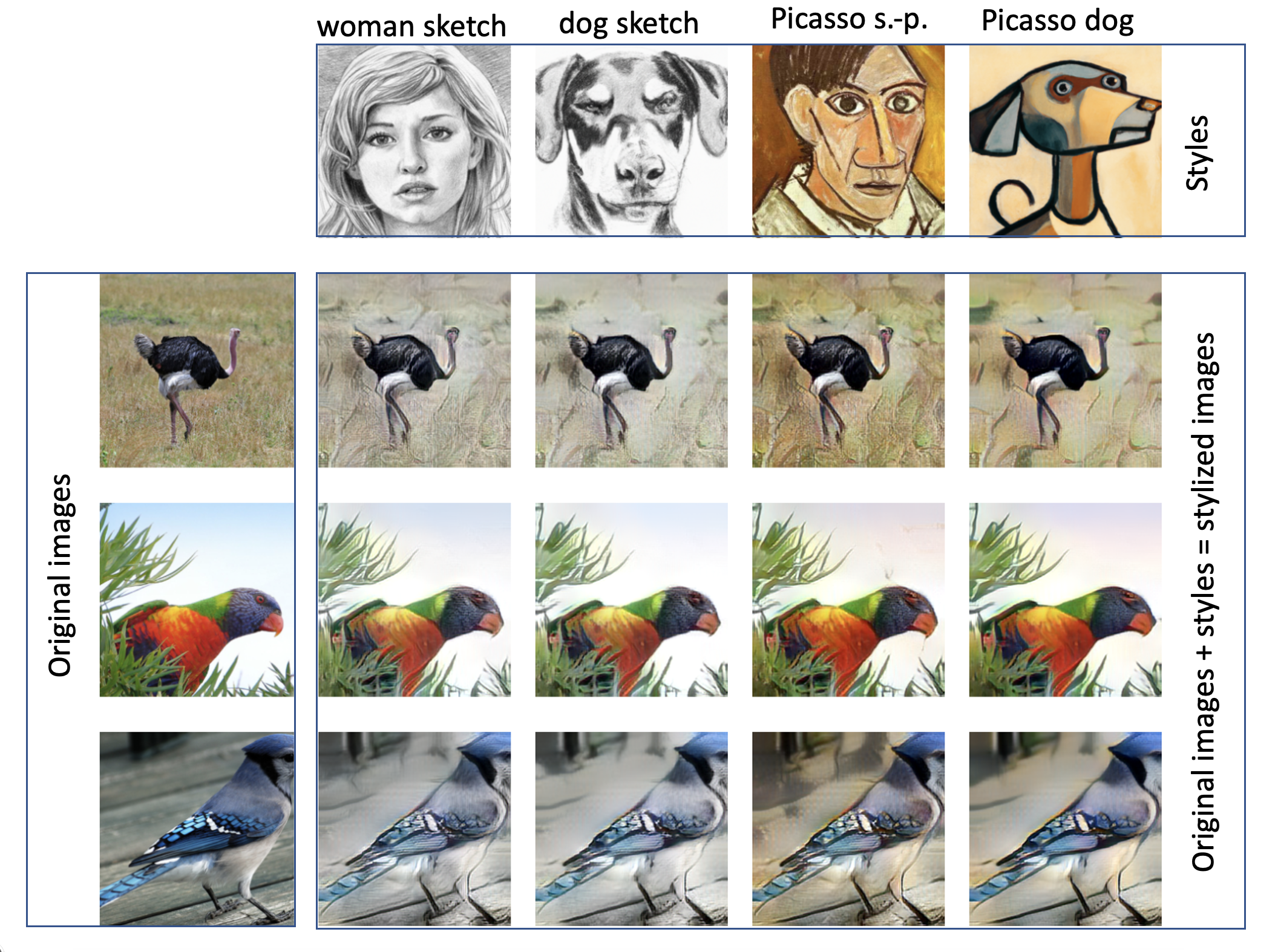}
    \caption{Example images from the ImageNet dataset with the styles applied to them.}
    \label{fig:stylized_images}
\end{figure*}

\subsection{Spurious correlations and error bars}
Spurious correlations for the experiments in both MNIST and CIFAR-10 datasets are created by first dividing images in each dataset into two halves, the first half contains images with class label below 4 and the second half contains images with class 4 and above. We then create datasets, where the image label is spuriously correlated with the image style, \ie\ color green, rotation, contrast, blur, or saturation, as follows: in the training dataset, images from the first half are transformed with probability $\alpha$ and images from the second half are transformed with probability $1-\alpha$. In the test dataset, we do the reverse of what we did in the training data; images from the first half are transformed with probability $1-\alpha$ and images from the second half are transformed with probability $\alpha$. We perform experiments and report results for $\alpha$ values $0.5,0.75,0.90,0.95,0.99,1.0$. At $\alpha = 0.5$, the train and test data have the same distribution, and at $\alpha = 1$, the spurious correlations between labels and styles is extreme. 

Results on experiments where there are spurious correlations between label and transformations (styles) in the data are reported in plots eg. Figure \ref{fig:resnet_results}, Figure \ref{fig:simclr_results}, Figure \ref{fig:resnet_results_1}, etc. The reported results are over 10 restarts. We include error bars in the plots, but the errors are small so the error bars are not very visible. 

\subsection{Training and testing logistic regression models for classification}
\paragraph{MNIST Data}
For MNIST data, the digits labels are from 0 to 9. For baseline results, we train and test the logistic regression model using all the features extracted using the MLP feature extractor. For \method\ results, we discard the style feature corresponding to color green and train and test the logistic regression model using only the remaining content features.

\paragraph{CIFAR-10 Data}
For CIFAR-10 data, we use the original image labels to train and test the logistic regression model. For baseline results, we train and test the logistic regression model using features extracted using \texttt{SimCLR} or \texttt{Supervised}. For \method\ results, we first discard the four style features corresponding to rotation, contrast, blur, and saturation and then train and test the logistic regression model using only the remaining content features.

\paragraph{ImageNet Data}
For ImageNet data, similar to CIFAR-10 and MNIST settings, after fitting the logistic regression model on features extracted using ResNet-50 to obtain baseline results, ImageNet features are post-processed with \method\ to isolate content and style, and then style features get dropped when fitting the model for OOD generalization. The batch size used when training the logistic regression model on ImageNet was 32768, the learning rate was 0.0001, and the number of epochs was 50.

For more experimental details, check out our released code on GitHub\footnote{\url{https://github.com/lilianngweta/PISCO}}.

\textbf{Selecting the hyperparameter $\lambda$:} $\lambda$ trades-off disentanglement with the preservation of variance in the data (second and first terms in eq. (\ref{eq:unchanged-content}), correspondingly). The easiest, no-harm, way to select $\lambda$ is to increase it until the in-distribution performance starts to degrade. In our reported results, $\lambda=1$ is the best $\lambda$ value because it improves OOD performance without affecting the in-distribution accuracy. Further increasing $\lambda$ can provide additional OOD gains at the cost of in-distribution performance. Our method works best when the styles are easy to predict from the original representations with a linear model (see first column in Tables \ref{tb:correlations_resnet} and \ref{tb:correlations_simclr}; note that this is also easy to evaluate at training time). For example, blur is hard to predict and \method\ with larger values of $\lambda$ degrades the corresponding OOD performance in Tables \ref{tb:ood_gen_resnet} and \ref{tb:ood_gen_simclr}, while rotation is easier to predict and PISCO with $\lambda=10$ improves the performance in both tables. When a given style is hard to predict, it means that the representation is robust to it (as is the case with blur and some other styles for \texttt{SimCLR} representations) and it might make sense to exclude it when applying PISCO to avoid unnecessary trade-offs with the variance preservation. However, if strong spurious correlation is present, PISCO improves performance even for harder to predict styles (see Figures \ref{fig:resnet_results} and \ref{fig:simclr_results}).


\section{Additional results and selecting the hyperparameter $\eta$}
\label{supp:results}

In this section, we present results on the MNIST dataset (see \S\ref{supp:mnist_results}). We also present additional ImageNet results (see \S\ref{supp:add_imagenet_results}) and additional CIFAR-10 results (see \S\ref{supp:add_cifar_results}) on different values of the hyperparameter $\eta$, as well as ImageNet results for different values of $\lambda$. In experiments for all datasets (MNIST, CIFAR-10, and ImageNet), we have presented results for when $\eta = 0.95$. Here we present additional ImageNet and CIFAR-10 results when $\eta$ is 0.90, 0.93, 0.95 (for ImageNet only), 0.98, and 1.0 to demonstrate its impact on performance. We also presented ImageNet results for when $\lambda=1$ in the main paper; here we present additional results for when $\lambda$ is 10 and 50 to demonstrate how varying $\lambda$ affects performance on ImageNet data.


\subsection{Colored MNIST experiment}
\label{supp:mnist_results}

In this experiment, our goal is to isolate color green from the digit class. First, to obtain representations with entangled color green and digit information, we train a neural network feature extractor to predict both color and digit label (see \S\ref{supp:exp-details} for details) and then use it to extract features from original and green MNIST images that we use in the experiment. In this experiment, we have a single style factor, i.e. color green, $m=1$. We learn post-processing feature transformation matrices $\bP(\lambda)$ with \method\ as in Algorithm \ref{alg:method} and report results for $\lambda \in \{1, 10, 50\}$.

Next, we create a dataset where the label is spuriously correlated with the color green, similar to Colored MNIST \citep{arjovsky2019Invariant}. Specifically, in the training dataset, images from the first half of the classes are colored green with probability $\alpha$ and images from the second half of the classes are colored green with probability $1-\alpha$. In test data the correlation is reversed, i.e., images from the first half of the classes are colored green with probability $1-\alpha$ and images from the second half with probability $\alpha$ (see \S\ref{supp:exp-details} for additional details). Thus, for $\alpha=0.5$ train and test data have the same distribution where each image is randomly colored green, and $\alpha=1$ corresponds to the extreme spurious correlation setting.

For each $\alpha$ we train and test a linear model on the original representations and on \method\ representations (discarding the learned color green factor) for varying $\lambda$. We summarize the results in Figure \ref{fig:mnist_results}. \method\ outperforms the baseline across all values of $\alpha$ and matches the baseline accuracy when there is no spurious correlation and train and test distributions are the same, i.e., $\alpha=0.5$. Thus, our method provides a significant OOD accuracy boost while preserving the in-distribution accuracy.

\begin{figure}[h]
\centering
\includegraphics[scale=0.46]{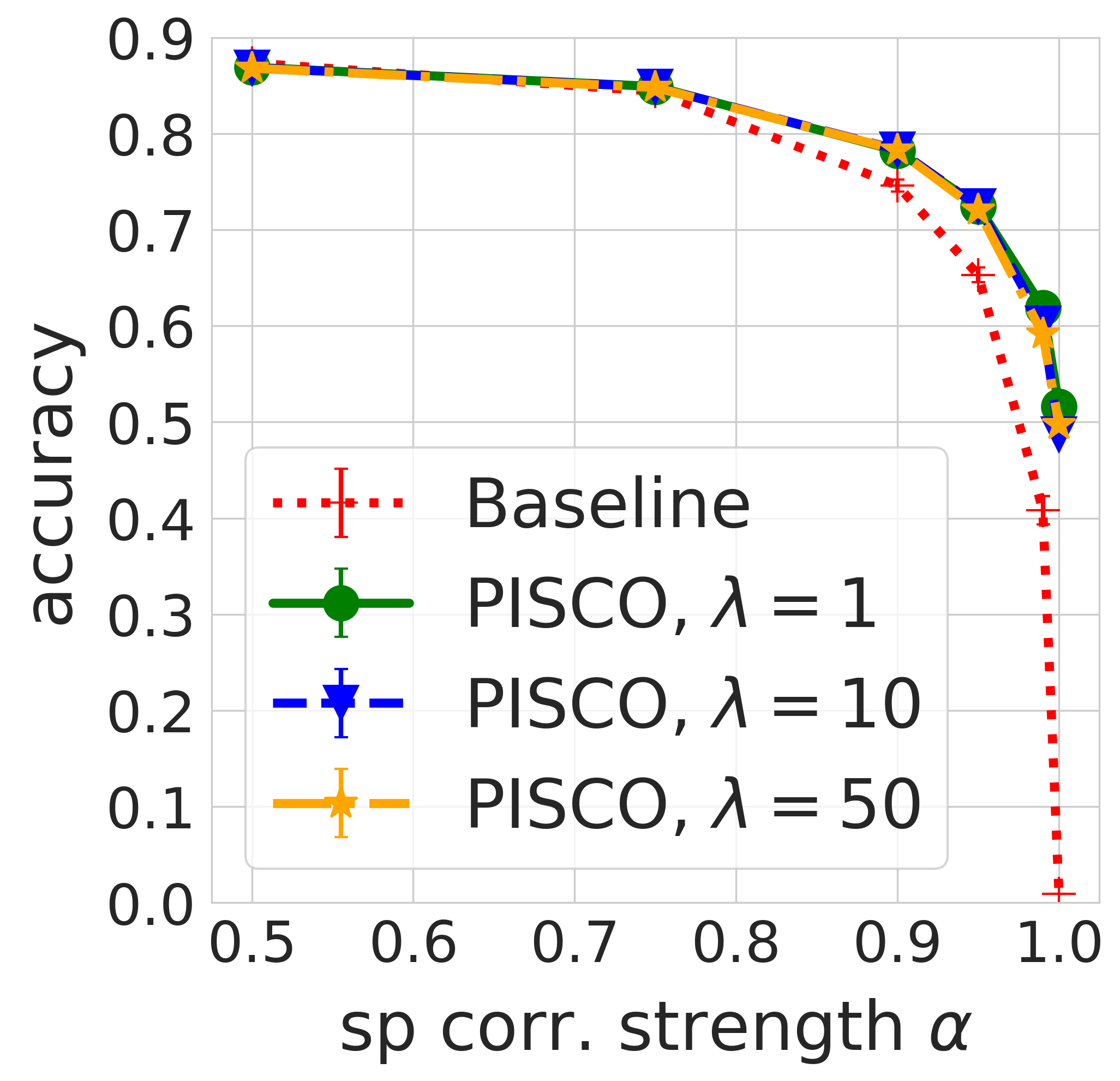}
\caption{OOD accuracy on Colored MNIST dataset where the label is spuriously correlated with color green. The strength of the correlation is controlled by $\alpha$. For each $\alpha$, we train a logistic regression on the training data using the corresponding representations and report test accuracy. \method\ is robust to spurious correlations across all values of $\alpha$ with only a slight accuracy drop for extreme $\alpha$ values. The baseline is a multilayer perceptron (MLP). 
}
\label{fig:mnist_results}
\end{figure}

\subsection{Additional transformed CIFAR-10 experiment results}
\label{supp:add_cifar_results}

\paragraph{Results for $\eta = 0.90$:} Results for when $\eta = 0.90$ can be found in Figure \ref{fig:resnet_results_090}, Figure \ref{fig:simclr_results_090}, Table \ref{tb:ood_gen_resnet_090}, and Table \ref{tb:ood_gen_simclr_090}. 

\paragraph{Results for $\eta = 0.93$:} Results for when $\eta = 0.93$ can be found in Figure \ref{fig:resnet_results_093}, Figure \ref{fig:simclr_results_093}, Table \ref{tb:ood_gen_resnet_093}, and Table \ref{tb:ood_gen_simclr_093}. 

\paragraph{Results for $\eta = 0.98$:} Results for when $\eta = 0.98$ can be found in Figure \ref{fig:resnet_results_098}, Figure \ref{fig:simclr_results_098}, Table \ref{tb:ood_gen_resnet_098}, and Table \ref{tb:ood_gen_simclr_098}. 

\paragraph{Results for $\eta = 1.0$:} Results for when $\eta = 1.0$ can be found in Figure \ref{fig:resnet_results_1}, Figure \ref{fig:simclr_results_1}, Table \ref{tb:ood_gen_resnet_1}, and Table \ref{tb:ood_gen_simclr_1}. When $\eta = 1.0$, it means the number of features in the baseline is the same as the number of features learned using \method\ and as a result, we observe the in-distribution performance of \method\ is almost the same as that of baseline methods even for higher values of $\lambda$.

\paragraph{Overall CIFAR-10 results discussion.} Even with different values of $\eta$, \method\ still outperforms the baselines in almost all cases. An expected observation from the results is as $\eta$ increases, the in-distribution performance of \method\ goes up even for high values of $\lambda$ and its OOD performance slightly goes down.



\begin{figure}
\captionsetup[subfigure]{justification=centering}
     \centering
     \begin{subfigure}[b]{0.22\textwidth}
         \centering
         \includegraphics[width=\textwidth]{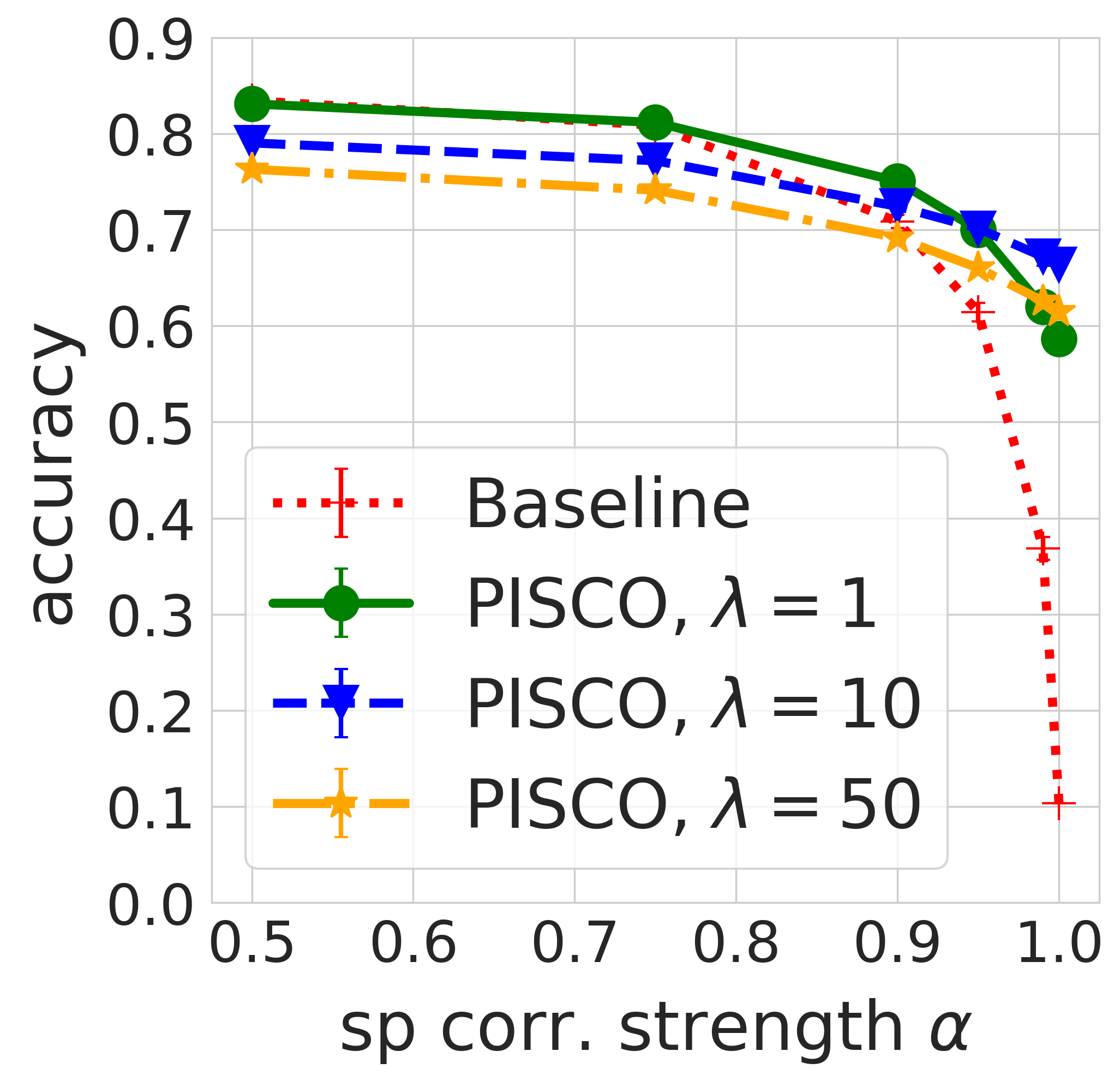}
         \caption{Rotation - \texttt{Supervised}}
         \label{fig:rotat-resnet-supp1}
     \end{subfigure}
     \hfill
     \begin{subfigure}[b]{0.22\textwidth}
         \centering
         \includegraphics[width=\textwidth]{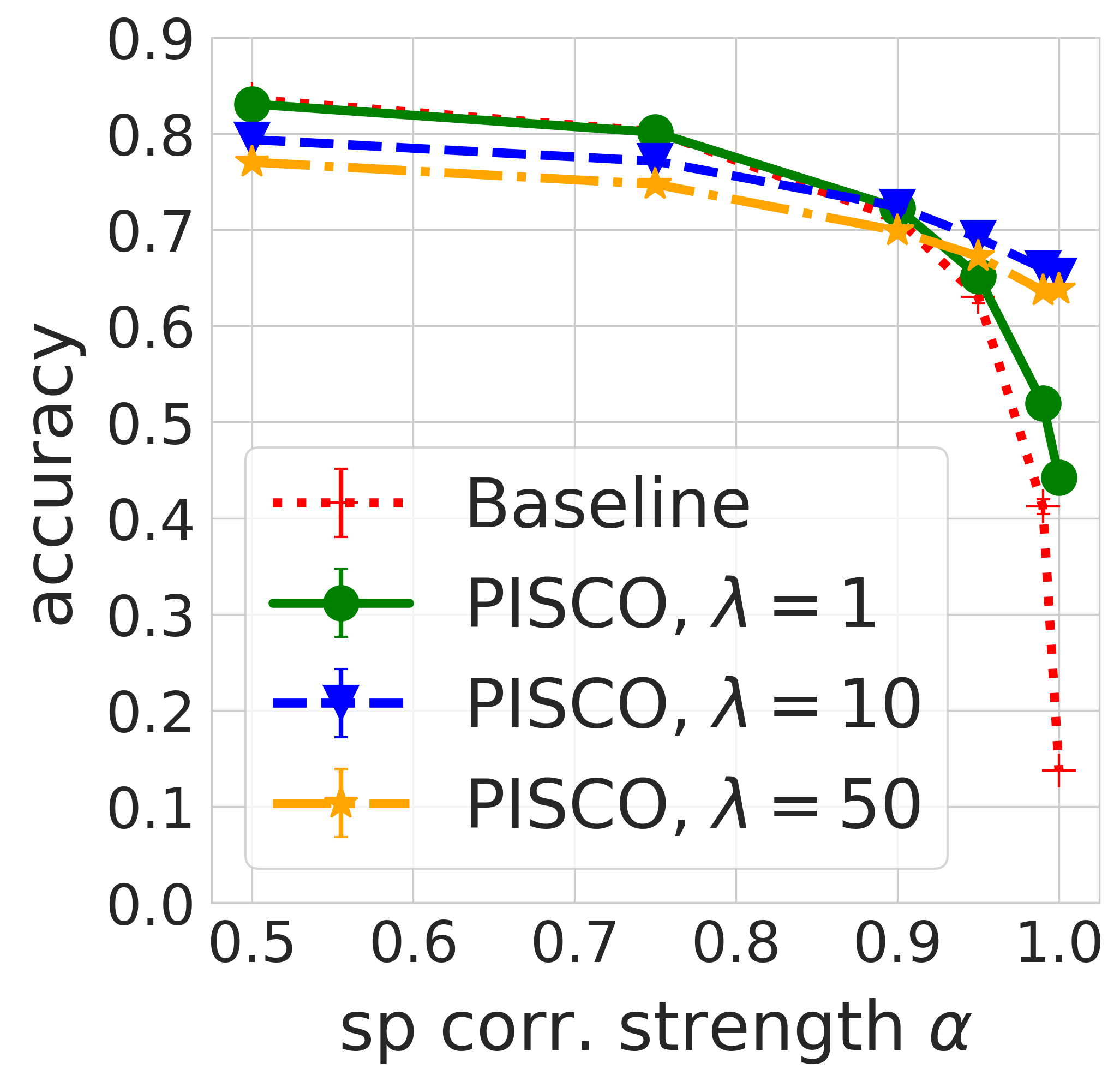}
         \caption{Contrast - \texttt{Supervised}}
         \label{fig:contr-resnet-supp1}
     \end{subfigure}
     \hfill
     \begin{subfigure}[b]{0.22\textwidth}
         \centering
         \includegraphics[width=\textwidth]{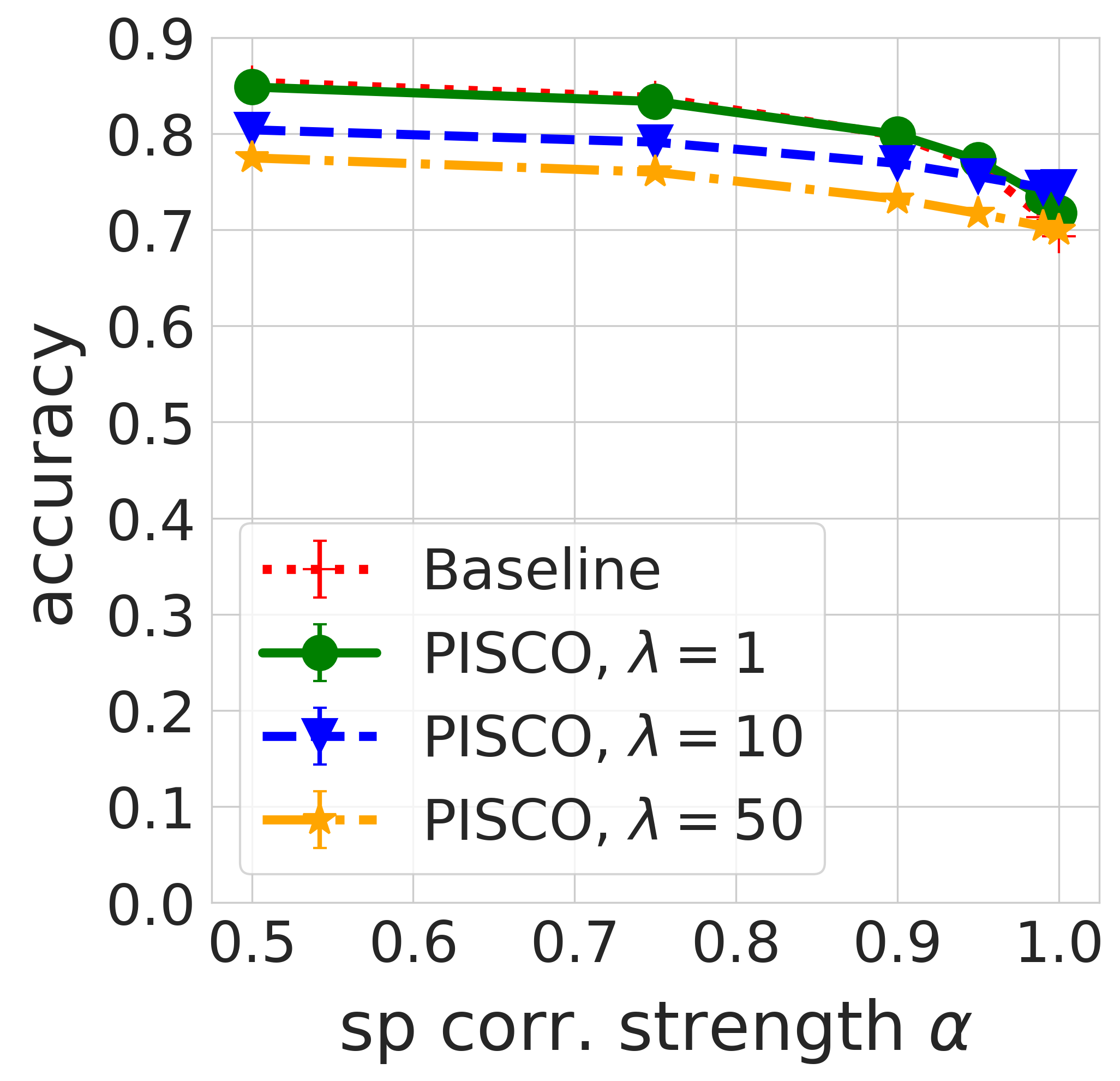}
         \caption{Blur - \texttt{Supervised}}
         \label{fig:blur-resnet-supp1}
     \end{subfigure}
     \hfill
     \begin{subfigure}[b]{0.22\textwidth}
         \centering
         \includegraphics[width=\textwidth]{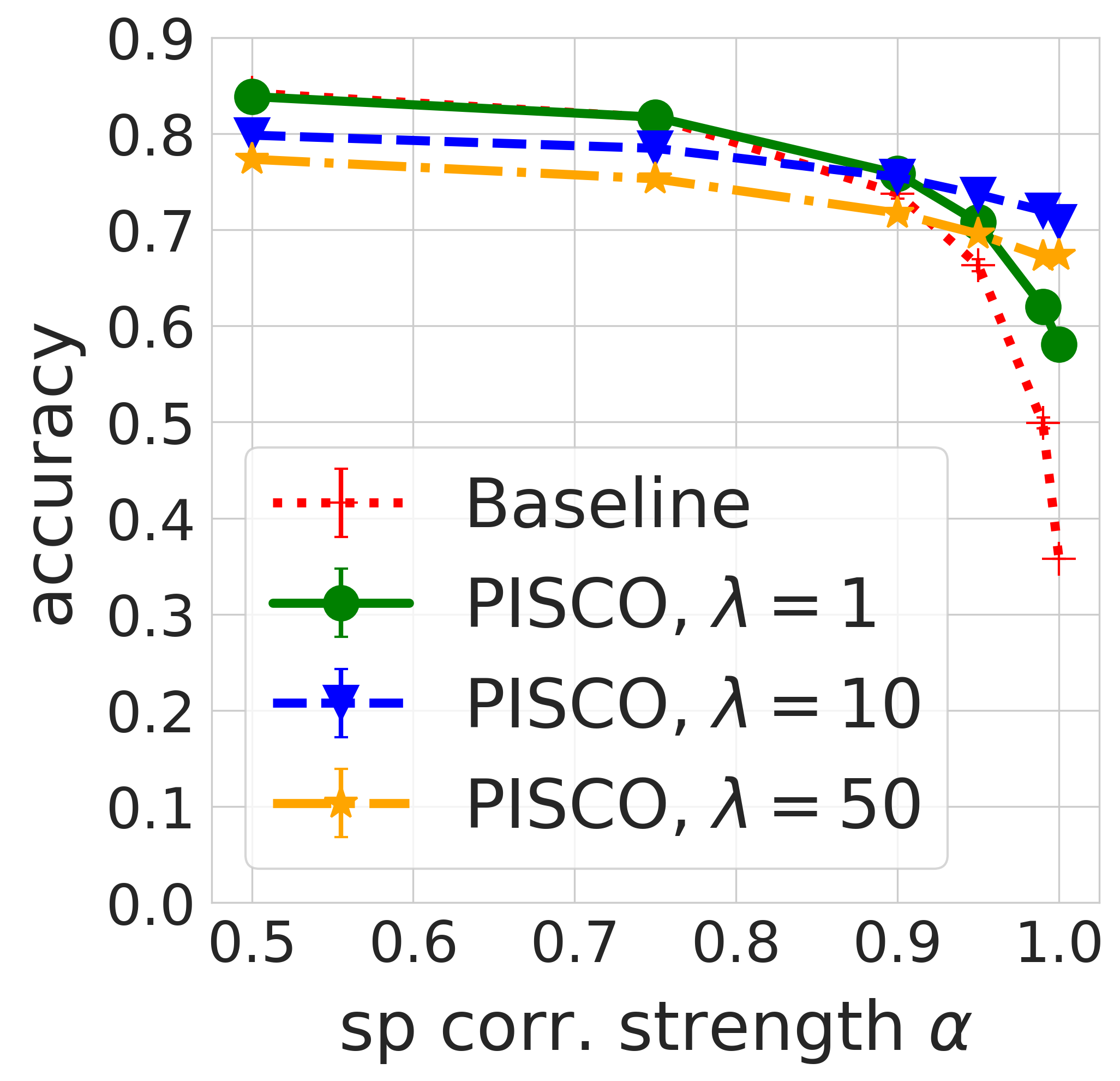}
         \caption{Satur. - \texttt{Supervised}}
         \label{fig:sat-resnet-supp1}
     \end{subfigure}
         \centering
         \caption{$\eta = 0.90$, OOD performance of \texttt{Supervised} representations on CIFAR-10 where the label is spuriously correlated with the corresponding transformation. \method\ significantly improves OOD performance, especially in the case of rotation. Both $\lambda=1$ and $\lambda=10$ preserve in-distribution accuracy, while larger $\lambda=50$ may degrade it as per \eqref{eq:loss}.
         }
         \label{fig:resnet_results_090}
\end{figure}

\begin{figure}
\captionsetup[subfigure]{justification=centering}
     \centering
     \begin{subfigure}[b]{0.22\textwidth}
         \centering
         \includegraphics[width=\textwidth]{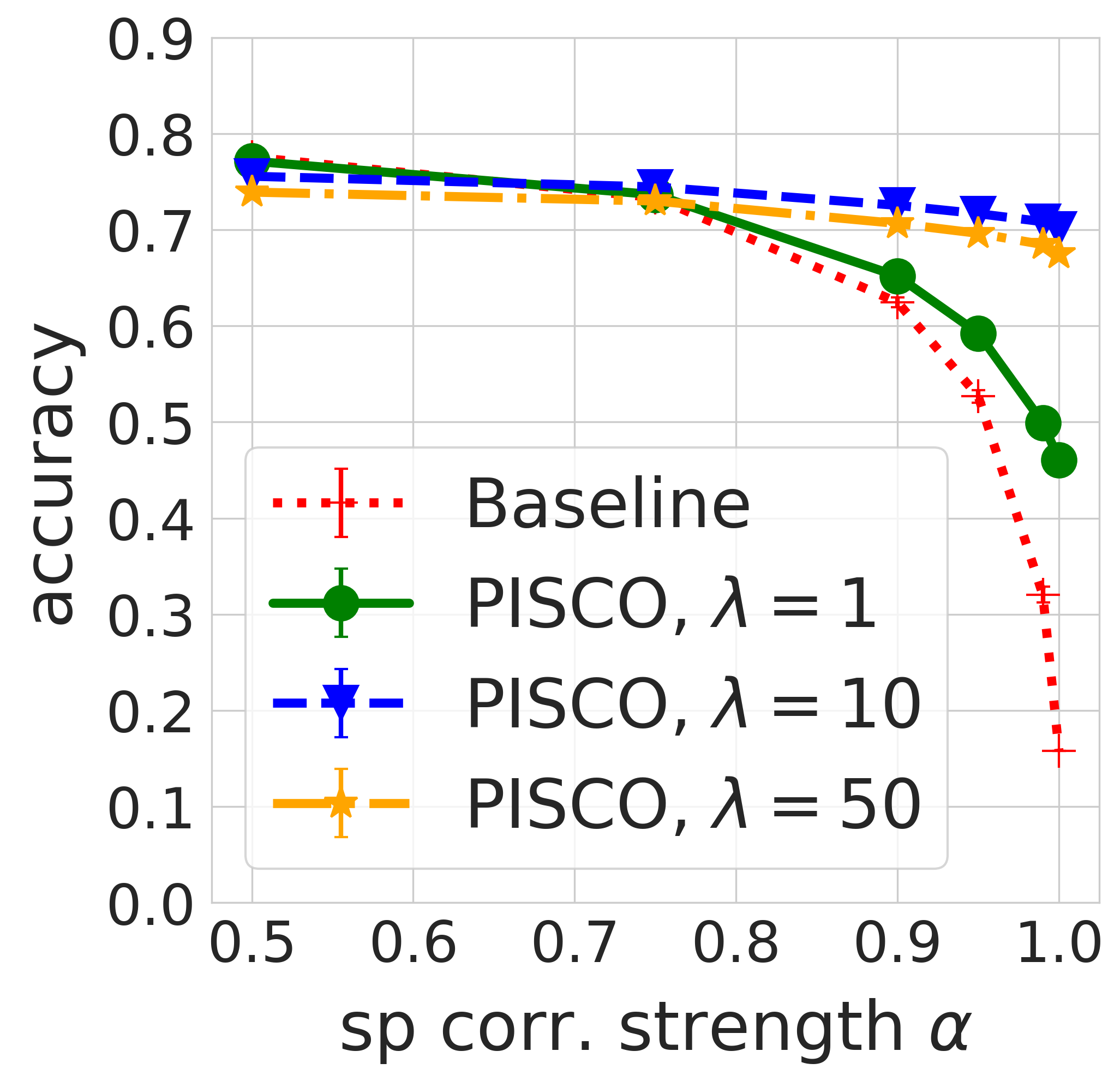}
         \caption{Rotation -  \texttt{SimCLR}}
         \label{fig:rotat-simclr-supp1}
     \end{subfigure}
     \hfill
     \begin{subfigure}[b]{0.22\textwidth}
         \centering
         \includegraphics[width=\textwidth]{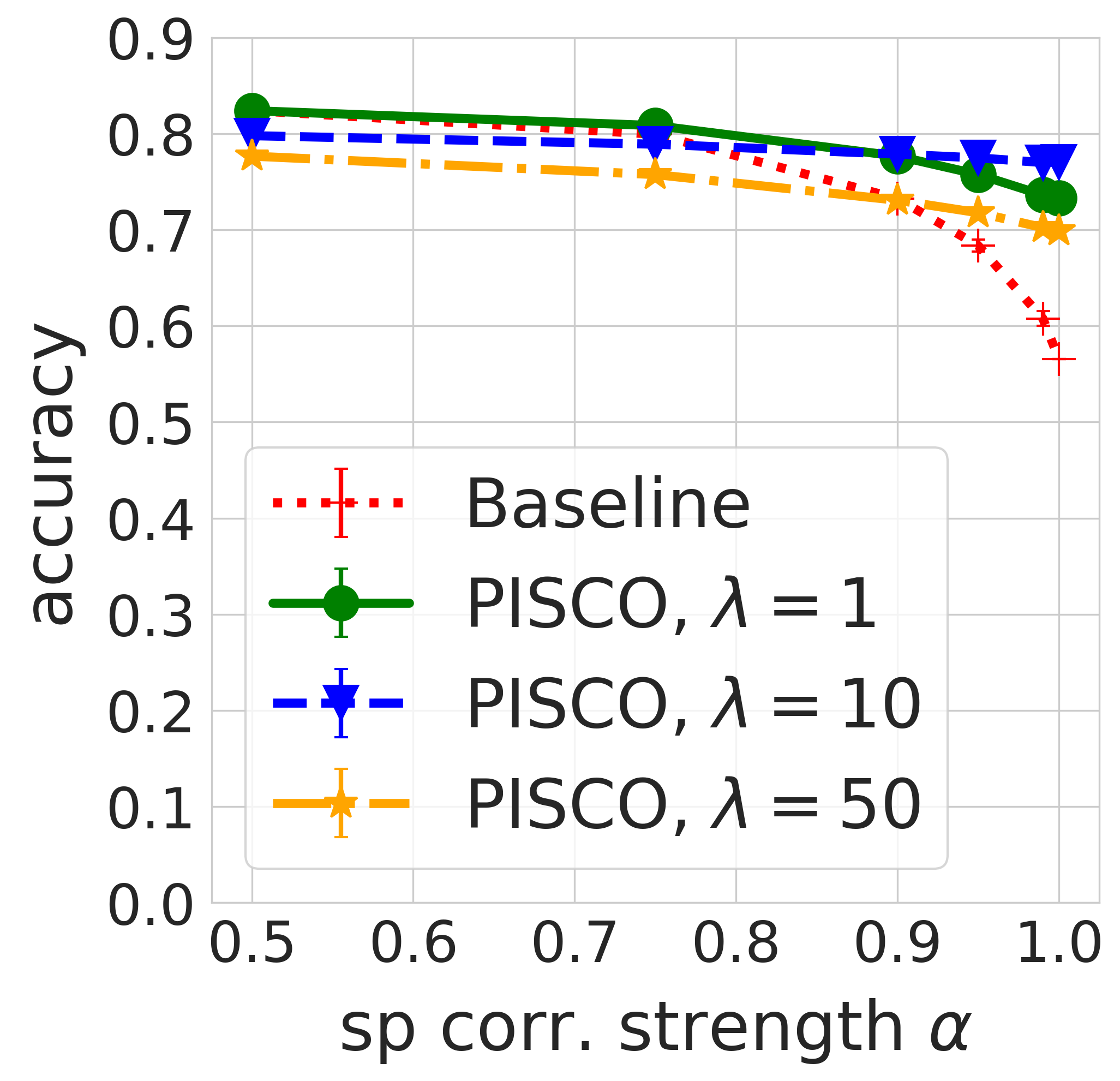}
         \caption{Contrast -  \texttt{SimCLR}}
         \label{fig:contr-simclr-supp1}
     \end{subfigure}
     \hfill
     \begin{subfigure}[b]{0.22\textwidth}
         \centering
         \includegraphics[width=\textwidth]{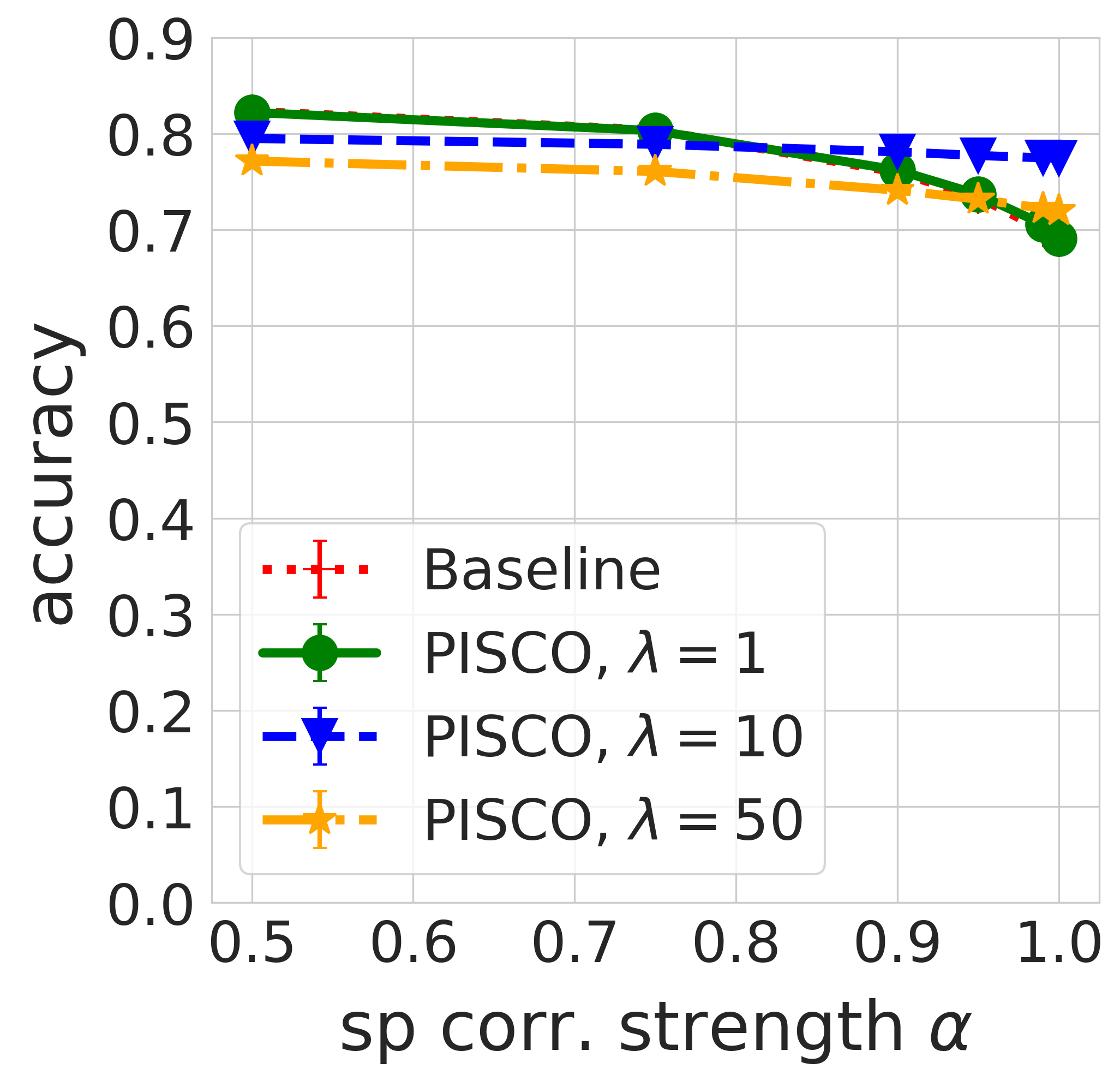}
         \caption{Blur - \texttt{SimCLR}}
         \label{fig:blur-simclr-supp1}
     \end{subfigure}
     \hfill
     \begin{subfigure}[b]{0.22\textwidth}
         \centering
         \includegraphics[width=\textwidth]{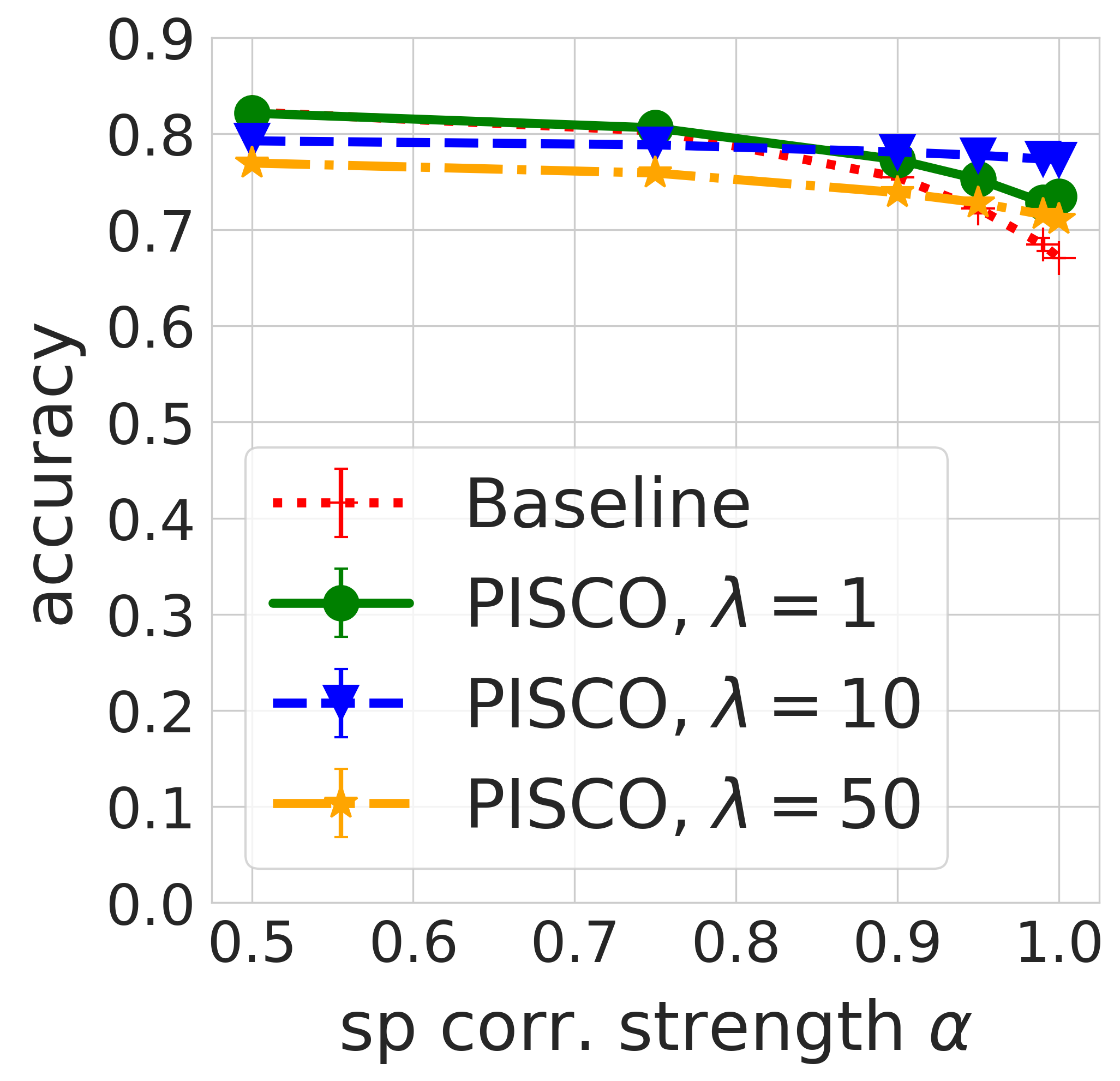}
         \caption{Saturation - \texttt{SimCLR}}
         \label{fig:sat-simclr-supp1}
     \end{subfigure}
         \centering
         \caption{$\eta = 0.90$, OOD performance of \texttt{SimCLR} representations on CIFAR-10 where the label is spuriously correlated with the corresponding transformation. Results are analogous to Figure \ref{fig:resnet_results_090}. The \texttt{SimCLR} baseline representations are less sensitive to contrast and saturation but remain sensitive to rotation.
         }
         \label{fig:simclr_results_090}
        
\end{figure}

\begin{table}
 \caption{$\eta = 0.90$, Performance of \texttt{Supervised} representations on CIFAR-10 test set in-distribution, i.e., no transformation (referred to as ``none''; last row), and OOD when modified with the corresponding transformation. \method\ with $\lambda=1$ provides significant improvements for rotation, contrast, and saturation, while preserving in-distribution accuracy.
 }
 \label{tb:ood_gen_resnet_090}
 \centering

\begin{tabular}{lccccc}
\toprule
Style &  \makecell{Baseline\\\scriptsize{(\texttt{Supervised})}} &  \makecell{\method\\(\small{$\lambda = 1$})} &  \makecell{\method\\(\small{$\lambda = 10$})} &  \makecell{\method\\(\small{$\lambda = 50$})} \\
\midrule
      rotation &  0.678  &  \textbf{0.741}  & 0.722  & 0.693                 \\
    contrast & 0.625  & 0.680 & \textbf{0.741} & 0.718 \\
      saturation & 0.699 &                \textbf{0.759} &                0.742 &                0.714  \\
      blur &  \textbf{0.817} &                \textbf{0.817} &                0.777 &                0.750   \\
     
     none &  \textbf{0.873} & 0.869 & 0.823 &  0.791    \\
\bottomrule
\end{tabular}
\end{table}

\begin{table}
 \caption{
 $\eta = 0.90$, Performance of \texttt{SimCLR} representations on CIFAR-10 test set in-distribution, i.e., no transformation (referred to as ``none''; last row), and OOD when modified with the corresponding transformation. \texttt{SimCLR} features are robust to these transformations and perform similarly to \method\ with $\lambda=1$.
 }
 \label{tb:ood_gen_simclr_090}
 \centering

\begin{tabular}{lccccc}
\toprule
 Style &  \makecell{Baseline\\\scriptsize{(\texttt{SimCLR})}} &  \makecell{\method\\(\small{$\lambda = 1$})} &  \makecell{\method\\(\small{$\lambda = 10$})} &  \makecell{\method\\(\small{$\lambda = 50$})} \\
\midrule
       rotation &             0.620 &                0.632 &                \textbf{0.697} &                 0.689\\
    contrast &             \textbf{0.816} &                0.815 &                0.795 &                 0.775 \\
      saturation &             \textbf{0.810} &                0.805 &                0.782 &                 0.763 \\
      blur &             \textbf{0.808} &                0.804 &                0.783 &                 0.762 \\
      
      none &    \textbf{0.828}    & \textbf{0.828}	  & 0.800    & 0.778     \\
      
\bottomrule
\end{tabular}
\end{table}



\begin{figure}
\captionsetup[subfigure]{justification=centering}
     \centering
     \begin{subfigure}[b]{0.22\textwidth}
         \centering
         \includegraphics[width=\textwidth]{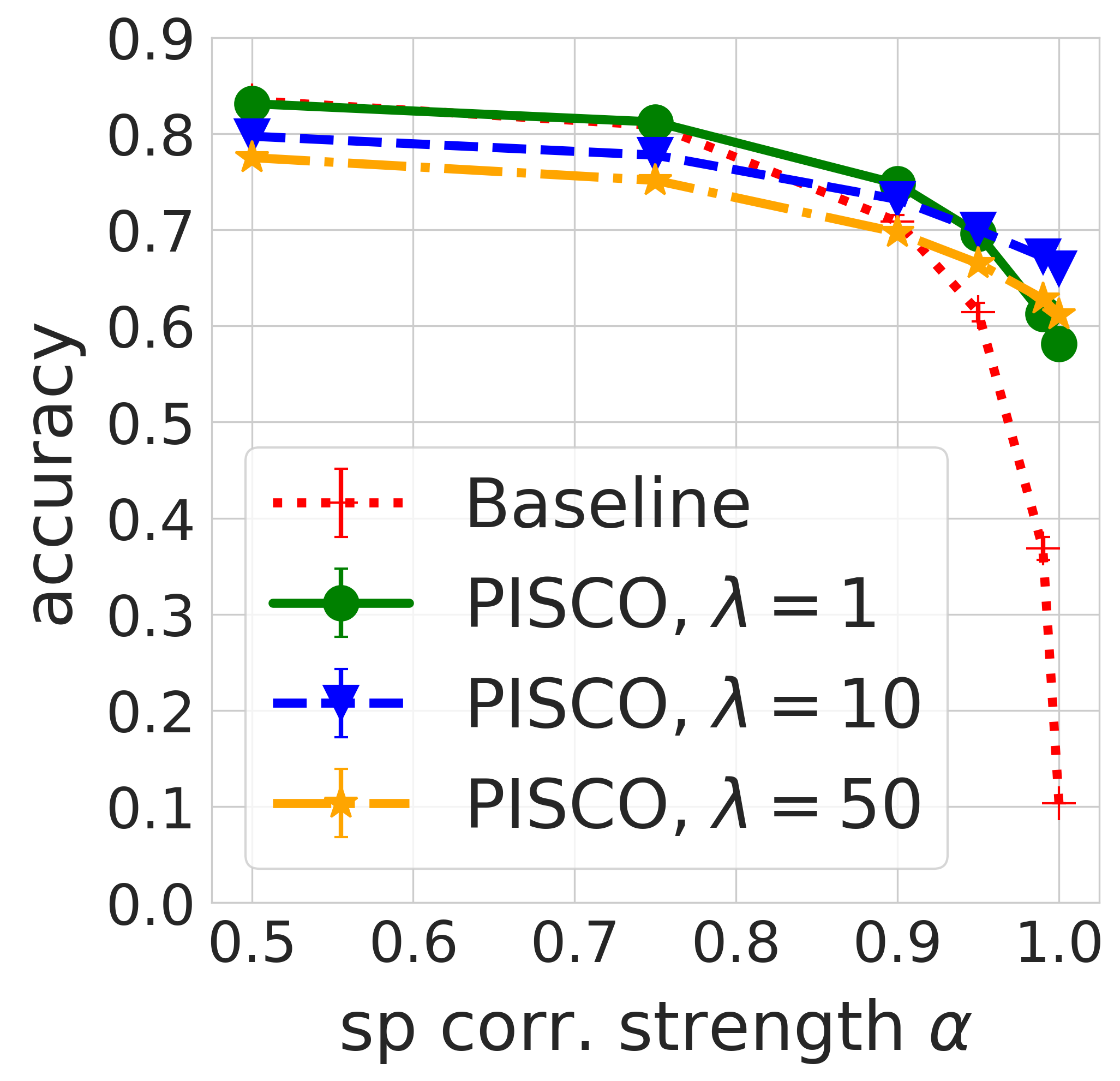}
         \caption{Rotation - \texttt{Supervised}}
         \label{fig:rotat-resnet-supp2}
     \end{subfigure}
     \hfill
     \begin{subfigure}[b]{0.22\textwidth}
         \centering
         \includegraphics[width=\textwidth]{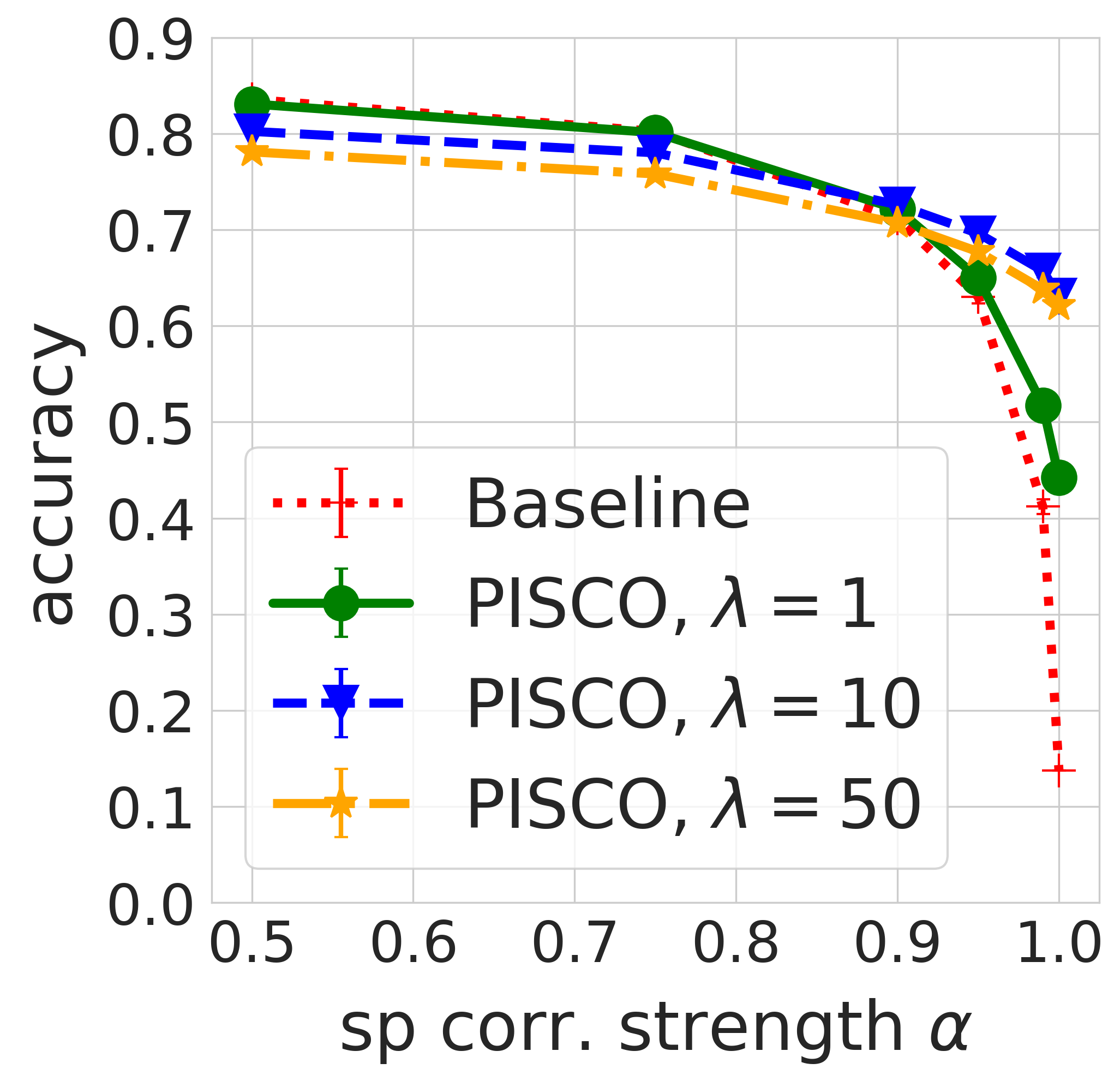}
         \caption{Contrast - \texttt{Supervised}}
         \label{fig:contr-resnet-supp2}
     \end{subfigure}
     \hfill
     \begin{subfigure}[b]{0.22\textwidth}
         \centering
         \includegraphics[width=\textwidth]{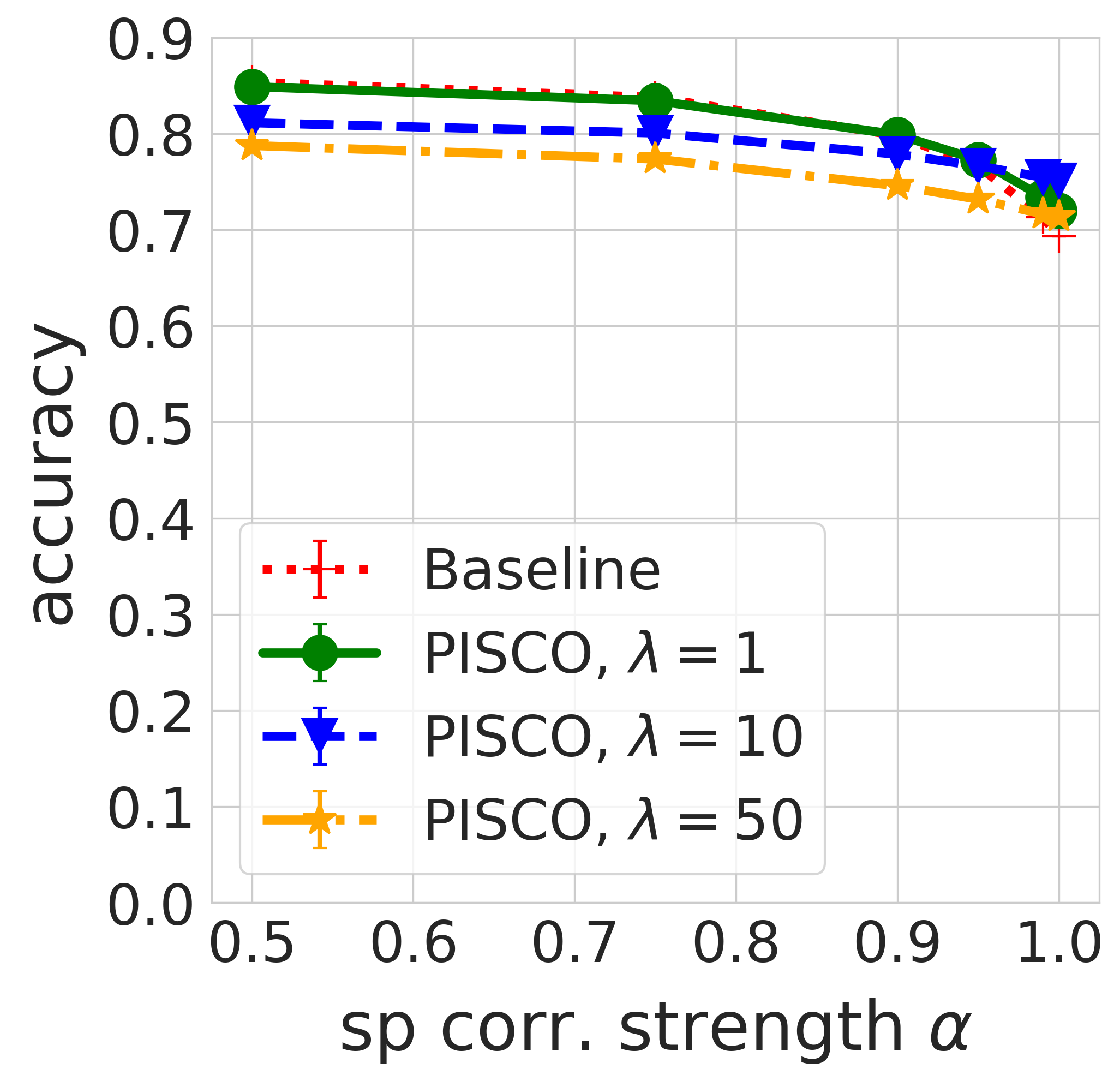}
         \caption{Blur - \texttt{Supervised}}
         \label{fig:blur-resnet-supp2}
     \end{subfigure}
     \hfill
     \begin{subfigure}[b]{0.22\textwidth}
         \centering
         \includegraphics[width=\textwidth]{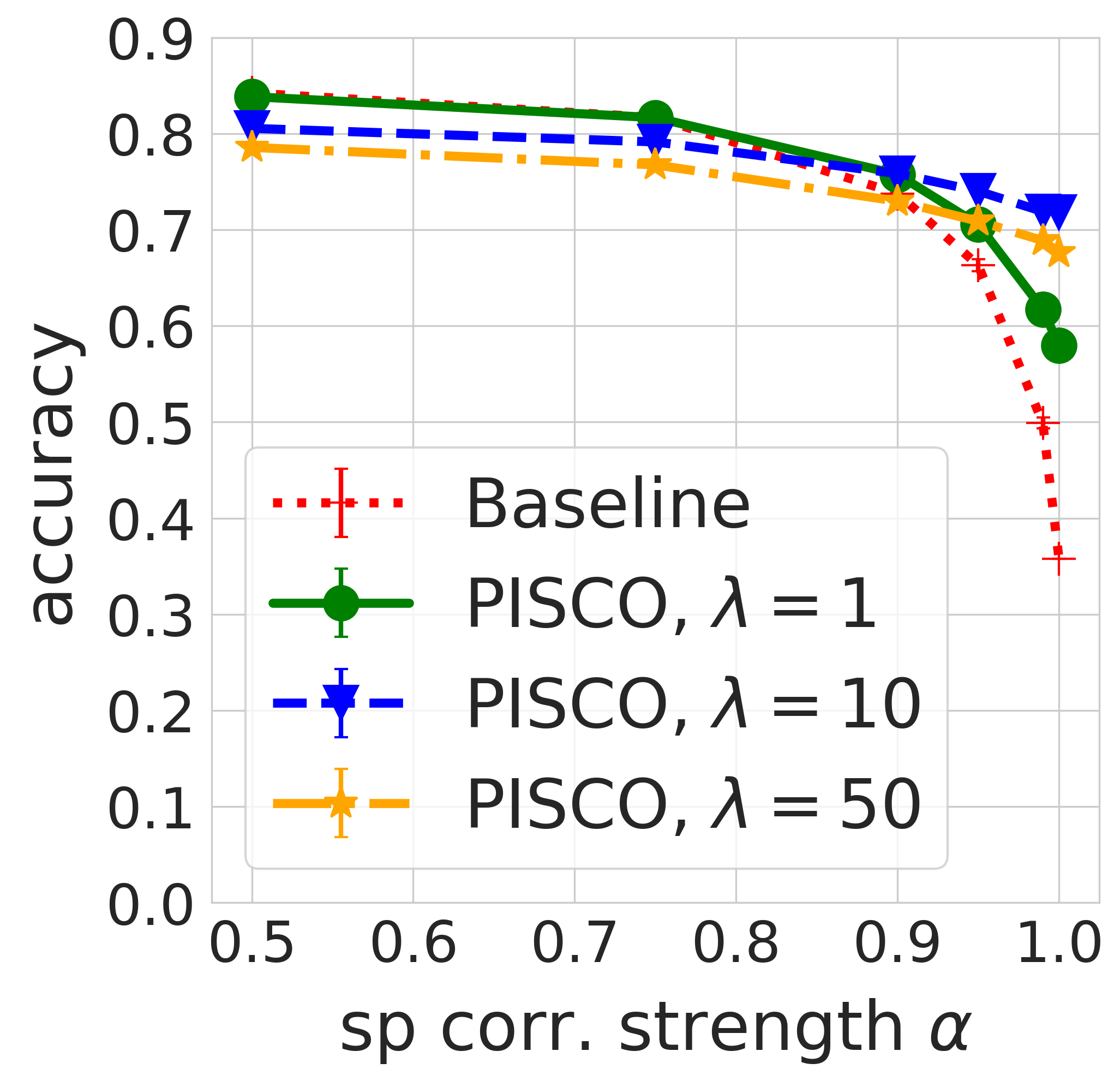}
         \caption{Satur. - \texttt{Supervised}}
         \label{fig:sat-resnet-supp2}
     \end{subfigure}
         \centering
         \caption{$\eta = 0.93$, OOD performance of \texttt{Supervised} representations on CIFAR-10 where the label is spuriously correlated with the corresponding transformation. \method\ significantly improves OOD performance, especially in the case of rotation. Both $\lambda=1$ and $\lambda=10$ preserve in-distribution accuracy, while larger $\lambda=50$ may degrade it as per \eqref{eq:loss}.
         }
         \label{fig:resnet_results_093}
\end{figure}

\begin{figure}
\captionsetup[subfigure]{justification=centering}
     \centering
     \begin{subfigure}[b]{0.22\textwidth}
         \centering
         \includegraphics[width=\textwidth]{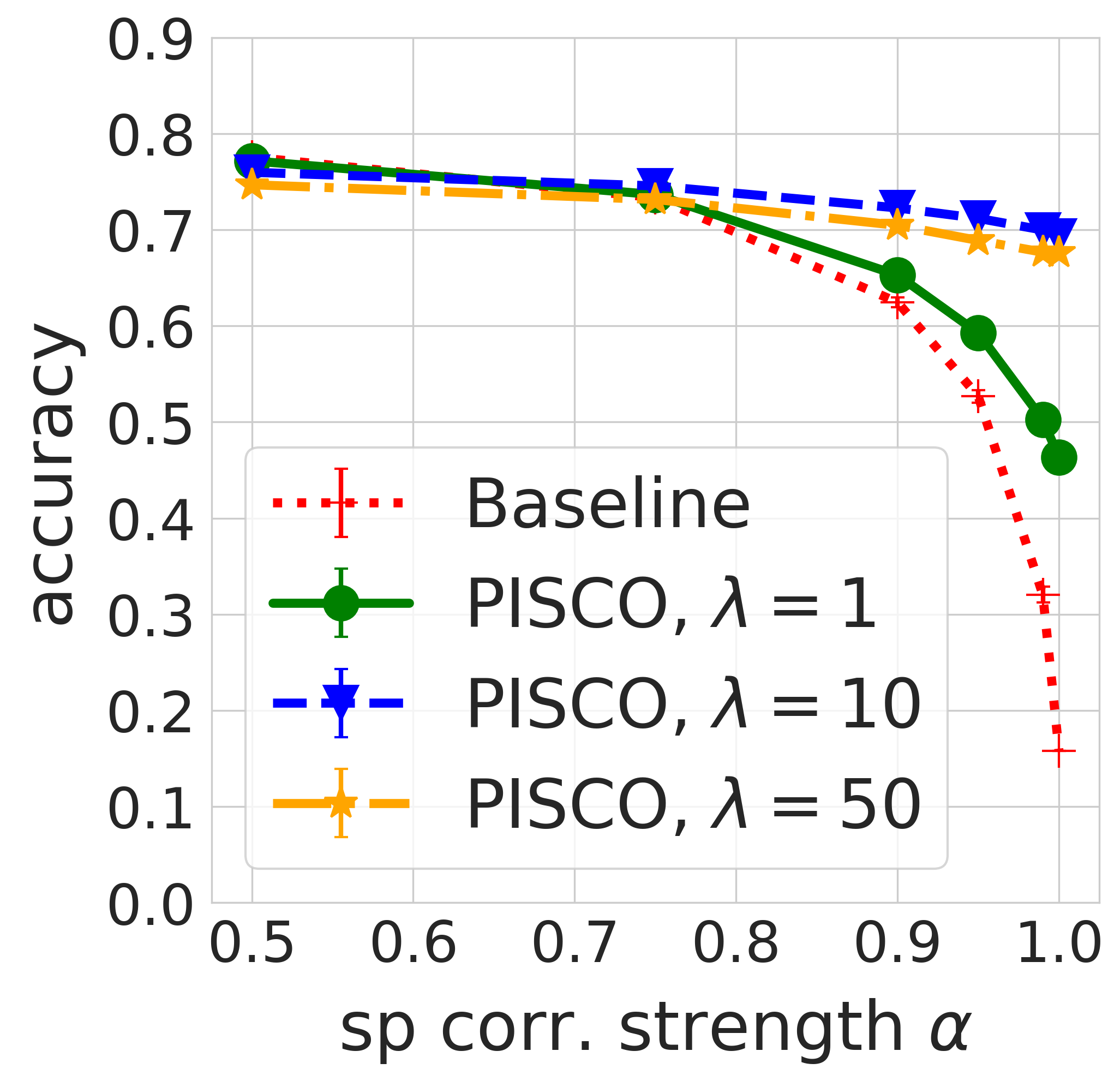}
         \caption{Rotation - \texttt{SimCLR}}
         \label{fig:rotat-simclr-supp2}
     \end{subfigure}
     \hfill
     \begin{subfigure}[b]{0.22\textwidth}
         \centering
         \includegraphics[width=\textwidth]{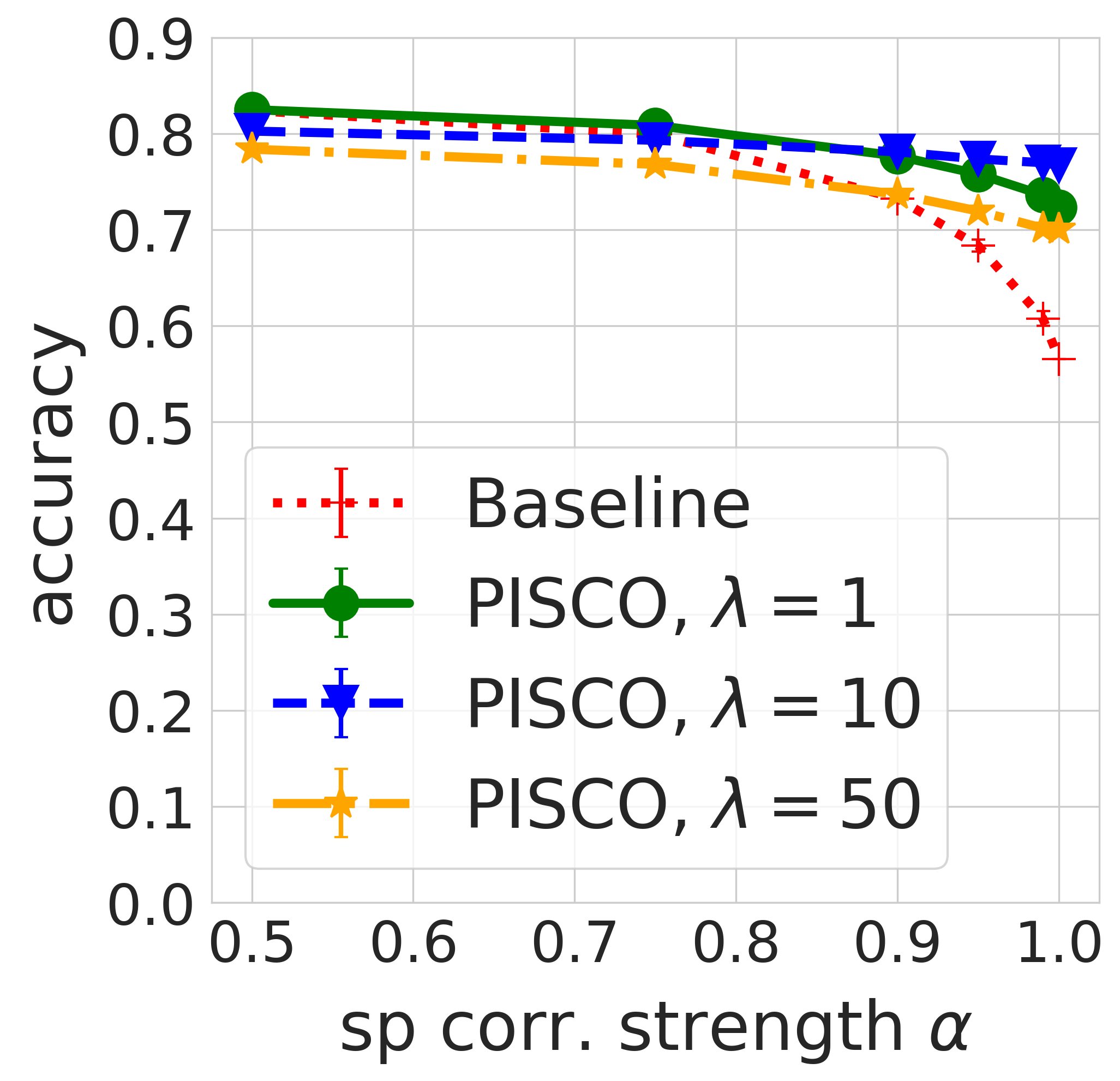}
         \caption{Contrast - \texttt{SimCLR}}
         \label{fig:contr-simclr-supp2}
     \end{subfigure}
     \hfill
     \begin{subfigure}[b]{0.22\textwidth}
         \centering
         \includegraphics[width=\textwidth]{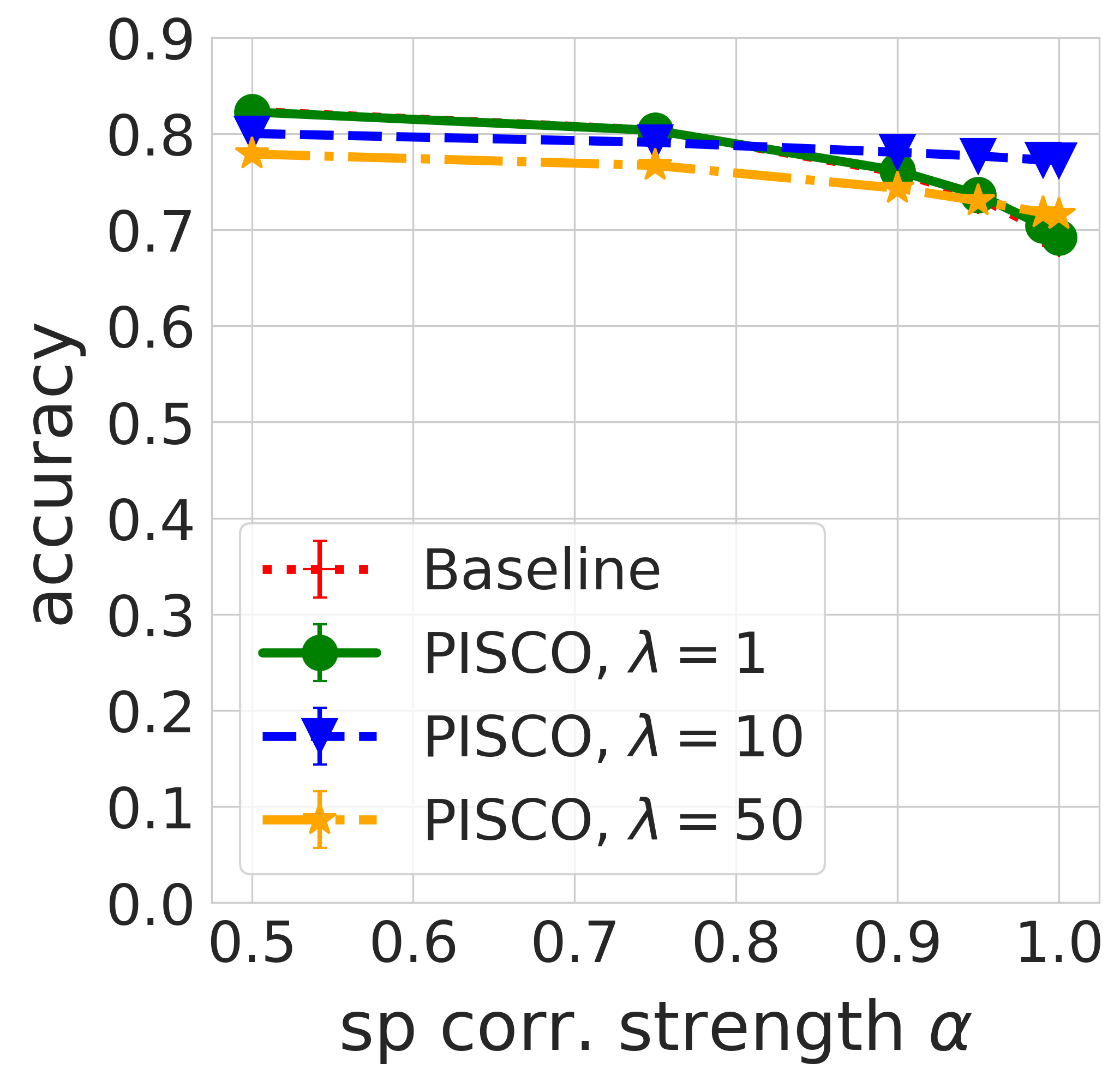}
         \caption{Blur - \texttt{SimCLR}}
         \label{fig:blur-simclr-supp2}
     \end{subfigure}
     \hfill
     \begin{subfigure}[b]{0.22\textwidth}
         \centering
         \includegraphics[width=\textwidth]{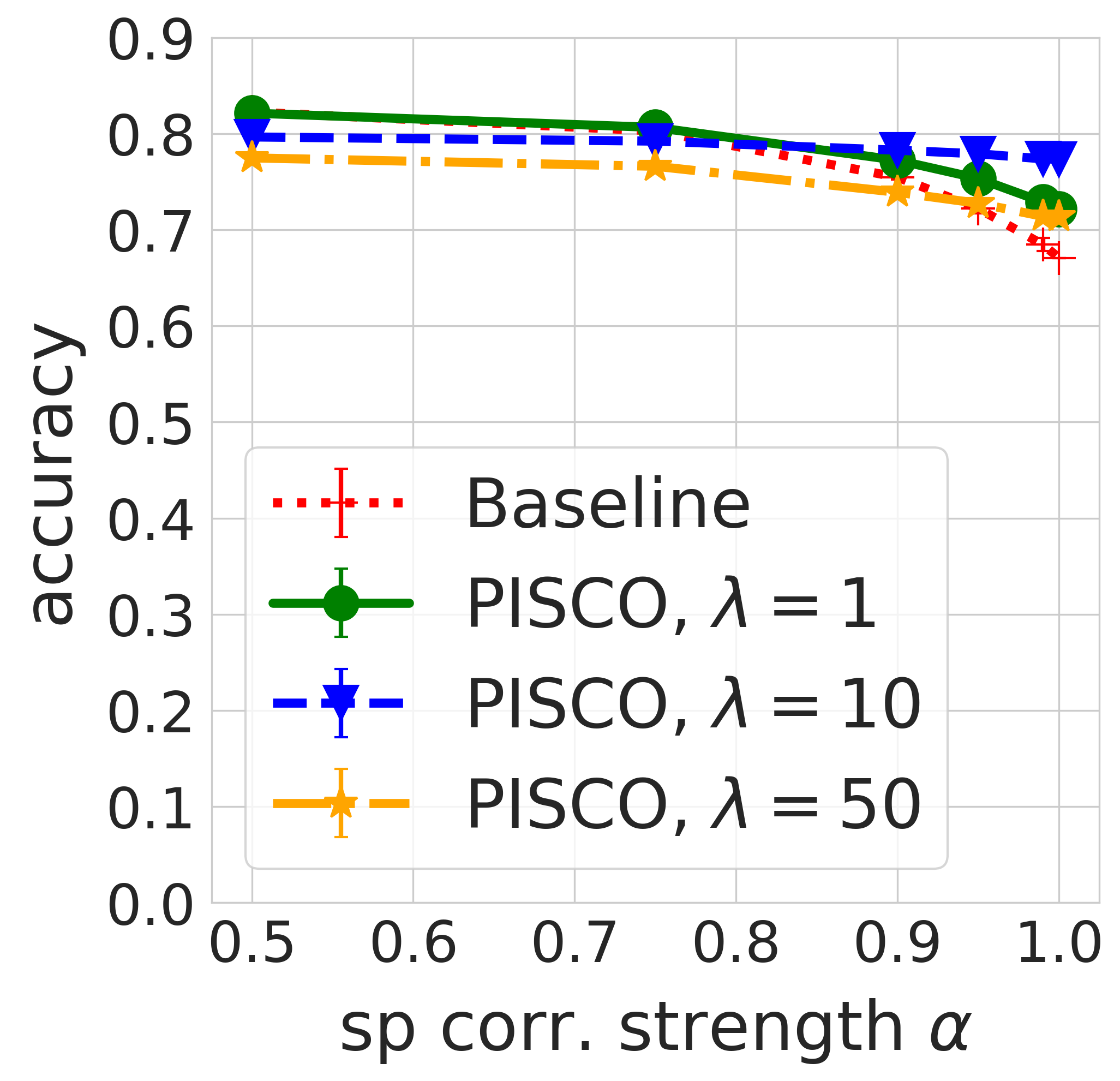}
         \caption{Saturation - \texttt{SimCLR}}
         \label{fig:sat-simclr-supp2}
     \end{subfigure}
         \centering
         \caption{$\eta = 0.93$, OOD performance of \texttt{SimCLR} representations on CIFAR-10 where the label is spuriously correlated with the corresponding transformation. Results are analogous to Figure \ref{fig:resnet_results_093}. The \texttt{SimCLR} baseline representations are less sensitive to contrast and saturation but remain sensitive to rotation.
         }
         \label{fig:simclr_results_093}
        
\end{figure}

\begin{table}
 \caption{$\eta = 0.93$, Performance of \texttt{Supervised} representations on CIFAR-10 test set in-distribution, i.e., no transformation (referred to as ``none''; last row), and OOD when modified with the corresponding transformation. \method\ with $\lambda=1$ provides significant improvements for rotation, contrast, and saturation while preserving in-distribution accuracy.
 }
 \label{tb:ood_gen_resnet_093}
 \centering

\begin{tabular}{lccccc}
\toprule
Style &  \makecell{Baseline\\\scriptsize{(\texttt{Supervised})}} &  \makecell{\method\\(\small{$\lambda = 1$})} &  \makecell{\method\\(\small{$\lambda = 10$})} &  \makecell{\method\\(\small{$\lambda = 50$})} \\
\midrule
      rotation &  0.678 &                \textbf{0.741} &                0.726 &                0.700                  \\
    contrast & 0.625 &                0.678 &                \textbf{0.744} &                0.723  \\
      saturation & 0.699 &                \textbf{0.759} &                0.742 &                0.718   \\
      blur & \textbf{0.817} &                \textbf{0.817} &                0.788 &                0.761    \\
     
     none &  \textbf{0.873} & 0.871 & 0.827 &  0.805    \\
\bottomrule
\end{tabular}
\end{table}

\begin{table}
 \caption{
 $\eta = 0.93$, Performance of \texttt{SimCLR} representations on CIFAR-10 test set in-distribution, i.e., no transformation (referred to as ``none''; last row), and OOD when modified with the corresponding transformation. \texttt{SimCLR} features are robust to these transformations and perform similarly to \method\ with $\lambda=1$.
 }
 \label{tb:ood_gen_simclr_093}
 \centering

\begin{tabular}{lccccc}
\toprule
 Style &  \makecell{Baseline\\\scriptsize{(SimCLR)}} &  \makecell{\method\\(\small{$\lambda = 1$})} &  \makecell{\method\\(\small{$\lambda = 10$})} &  \makecell{\method\\(\small{$\lambda = 50$})} \\
\midrule
       rotation &             0.620 &               0.632 &                 \textbf{0.695} &                0.692\\
    contrast &             0.816 &               \textbf{0.817} &                 0.797 &                0.786 \\
      saturation &             \textbf{0.810} &               0.809 &                 0.786 &                0.765 \\
      blur &             \textbf{0.808} &                0.804 &                0.783 &                 0.762 \\
      
      none &    \textbf{0.828}    & 0.826	  & 0.804    & 0.782     \\
      
\bottomrule
\end{tabular}
\end{table}



\begin{figure}
\captionsetup[subfigure]{justification=centering}
     \centering
     \begin{subfigure}[b]{0.22\textwidth}
         \centering
         \includegraphics[width=\textwidth]{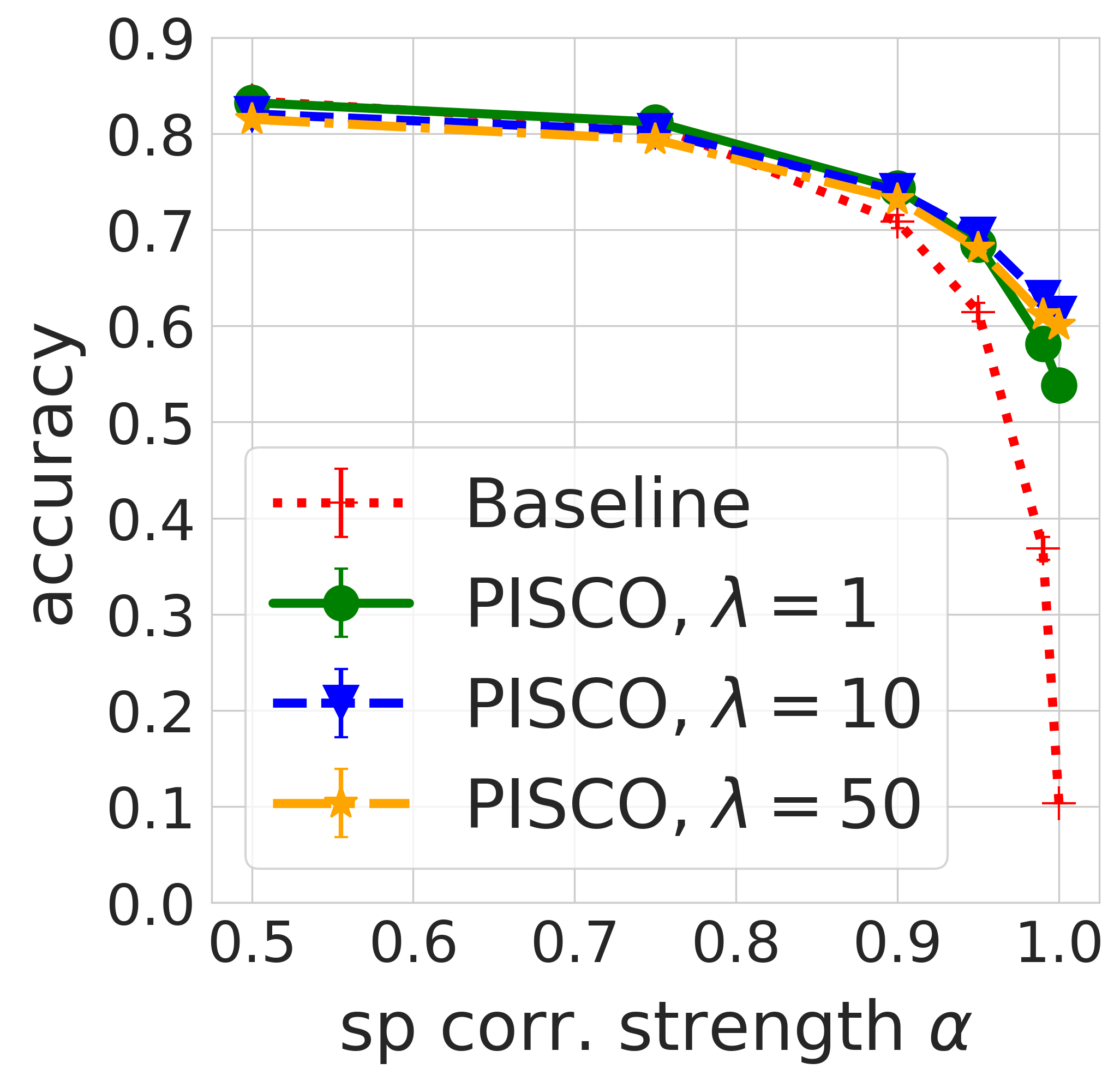}
         \caption{Rotation - \texttt{Supervised}}
         \label{fig:rotat-resnet-supp3}
     \end{subfigure}
     \hfill
     \begin{subfigure}[b]{0.22\textwidth}
         \centering
         \includegraphics[width=\textwidth]{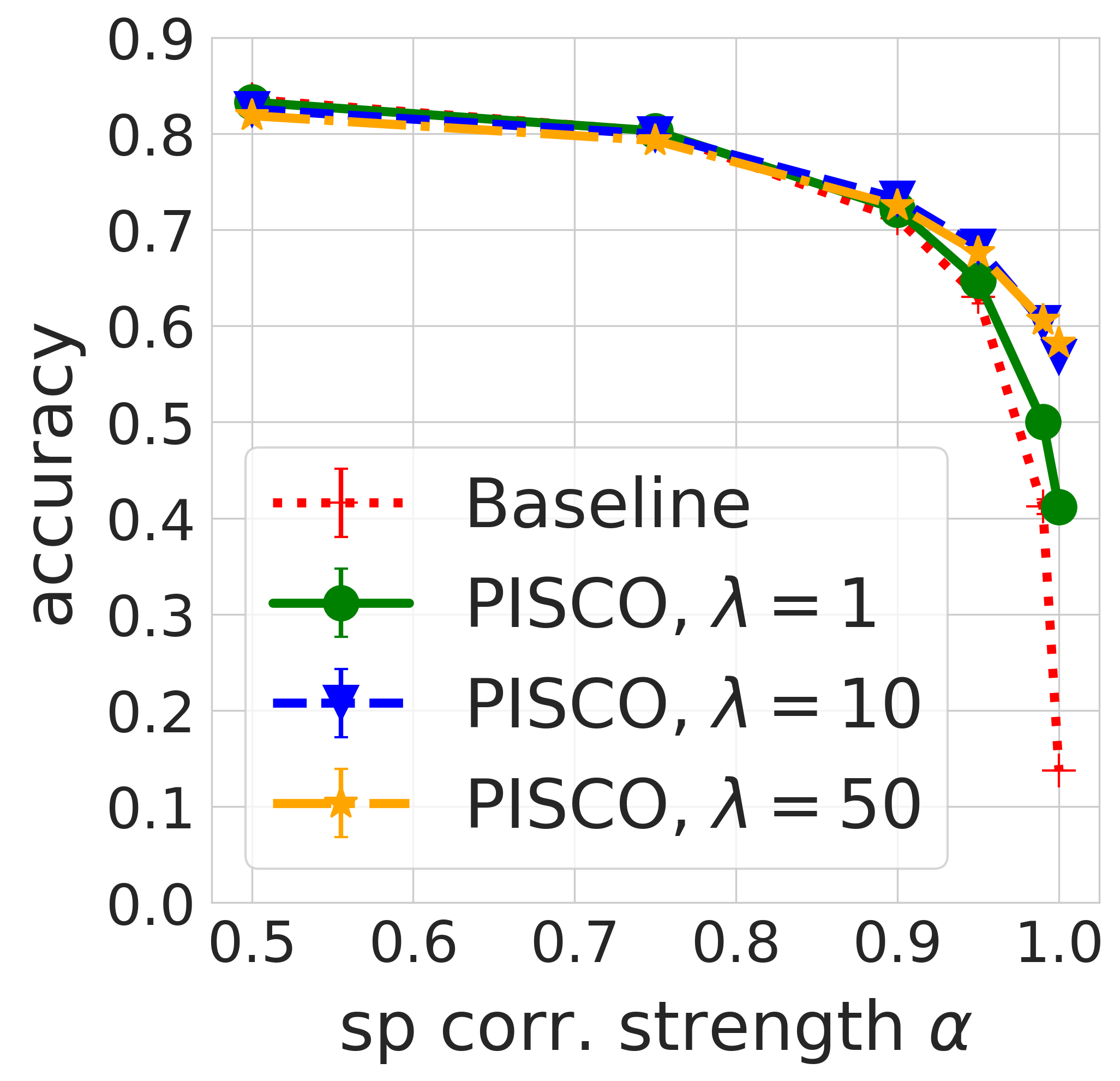}
         \caption{Contrast - \texttt{Supervised}}
         \label{fig:contr-resnet-supp3}
     \end{subfigure}
     \hfill
     \begin{subfigure}[b]{0.22\textwidth}
         \centering
         \includegraphics[width=\textwidth]{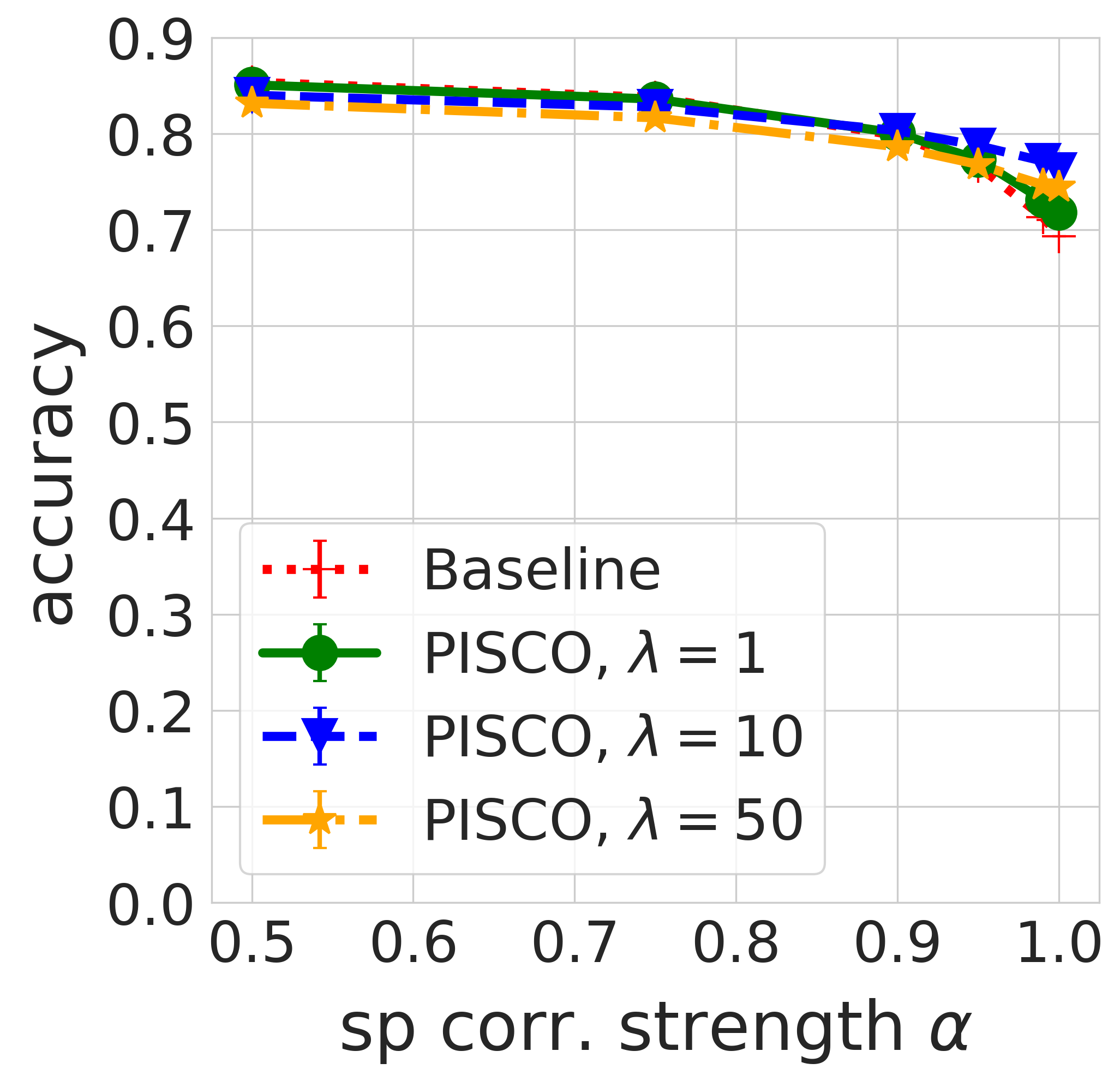}
         \caption{Blur - \texttt{Supervised}}
         \label{fig:blur-resnet-supp3}
     \end{subfigure}
     \hfill
     \begin{subfigure}[b]{0.22\textwidth}
         \centering
         \includegraphics[width=\textwidth]{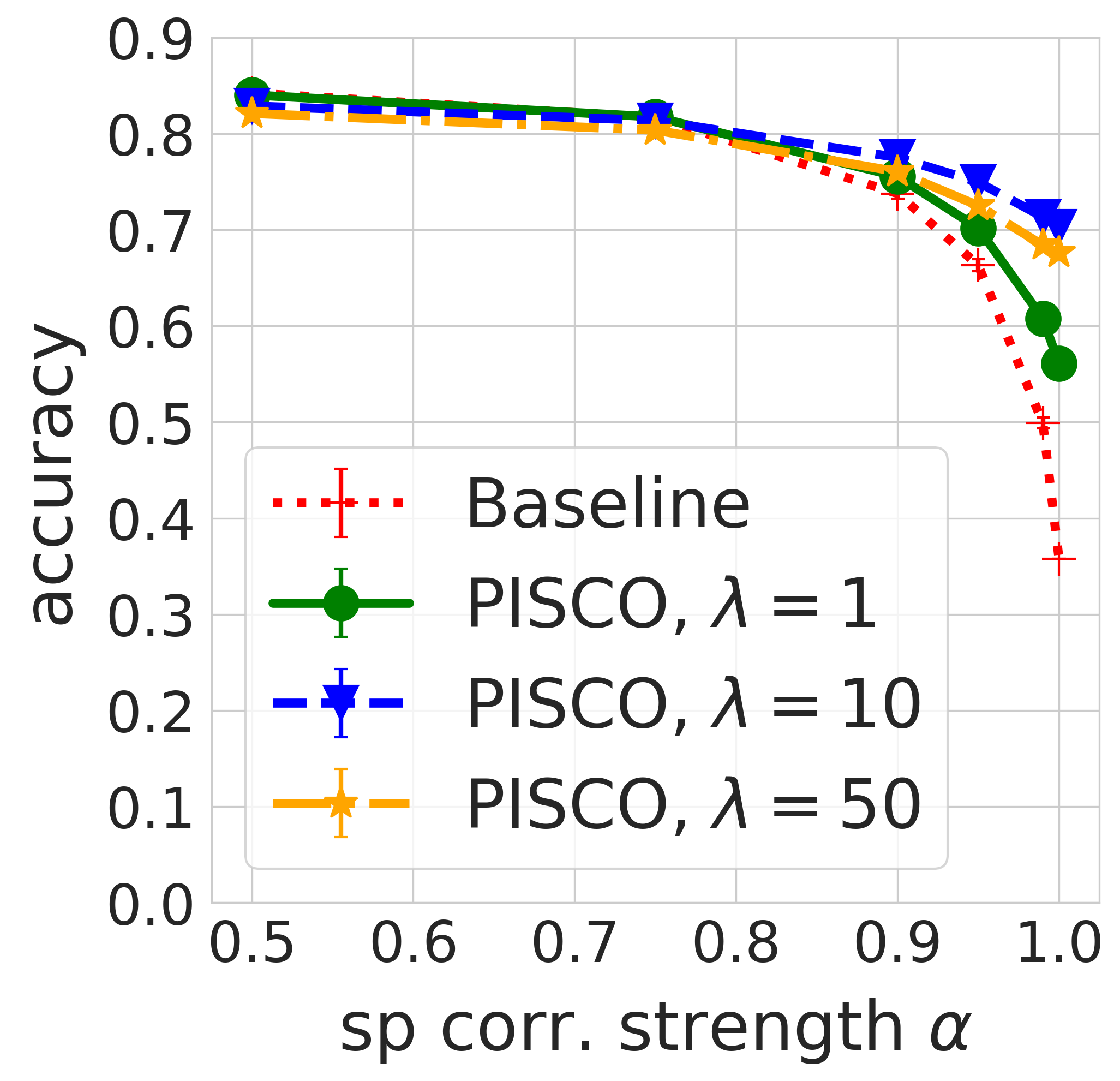}
         \caption{Saturation - \texttt{Supervised}}
         \label{fig:sat-resnet-supp3}
     \end{subfigure}
         \centering
         \caption{$\eta = 0.98$, OOD performance of \texttt{Supervised} representations on CIFAR-10 where the label is spuriously correlated with the corresponding transformation. \method\ significantly improves OOD performance, especially in the case of rotation. Both $\lambda=1$ and $\lambda=10$ preserve in-distribution accuracy, while larger $\lambda=50$ may degrade it as per \eqref{eq:loss}.
         }
         \label{fig:resnet_results_098}
\end{figure}

\begin{figure}
\captionsetup[subfigure]{justification=centering}
     \centering
     \begin{subfigure}[b]{0.22\textwidth}
         \centering
         \includegraphics[width=\textwidth]{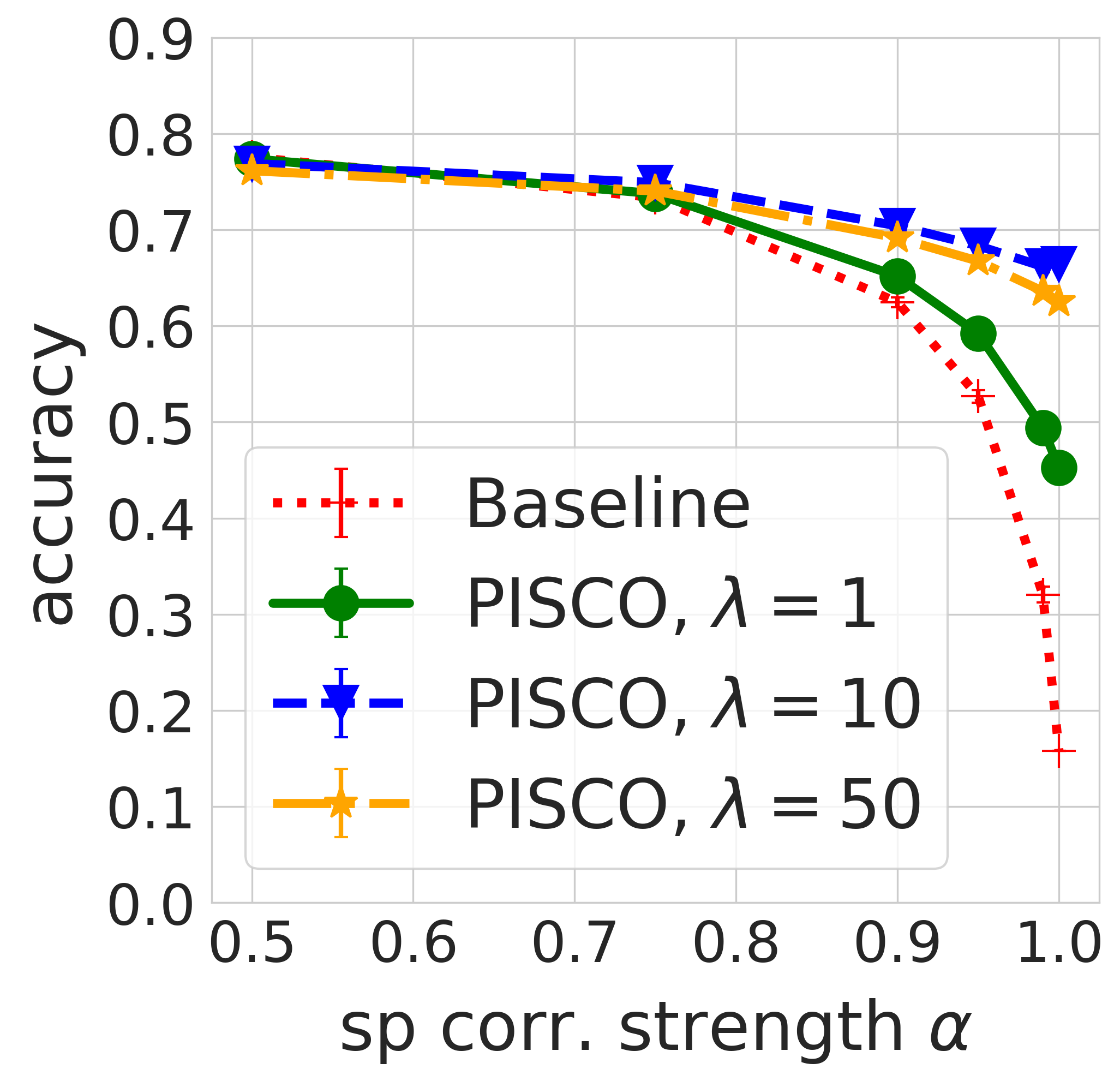}
         \caption{Rotation - \texttt{SimCLR}}
         \label{fig:rotat-simclr-supp3}
     \end{subfigure}
     \hfill
     \begin{subfigure}[b]{0.22\textwidth}
         \centering
         \includegraphics[width=\textwidth]{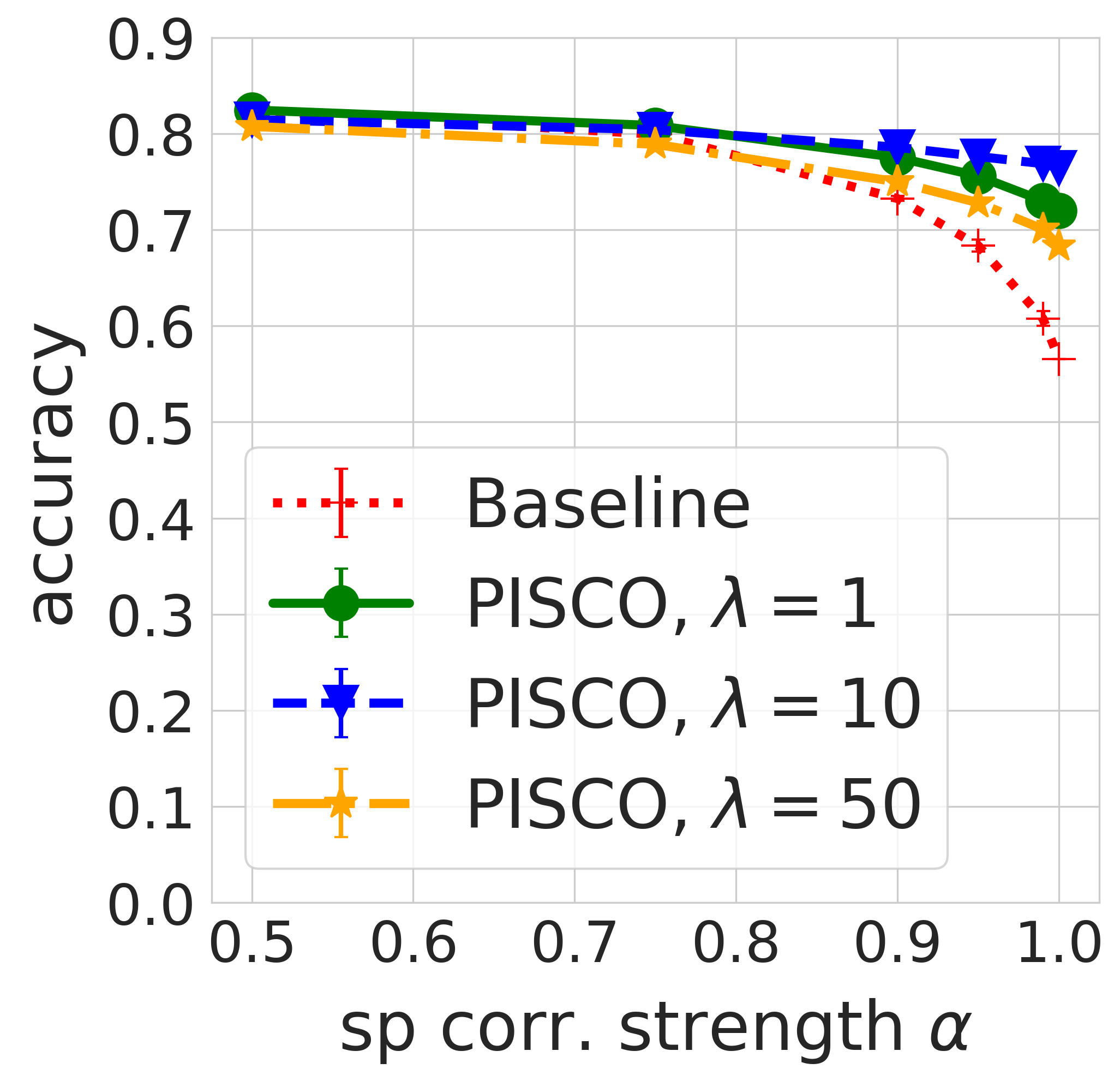}
         \caption{Contrast - \texttt{SimCLR}}
         \label{fig:contr-simclr-supp3}
     \end{subfigure}
     \hfill
     \begin{subfigure}[b]{0.22\textwidth}
         \centering
         \includegraphics[width=\textwidth]{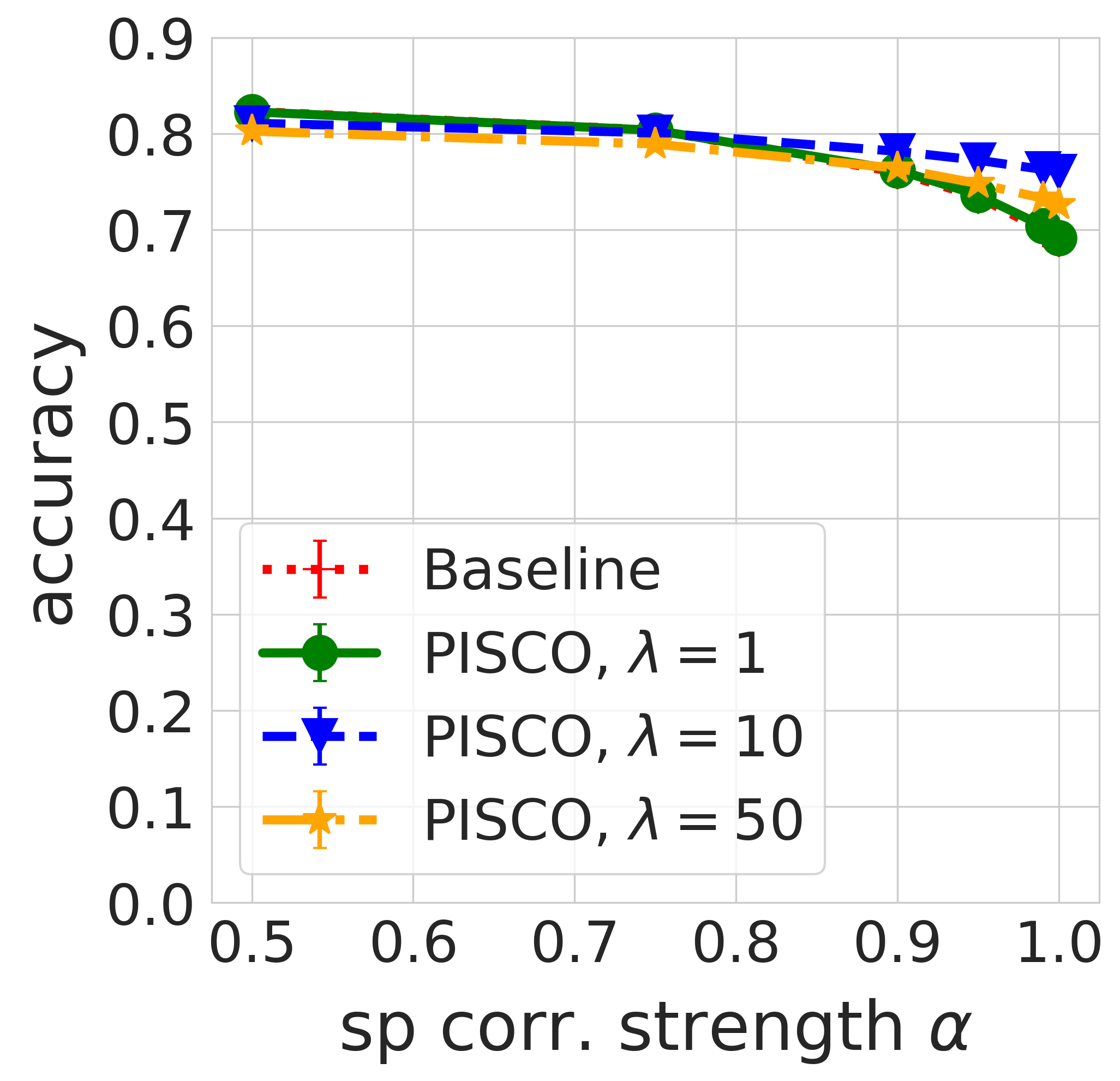}
         \caption{Blur - \texttt{SimCLR}}
         \label{fig:blur-simclr-supp3}
     \end{subfigure}
     \hfill
     \begin{subfigure}[b]{0.22\textwidth}
         \centering
         \includegraphics[width=\textwidth]{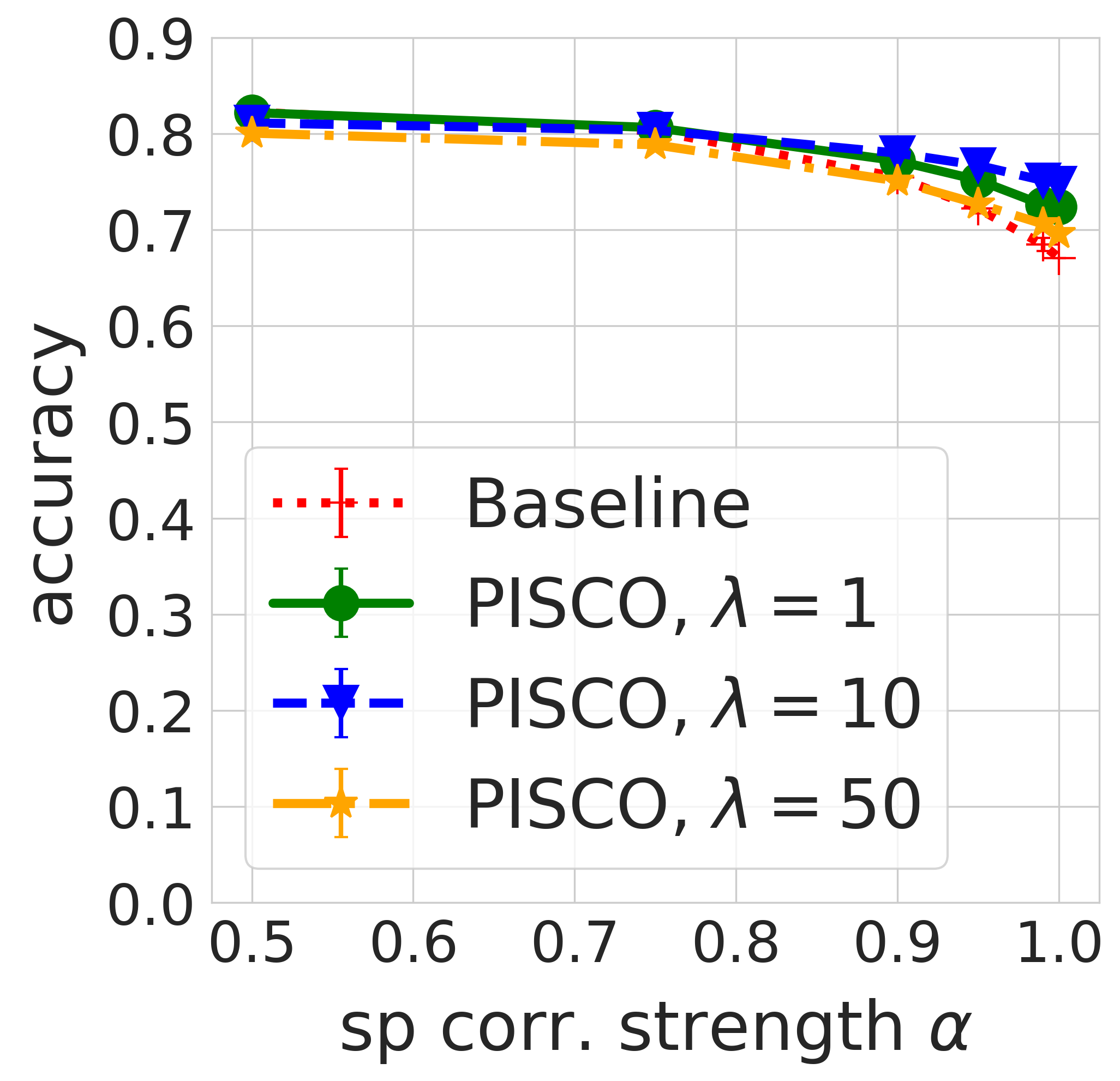}
         \caption{Saturation - \texttt{SimCLR}}
         \label{fig:sat-simclr-supp3}
     \end{subfigure}
         \centering
         \caption{$\eta = 0.98$, OOD performance of \texttt{SimCLR} representations on CIFAR-10 where the label is spuriously correlated with the corresponding transformation. Results are analogous to Figure \ref{fig:resnet_results_098}. The \texttt{SimCLR} baseline representations are less sensitive to contrast and saturation but remain sensitive to rotation.
         }
         \label{fig:simclr_results_098}
        
\end{figure}

\begin{table}
 \caption{$\eta = 0.98$, Performance of \texttt{Supervised} representations on CIFAR-10 test set in-distribution, i.e., no transformation (referred to as ``none''; last row), and OOD when modified with the corresponding transformation. \method\ with $\lambda=1$ provides significant improvements for rotation, contrast, and saturation, while preserving in-distribution accuracy.
 }
 \label{tb:ood_gen_resnet_098}
 \centering

\begin{tabular}{lccccc}
\toprule
Style &  \makecell{Baseline\\\scriptsize{(\texttt{Supervised})}} &  \makecell{\method\\(\small{$\lambda = 1$})} &  \makecell{\method\\(\small{$\lambda = 10$})} &  \makecell{\method\\(\small{$\lambda = 50$})} \\
\midrule
      rotation &  0.678 &               \textbf{0.739} &                0.736 &                 0.726                 \\
    contrast & 0.625 &               0.680 &                \textbf{0.740} &                 0.729  \\
      saturation & 0.699 &               \textbf{0.757} &                0.740 &                 0.726   \\
      blur & \textbf{0.817} &               \textbf{0.817} &                0.807 &                 0.797    \\
     
     none &  \textbf{0.873} & 0.871 & 0.861 &  0.851    \\
\bottomrule
\end{tabular}
\end{table}

\begin{table}
 \caption{
 $\eta = 0.98$, Performance of \texttt{SimCLR} representations on CIFAR-10 test set in-distribution, i.e., no transformation (referred to as ``none''; last row), and OOD when modified with the corresponding transformation. \texttt{SimCLR} features are robust to these transformations and perform similarly to \method\ with $\lambda=1$.
 }
 \label{tb:ood_gen_simclr_098}
 \centering

\begin{tabular}{lccccc}
\toprule
 Style &  \makecell{Baseline\\\scriptsize{(\texttt{SimCLR})}} &  \makecell{\method\\(\small{$\lambda = 1$})} &  \makecell{\method\\(\small{$\lambda = 10$})} &  \makecell{\method\\(\small{$\lambda = 50$})} \\
\midrule
       rotation &             0.620 &               0.633 &                \textbf{0.681} &                 0.677\\
    contrast &             0.816 &               \textbf{0.817} &                0.808 &                 0.799 \\
      saturation &             \textbf{0.810} &               0.809 &                0.796 &                 0.778 \\
      blur &             \textbf{0.808} &               0.806 &                0.795 &                 0.793 \\
      
      none &    \textbf{0.828}    & 0.826	  & 0.815    & 0.806     \\
      
\bottomrule
\end{tabular}
\end{table}



\begin{figure}
\captionsetup[subfigure]{justification=centering}
     \centering
     \begin{subfigure}[b]{0.22\textwidth}
         \centering
         \includegraphics[width=\textwidth]{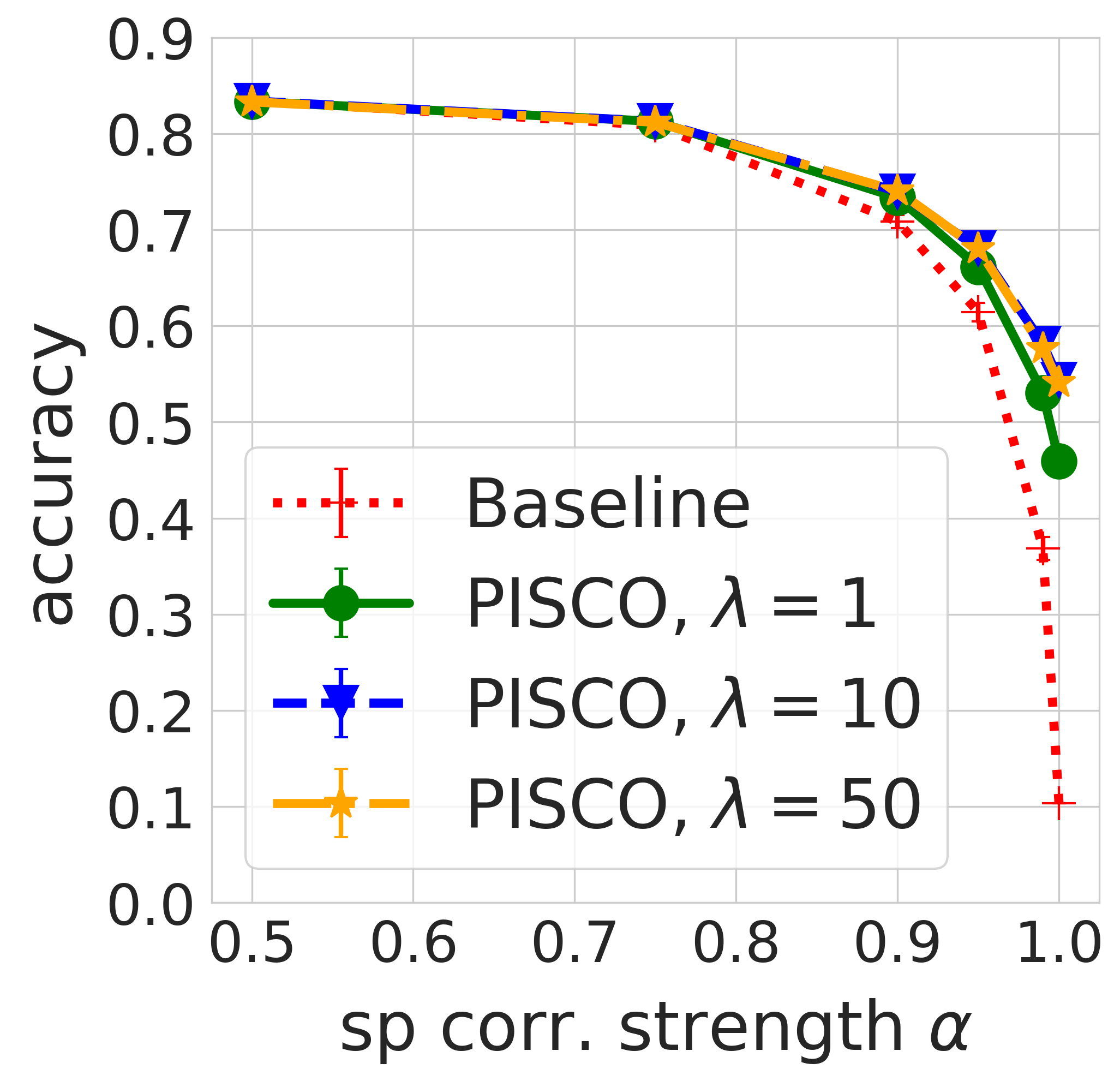}
         \caption{Rotation - \texttt{Supervised}}
         \label{fig:rotat-resnet-supp4}
     \end{subfigure}
     \hfill
     \begin{subfigure}[b]{0.22\textwidth}
         \centering
         \includegraphics[width=\textwidth]{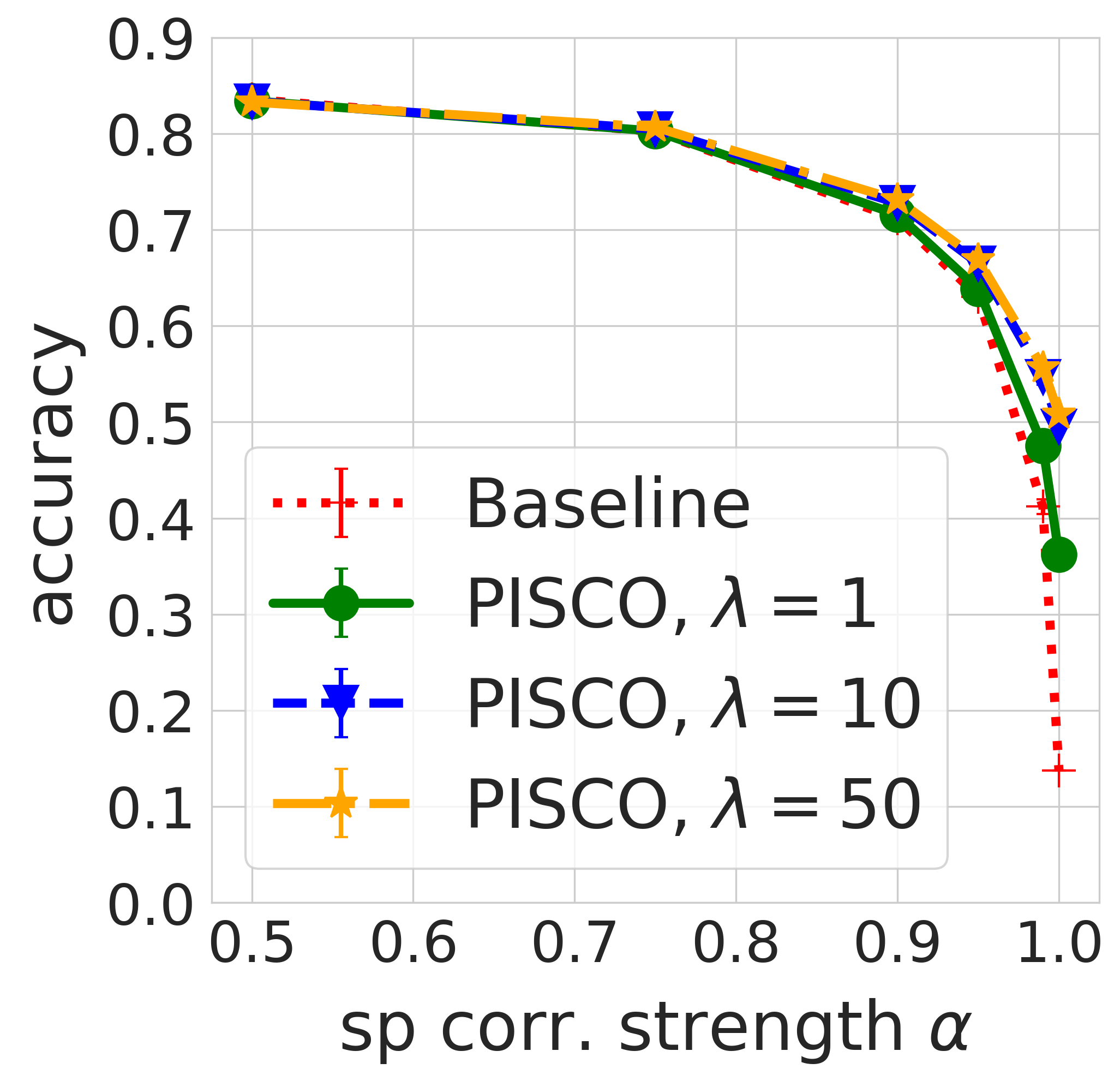}
         \caption{Contrast - \texttt{Supervised}}
         \label{fig:contr-resnet-supp4}
     \end{subfigure}
     \hfill
     \begin{subfigure}[b]{0.22\textwidth}
         \centering
         \includegraphics[width=\textwidth]{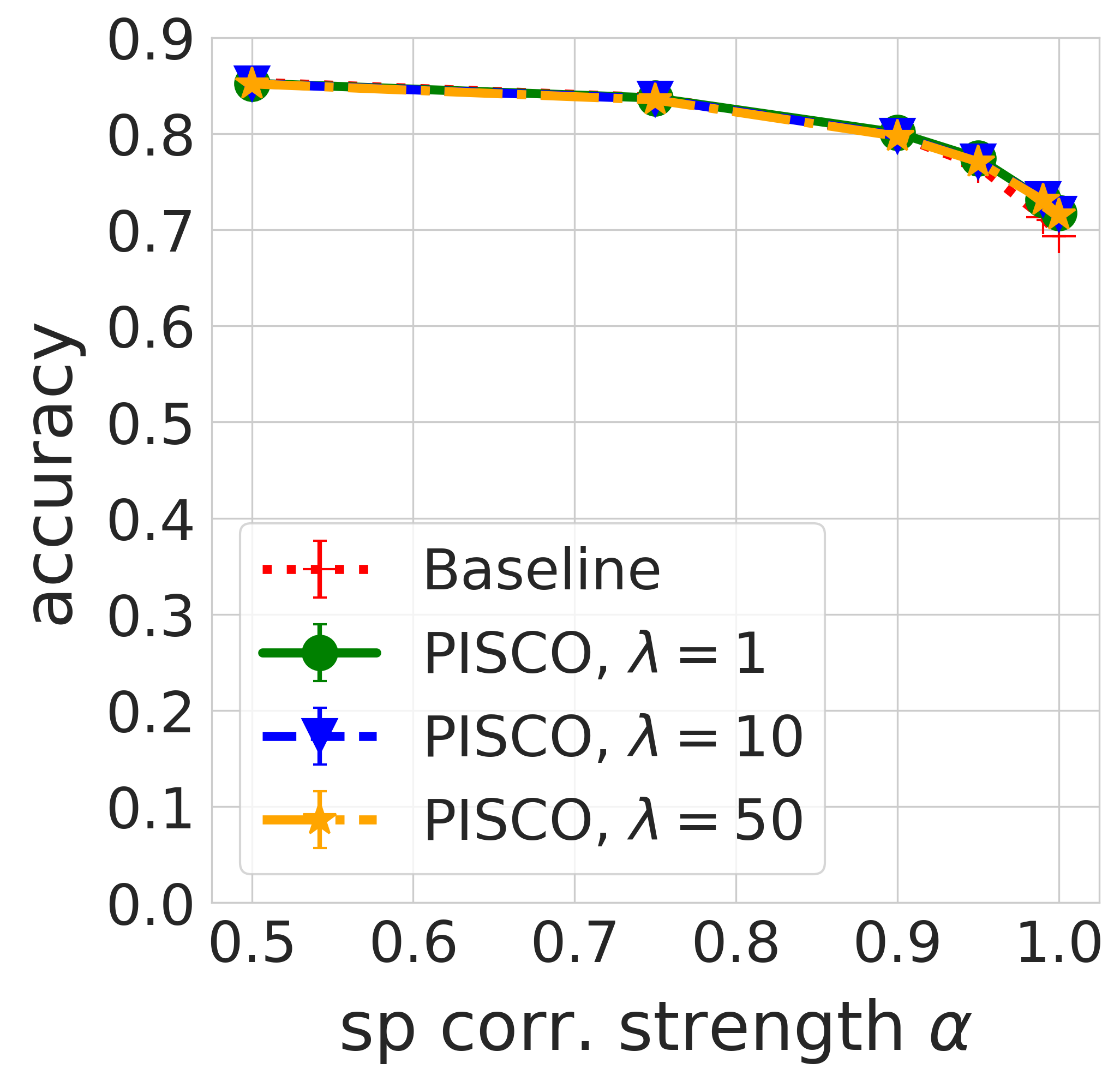}
         \caption{Blur - \texttt{Supervised}}
         \label{fig:blur-resnet-supp4}
     \end{subfigure}
     \hfill
     \begin{subfigure}[b]{0.22\textwidth}
         \centering
         \includegraphics[width=\textwidth]{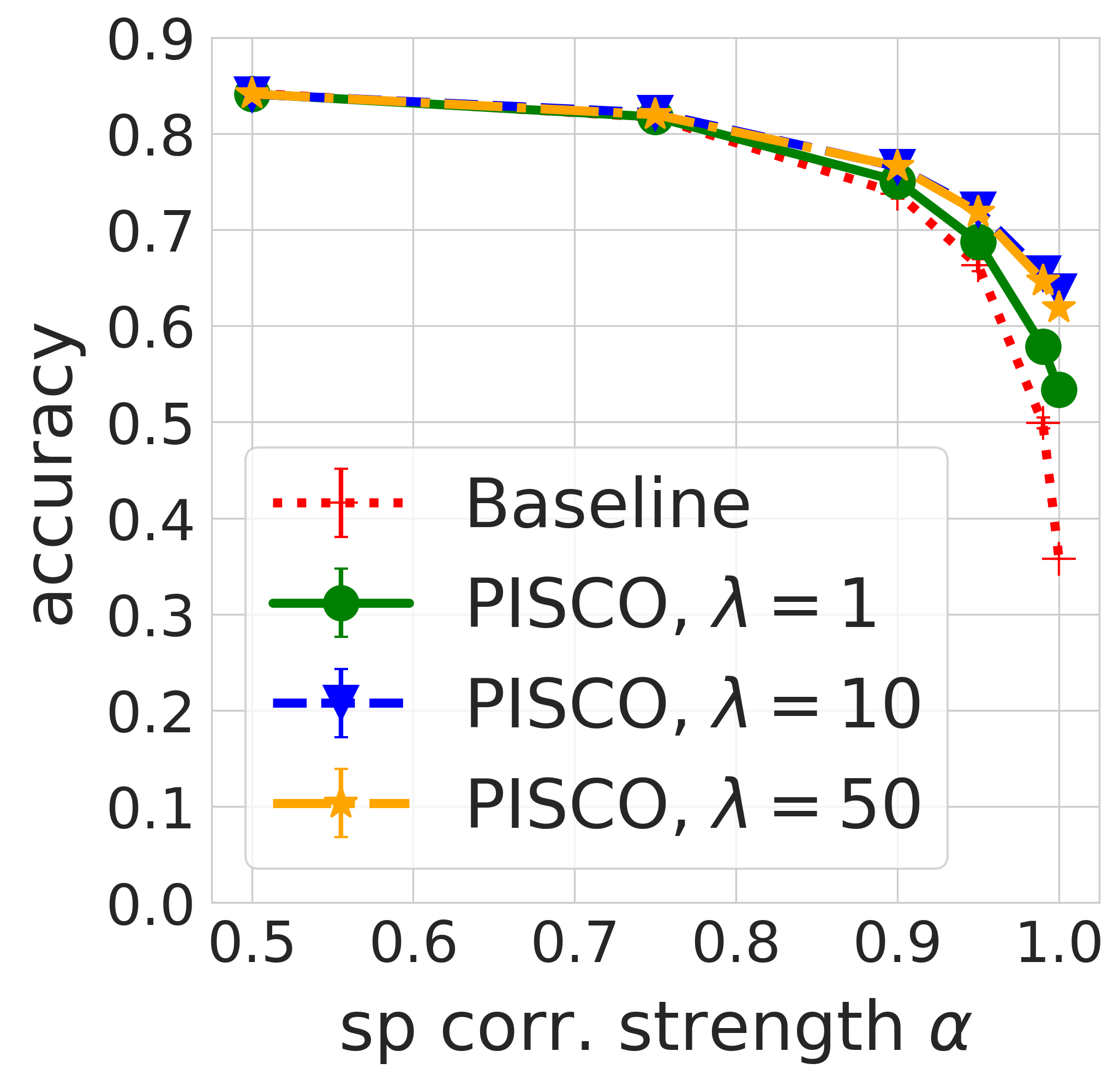}
         \caption{Saturation - \texttt{Supervised}}
         \label{fig:sat-resnet-supp4}
     \end{subfigure}
         \centering
         \caption{$\eta = 1.0$, OOD performance of \texttt{Supervised} representations on CIFAR-10 where the label is spuriously correlated with the corresponding transformation. \method\ significantly improves OOD performance, especially in the case of rotation. Both $\lambda=1$ and $\lambda=10$ preserve in-distribution accuracy, while larger $\lambda=50$ may degrade it as per \eqref{eq:loss}.
         }
         \label{fig:resnet_results_1}
\end{figure}

\begin{figure}
\captionsetup[subfigure]{justification=centering}
     \centering
     \begin{subfigure}[b]{0.22\textwidth}
         \centering
         \includegraphics[width=\textwidth]{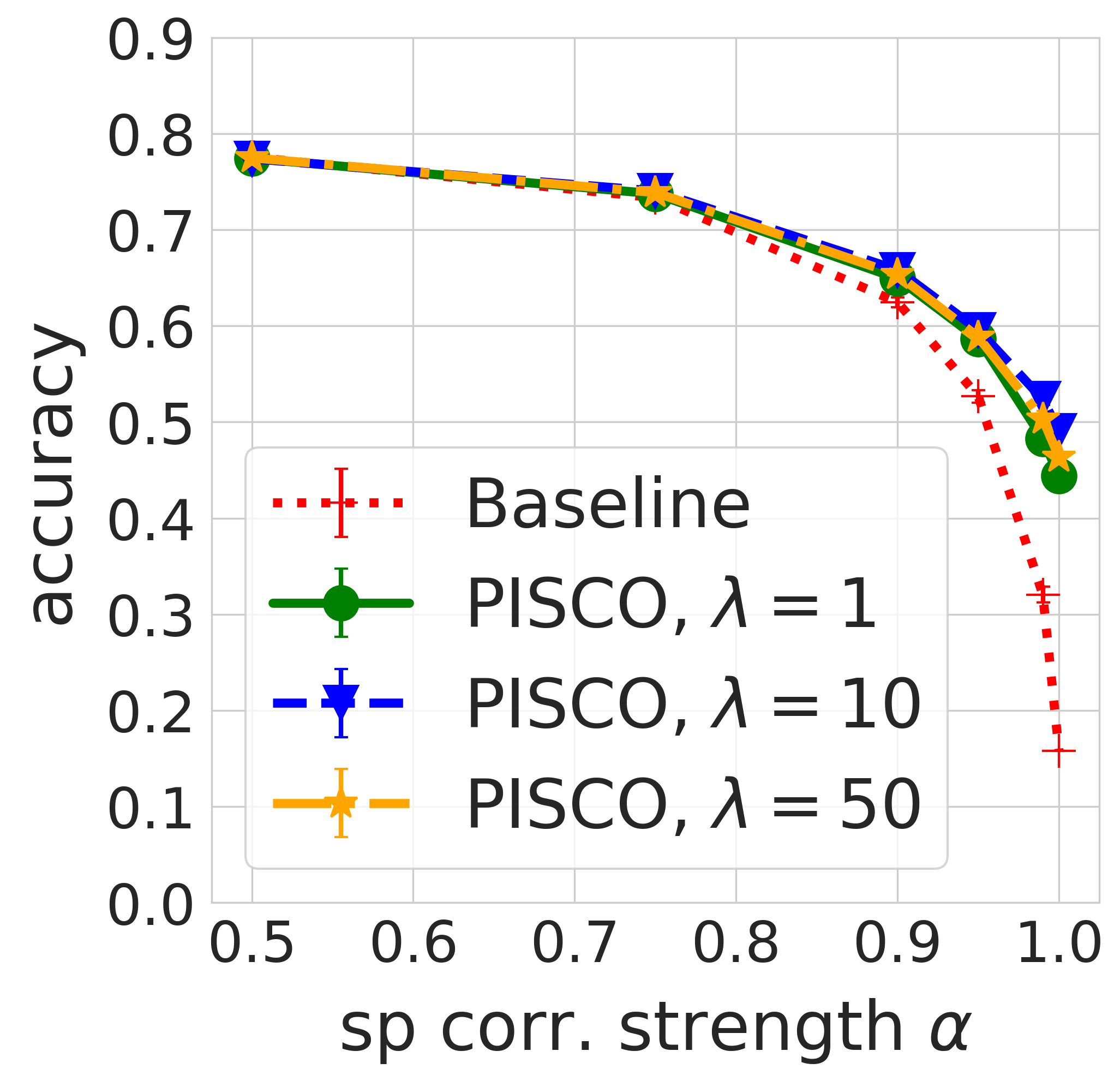}
         \caption{Rotation - \texttt{SimCLR}}
         \label{fig:rotat-simclr-supp4}
     \end{subfigure}
     \hfill
     \begin{subfigure}[b]{0.22\textwidth}
         \centering
         \includegraphics[width=\textwidth]{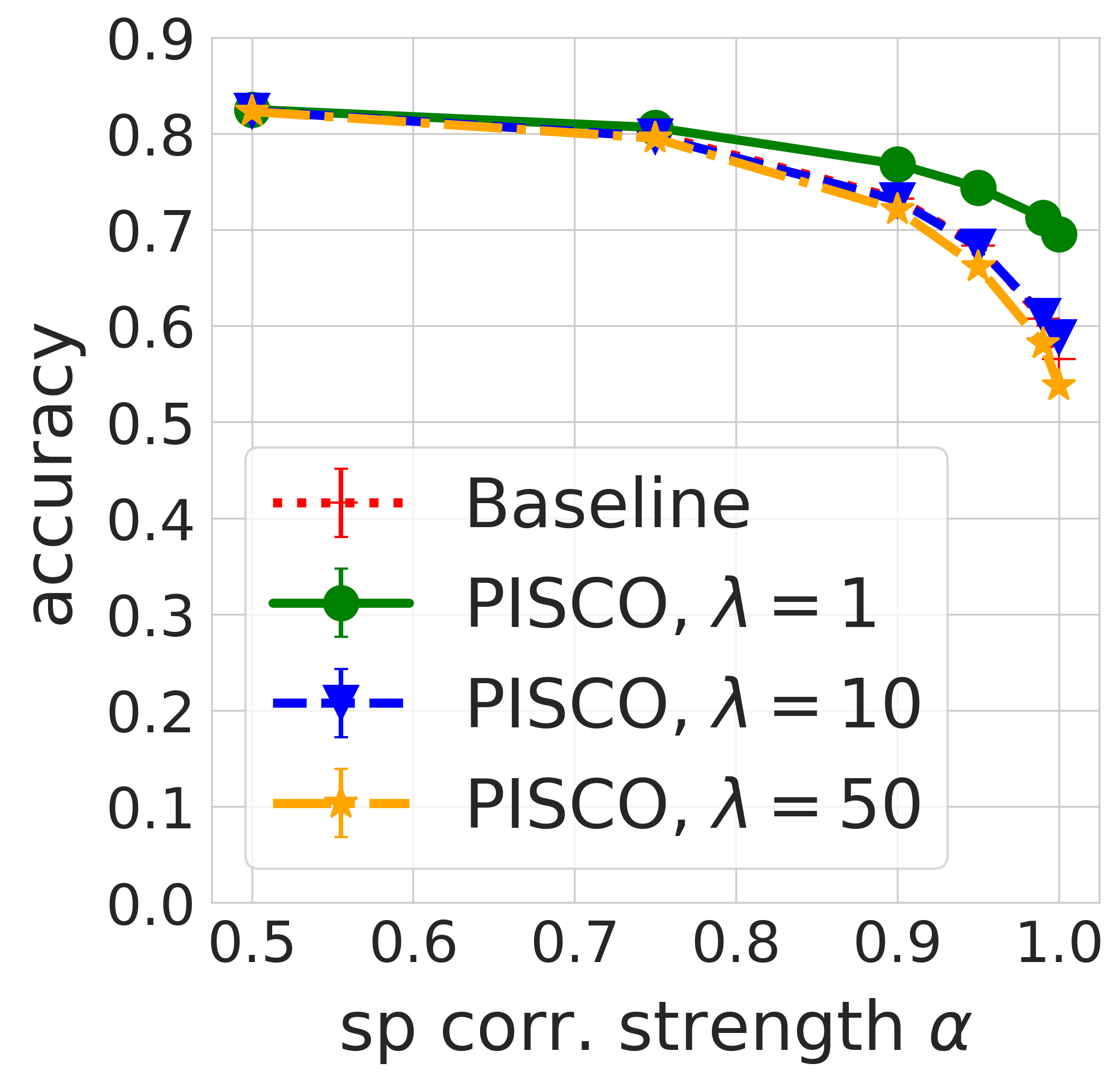}
         \caption{Contrast - \texttt{SimCLR}}
         \label{fig:contr-simclr-supp4}
     \end{subfigure}
     \hfill
     \begin{subfigure}[b]{0.22\textwidth}
         \centering
         \includegraphics[width=\textwidth]{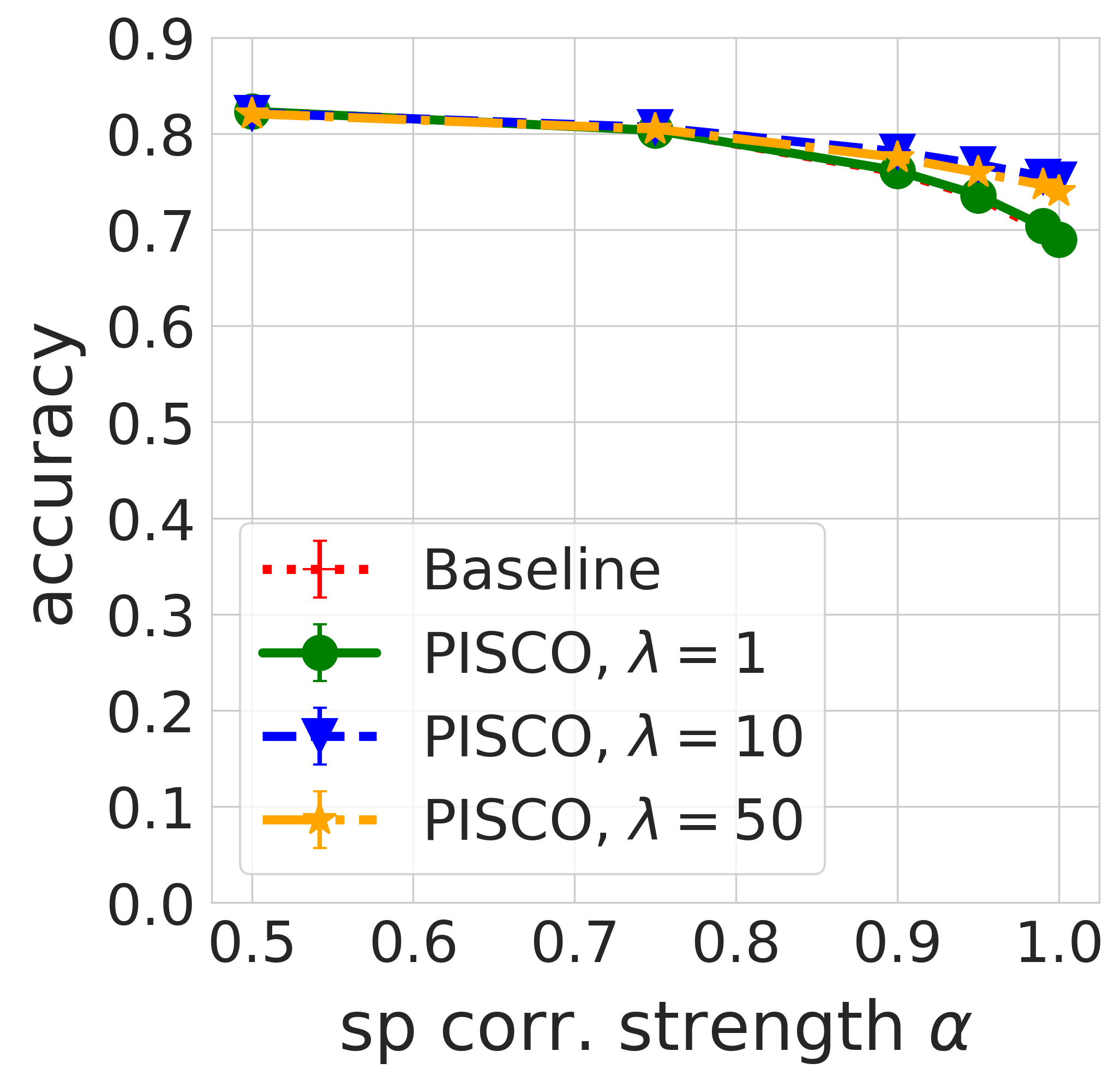}
         \caption{Blur - \texttt{SimCLR}}
         \label{fig:blur-simclr-supp4}
     \end{subfigure}
     \hfill
     \begin{subfigure}[b]{0.22\textwidth}
         \centering
         \includegraphics[width=\textwidth]{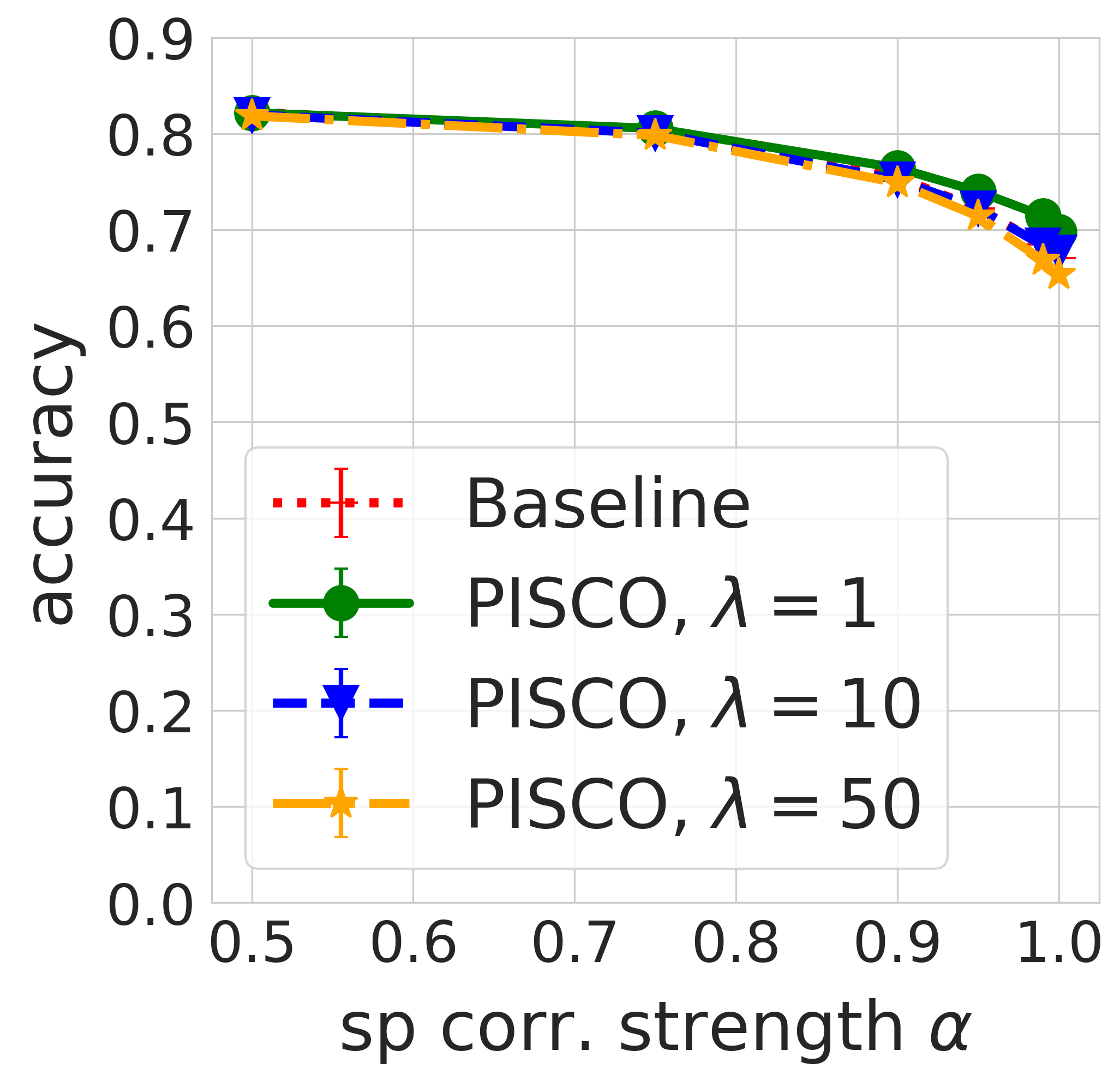}
         \caption{Saturation - \texttt{SimCLR}}
         \label{fig:sat-simclr-supp4}
     \end{subfigure}
         \centering
         \caption{$\eta = 1.0$, OOD performance of \texttt{SimCLR} representations on CIFAR-10 where the label is spuriously correlated with the corresponding transformation. Results are analogous to Figure \ref{fig:resnet_results_1}. The \texttt{SimCLR} baseline representations are less sensitive to contrast and saturation but remain sensitive to rotation.
         }
         \label{fig:simclr_results_1}
        
\end{figure}

\begin{table}
 \caption{$\eta = 1.0$, Performance of \texttt{Supervised} representations on CIFAR-10 test set in-distribution, i.e., no transformation (referred to as ``none''; last row), and OOD when modified with the corresponding transformation. \method\ with $\lambda=1$ provides significant improvements for rotation, contrast, and saturation, while preserving in-distribution accuracy.
 }
 \label{tb:ood_gen_resnet_1}
 \centering

\begin{tabular}{lccccc}
\toprule
Style &  \makecell{Baseline\\\scriptsize{(\texttt{Supervised})}} &  \makecell{\method\\(\small{$\lambda = 1$})} &  \makecell{\method\\(\small{$\lambda = 10$})} &  \makecell{\method\\(\small{$\lambda = 50$})} \\
\midrule
      rotation &  0.678 &               0.736 &                \textbf{0.748} &                0.747                 \\
    contrast & 0.625&               0.669 &                \textbf{0.724} &                0.721  \\
      saturation & 0.699 &               \textbf{0.744} &                0.736 &                0.729   \\
      blur & 0.817 &               \textbf{0.820} &                0.819 &                0.819    \\
     
     none &  \textbf{0.873} & 0.872 & 0.872 &  0.872    \\
\bottomrule
\end{tabular}
\end{table}

\begin{table}
 \caption{
 $\eta = 1.0$, Performance of \texttt{SimCLR} representations on CIFAR-10 test set in-distribution, i.e., no transformation (referred to as ``none''; last row), and OOD when modified with the corresponding transformation. \texttt{SimCLR} features are robust to these transformations and perform similarly to \method\ with $\lambda=1$.
 }
 \label{tb:ood_gen_simclr_1}
 \centering

\begin{tabular}{lccccc}
\toprule
 Style &  \makecell{Baseline\\\scriptsize{(\texttt{SimCLR})}} &  \makecell{\method\\(\small{$\lambda = 1$})} &  \makecell{\method\\(\small{$\lambda = 10$})} &  \makecell{\method\\(\small{$\lambda = 50$})} \\
\midrule
       rotation &             0.620 &               0.629 &                 \textbf{0.647} &                0.624\\
    contrast &             \textbf{0.816} &               \textbf{0.816} &                 0.814 &                0.811 \\
      saturation &             \textbf{0.810} &               0.808 &                 0.806 &                0.805 \\
      blur &             \textbf{0.808} &               0.805 &                 0.806 &                0.802 \\
      
      none &    \textbf{0.828}    & 0.827	  & 0.826    & 0.826    \\
      
\bottomrule
\end{tabular}
\end{table}


\subsection{Additional stylized ImageNet experiment results}
\label{supp:add_imagenet_results}
In this section for the ResNet-50 baseline, for each value of $\eta$, we report results for $\lambda$ values 1, 10, and 50. For the MAE-ViT-Base \citep{he2022masked} baseline, we report additional results for $\lambda$ values 1, 10, and 50.

\paragraph{ResNet-50 results for $\eta = 0.90$:} Results for when $\eta = 0.90$ can be found in Table \ref{tb:ood_gen_imagenet_090}.

\paragraph{ResNet-50 results for $\eta = 0.93$:} Results for when $\eta = 0.93$ can be found in Table \ref{tb:ood_gen_imagenet_093}. 

\paragraph{ResNet-50 results for $\eta = 0.95$:} Results for when $\eta = 0.95$ can be found in Table \ref{tb:ood_gen_imagenet_095}. In the main paper we reported results for $\eta = 0.95$ when $\lambda=1$. Table \ref{tb:ood_gen_imagenet_095} contains results for $\lambda$ value 1, and additional $\lambda$ values 10 and 50.

\paragraph{ResNet-50 results for $\eta = 0.98$:} Results for when $\eta = 0.98$ can be found in Table \ref{tb:ood_gen_imagenet_098}. 

\paragraph{ResNet-50 results for $\eta = 1.0$:} Results for when $\eta = 1.0$ can be found in Table \ref{tb:ood_gen_imagenet_1}. When $\eta = 1.0$, it means the number of features in the baseline is the same as the number of features learned using \method\ and as a result, similar for CIFAR-10 results in \S\ref{supp:add_cifar_results}, we observe the in-distribution performance of \method\ in this case being almost the same as that of baseline methods even for higher values of $\lambda$.

\paragraph{MAE-ViT-Base results for $\lambda$ values 1, 10, and 50:} Results for additional values of $\lambda$ when MAE-ViT-Base is the baseline are in Table \ref{tb:ood_mae_vit_2}.

\paragraph{Overall ImageNet results discussion.} When we vary $\eta$ in ImageNet experiments, we observe behavior similar to what we observed in the CIFAR-10 experiments in \S\ref{supp:add_cifar_results}: \method\ outperforms the baseline across all values of $\eta$ and $\lambda=1$ provides the best \method\ results in all $\eta$ values. An expected observation from the results is as $\eta$ increases, the in-distribution performance of \method\ goes up even for high values of $\lambda$ and its OOD performance slightly goes down.

\begin{table*}
\large
 \caption{$\eta = 0.90$, Top-1 and top-5 accuracies on 5 variations of the ImageNet test set for Baseline pre-trained ResNet-50 features and the corresponding post-processed \method\ features on different values of $\lambda$.
 }
 \label{tb:ood_gen_imagenet_090}
 \centering

\begin{tabular}{lccccccccc} 
\toprule
              Style & \multicolumn{2}{l}{\makecell{Baseline\\\large{(ResNet-50})}} & \multicolumn{2}{l}{\makecell{\method\\\large{($\lambda = 1$})}} & \multicolumn{2}{l}{\makecell{\method\\\large{($\lambda = 10$})}} & \multicolumn{2}{l}{\makecell{\method\\\large{($\lambda = 50$})}} \\
{} & Top-1 &  Top-5 & Top-1 & Top-5 & Top-1 &  Top-5 & Top-1 &  Top-5 \\
\midrule
dog sketch &0.514&0.752&\textbf{0.542} &\textbf{0.769}&0.522 &0.723 & 0.520 &0.717 \\
woman sketch &0.477&0.711&\textbf{0.511}&\textbf{0.743}&0.490&0.693&0.488&0.686 \\
Picasso dog &0.445&0.686&\textbf{0.495}&\textbf{0.730}&0.472&0.672&0.467& 0.665 \\
Picasso s.-p.& 0.474&0.706 & \textbf{0.508}&\textbf{0.737}& 0.490& 0.691& 0.490 &  0.685 \\
none &\textbf{0.758} &\textbf{0.927} &0.743& 0.916 & 0.740 & 0.910 &0.740&0.910 \\

\bottomrule
\end{tabular}
\vspace{-0.2cm}
\end{table*}

\begin{table*}
\large
 \caption{$\eta = 0.93$, Top-1 and top-5 accuracies on 5 variations of the ImageNet test set for Baseline pre-trained ResNet-50 features and the corresponding post-processed \method\ features on different values of $\lambda$.
 }
 \label{tb:ood_gen_imagenet_093}
 \centering

\begin{tabular}{lccccccccc} 
\toprule
              Style & \multicolumn{2}{l}{\makecell{Baseline\\\large{(ResNet-50})}} & \multicolumn{2}{l}{\makecell{\method\\\large{($\lambda = 1$})}} & \multicolumn{2}{l}{\makecell{\method\\\large{($\lambda = 10$})}} & \multicolumn{2}{l}{\makecell{\method\\\large{($\lambda = 50$})}} \\
{} & Top-1 &  Top-5 & Top-1 & Top-5 & Top-1 &  Top-5 & Top-1 &  Top-5 \\
\midrule
dog sketch        &                0.515 &  0.753 &                 \textbf{0.545} &   \textbf{0.775} &                 0.530 &  0.741 &                 0.528 &  0.737 \\
woman sketch  &                0.478 &  0.712 &                 \textbf{0.514} &   \textbf{0.748 }&                 0.500 &  0.710 &                 0.498 &  0.705 \\
Picasso dog      &                0.446 &  0.686 &                 \textbf{0.495} &   \textbf{0.735} &                 0.480 &  0.692 &                 0.479 &  0.688 \\
Picasso s.-p.     &                0.474 &  0.706 &                 \textbf{0.511} &   \textbf{0.742} &                 0.501 &  0.709 &                 0.499 &  0.706 \\
 none                  &                 \textbf{0.757} &   \textbf{0.927} &                0.746 &  0.918 &                 0.742 &  0.914 &                 0.743 &  0.913 \\

\bottomrule
\end{tabular}
\vspace{-0.2cm}
\end{table*}

\begin{table*}
\large
 \caption{$\eta = 0.95$, Top-1 and top-5 accuracies on 5 variations of the ImageNet test set for Baseline pre-trained ResNet-50 features and the corresponding post-processed \method\ features on different values of $\lambda$.
 }
 \label{tb:ood_gen_imagenet_095}
 \centering

\begin{tabular}{lccccccccc} 
\toprule
              Style & \multicolumn{2}{l}{\makecell{Baseline\\\large{(ResNet-50})}} & \multicolumn{2}{l}{\makecell{\method\\\large{($\lambda = 1$})}} & \multicolumn{2}{l}{\makecell{\method\\\large{($\lambda = 10$})}} & \multicolumn{2}{l}{\makecell{\method\\\large{($\lambda = 50$})}} \\
{} & Top-1 &  Top-5 & Top-1 & Top-5 & Top-1 &  Top-5 & Top-1 &  Top-5 \\
\midrule
dog sketch                    &                0.516 &  0.752 &                \textbf{0.546} &  \textbf{0.777} &                 0.534 &  0.751 &                 0.532 &  0.750 \\
woman sketch               &                0.478 &  0.712 &                \textbf{0.518} &  \textbf{0.752} &                 0.506 &  0.723 &                 0.504 &  0.719 \\
Picasso dog                   &                0.445 &  0.686 &                \textbf{0.500} &  \textbf{0.738} &                 0.486 &  0.705 &                 0.485 &  0.702 \\
Picasso s.-p.                  &                0.474 &  0.706 &                \textbf{0.514} &  \textbf{0.747} &                 0.505 &  0.721 &                 0.504 &  0.718 \\
none                                &                \textbf{0.757} &  \textbf{0.927} &                0.749 &  0.921 &                 0.745 &  0.917 &                 0.745 &  0.916 \\
\bottomrule
\end{tabular}
\vspace{-0.2cm}
\end{table*}

\begin{table*}
\large
 \caption{$\eta = 0.98$, Top-1 and top-5 accuracies on 5 variations of the ImageNet test set for Baseline pre-trained ResNet-50 features and the corresponding post-processed \method\ features on different values of $\lambda$.
 }
 \label{tb:ood_gen_imagenet_098}
 \centering

\begin{tabular}{lccccccccc} 
\toprule
              Style & \multicolumn{2}{l}{\makecell{Baseline\\\large{(ResNet-50})}} & \multicolumn{2}{l}{\makecell{\method\\\large{($\lambda = 1$})}} & \multicolumn{2}{l}{\makecell{\method\\\large{($\lambda = 10$})}} & \multicolumn{2}{l}{\makecell{\method\\\large{($\lambda = 50$})}} \\
{} & Top-1 &  Top-5 & Top-1 & Top-5 & Top-1 &  Top-5 & Top-1 &  Top-5 \\
\midrule
dog sketch        &                0.515 &  0.752 &                \textbf{0.553} &  \textbf{0.785} &                 0.541 &  0.769 &                 0.540 &  0.768 \\
woman sketch  &                0.478 &  0.711 &                \textbf{0.520} &  \textbf{0.757} &                 0.512 &  0.740 &                 0.511 &  0.738 \\
Picasso dog      &                0.446 &  0.686 &                \textbf{0.504} &  \textbf{0.744} &                 0.492 &  0.725 &                 0.492 &  0.722 \\
Picasso s.-p.    &                0.474 &  0.706 &                \textbf{0.518} &  \textbf{0.752} &                 0.512 &  0.737 &                 0.512 &  0.737 \\
none                 &                \textbf{0.757} &  \textbf{0.928} &                0.754 &  0.926 &                 0.751 &  0.922 &                 0.751 &  0.921 \\
\bottomrule
\end{tabular}
\vspace{-0.2cm}
\end{table*}

\begin{table*}
\large
 \caption{$\eta = 1.0$, Top-1 and top-5 accuracies on 5 variations of the ImageNet test set for Baseline pre-trained ResNet-50 features and the corresponding post-processed \method\ features on different values of $\lambda$.
 }
 \label{tb:ood_gen_imagenet_1}
 \centering

\begin{tabular}{lccccccccc} 
\toprule
              Style & \multicolumn{2}{l}{\makecell{Baseline\\\large{(ResNet-50})}} & \multicolumn{2}{l}{\makecell{\method\\\large{($\lambda = 1$})}} & \multicolumn{2}{l}{\makecell{\method\\\large{($\lambda = 10$})}} & \multicolumn{2}{l}{\makecell{\method\\\large{($\lambda = 50$})}} \\
{} & Top-1 &  Top-5 & Top-1 & Top-5 & Top-1 &  Top-5 & Top-1 &  Top-5 \\
\midrule
dog sketch                &                0.516 &  0.752 &                \textbf{0.552} &  \textbf{0.787} &                 0.548 &  \textbf{0.787} &                 0.548 &  0.786 \\
woman sketch           &                0.478 &  0.712 &                0.517 &  \textbf{0.756} &                 0.517 &  0.755 &                 \textbf{0.519} &  \textbf{0.756}\\
Picasso dog               &                0.446 &  0.686 &            \textbf{0.501} &  0.742 &                 0.499 &  0.744 &                 \textbf{0.501} &  \textbf{0.745} \\
Picasso s.-p.              &                0.474 &  0.706 &                0.517 &  0.752 &                 \textbf{0.518} &  0.754 &                 0.517 &  \textbf{0.755} \\
none                            &                \textbf{0.758} &  0.928 &                \textbf{0.758} &  \textbf{0.929} &                 \textbf{0.758} &  \textbf{0.929} &                 0.757 &  \textbf{0.929} \\

\bottomrule
\end{tabular}
\vspace{-0.2cm}
\end{table*}


\begin{table*}
\large
 \caption{$\eta = 0.95$, Top-1 and top-5 accuracies on 5 variations of the ImageNet test set for when MAE-ViT-Base features are the Baseline, and the corresponding accuracies for post-processed PISCO features on different values of $\lambda$.
 }
 \label{tb:ood_mae_vit_2}
 \centering

\begin{tabular}{lccccccccc} 
\toprule
              Style & \multicolumn{2}{l}{\makecell{Baseline\\\large{(MAE-ViT-Base})}} & \multicolumn{2}{l}{\makecell{\method\\\large{($\lambda = 1$})}} & \multicolumn{2}{l}{\makecell{\method\\\large{($\lambda = 10$})}} & \multicolumn{2}{l}{\makecell{\method\\\large{($\lambda = 50$})}} \\
{} & Top-1 &  Top-5 & Top-1 & Top-5 & Top-1 &  Top-5 & Top-1 &  Top-5 \\
\midrule
dog sketch         &                   0.530    &  0.749   &                0.575 &  \textbf{0.773} &                 \textbf{0.576} &  0.770 &                 0.576 &  0.770 \\
Picasso dog       &                   0.472     &  0.686  &                0.519 &  \textbf{0.716} &                 \textbf{0.520} &  0.714 &                 \textbf{0.520} &  0.714 \\
Picasso s.-p.      &                   0.512     &  0.727  &                \textbf{0.558} &  \textbf{0.752} &                 \textbf{0.558} &  0.748 &                 \textbf{0.558} &  0.748 \\
woman sketch    &                   0.504    &  0.719   &                \textbf{0.550 }&  \textbf{0.746} &                 0.549 &  0.744 &                 0.549 &  0.744 \\
none                    &                   0.811      &  0.952  &                \textbf{0.818} &  \textbf{0.953} &                 0.817 &  \textbf{0.953} &                 0.817 &  \textbf{0.953} \\

\bottomrule
\end{tabular}
\vspace{-0.2cm}
\end{table*}

\end{document}